\documentclass{article} 
\usepackage[a4paper, margin=1.1in]{geometry}

\usepackage{url}

\usepackage{graphicx}
\usepackage{amssymb}
\usepackage{amsmath}
\usepackage{amsfonts}
\usepackage{amsthm}
\usepackage{mathtools}
\usepackage{algorithm}
\usepackage{algcompatible}
\usepackage{thmtools,thm-restate}
 \usepackage[affil-it]{authblk}
 \usepackage[square,numbers,compress]{natbib}
\usepackage{nicefrac}
\usepackage{comment}
\usepackage{ifpdf}
\usepackage{bm}
\usepackage[table]{xcolor}
\usepackage{color}
\usepackage[mathlines]{lineno}
\usepackage{scalerel}
\usepackage{stackengine,wasysym}
\usepackage[stable]{footmisc}
\usepackage{enumerate}
\usepackage{tikz-cd}
\usetikzlibrary{arrows,positioning}
\tikzset{
  shift left/.style ={commutative diagrams/shift left={#1}},
  shift right/.style={commutative diagrams/shift right={#1}}
}
\usepackage{subcaption}
\usepackage{enumitem}
\usepackage{booktabs}
\usepackage{rotating}
\usepackage{wrapfig}

\usepackage{empheq}
\usepackage[hidelinks]{hyperref}

\newtheorem{thm}{Theorem}[section]

\newtheorem{lem}{Lemma}[section]

\newtheorem{prop}{Proposition}[section]

\newtheorem{dfn}{Definition}

\newtheorem{rk}{Remark}

\newcommand{\vect}[1]{\boldsymbol{#1}}
\newcommand{\tp}[1]{{#1}^{\mathsf T}}

\newcommand{\tr}{\mathrm{tr}}

\DeclareMathAlphabet\mathbfcal{OMS}{cmsy}{b}{n}
\DeclareMathOperator{\diag}{diag}
\newcommand{\half}{\frac12}

\newcommand{\mbE}{\mathbb E}
\newcommand{\E}{\mathrm E}
\newcommand{\V}{\mathrm{Var}}
\newcommand{\mCv}{\mathrm{Cov}}

\newcommand{\pa}{\partial}

\renewcommand{\epsilon}{\varepsilon}
\renewcommand{\Sigma}{\varSigma}

\newcommand{\bmu}{\vect\mu}

\newcommand{\bSigma}{\vect\Sigma}
\newcommand{\bsigma}{\vect\sigma}

\newcommand{\bv}{{\bf v}}

\newcommand{\bzero}{{\bf 0}}
\newcommand{\bu}{{\bf u}}

\newcommand{\be}{{\bf e}}

\newcommand{\bx}{{\bf x}}
\newcommand{\bX}{{\bf X}}

\newcommand{\bz}{{\bf z}}
\newcommand{\bZ}{{\bf Z}}

\newcommand{\bc}{{\bf c}}

\newcommand{\bI}{{\bf I}}
\newcommand{\bL}{{\bf L}}

\newcommand{\mI}{{\mathcal I}}
\newcommand{\mL}{{\mathcal L}}
\newcommand{\mM}{{\mathcal M}}
\newcommand{\mN}{{\mathcal N}}

\newcommand{\mP}{{\mathcal P}}

\newcommand{\mS}{{\mathcal S}}

\newcommand{\mbR}{{\mathbb R}}

\newcommand{\unif}{\mathcal{U}}

\DeclareMathOperator{\sech}{sech}

\DeclareMathOperator{\arcosh}{arcosh}

\DeclareMathOperator{\atan2}{atan2}

\title{Probabilistic Geodesic Flow Matching\\ on Location-Scale Families}

\author{Zeyuan Yu \qquad Zhi Chang \qquad Shiwei Lan$^\dagger$ \\
School of Mathematical \& Statistical Sciences\\
Arizona State University\\
Tempe, AZ 85287 USA \\
{\upshape\texttt{\symbol{123}zeyuany1,zchang7,slan\symbol{125}@asu.edu}} 
}

\date{}

\begin{document}

\maketitle

\begin{abstract}
Flow matching (FM) has recently emerged as a promising framework for generative modeling due to its conceptual simplicity and strong empirical performance. In FM, samples are transported along a vector field parameterized by a neural network, inducing a probability path that evolves from a simple noise distribution to the target data distribution, 
governed by an ordinary differential equation (ODE). 
However, existing FM approaches predominantly rely on probability paths derived from optimal transport (OT) between Gaussian distributions, which may be suboptimal for capturing complex data with inhomogeneous structures such as heavy tail or sharp contrast. In this work, we generalize FM to the broader class of location–scale families for handling data inhomogeneity and introduce a novel class of probability paths defined as geodesics on the manifold of probability distributions. We name this approach \emph{probabilistic geodesic flow matching} to distinguish it from prior geodesic (Riemannian) FM methods defined in input space.
We argue that Euclidean OT-based paths are not necessarily optimal in probability space and may limit modeling flexibility. Through synthetic benchmarks and scientific datasets at different scales, we demonstrate that the proposed method more effectively captures complex distributions, leading to improved or comparable performance compared with SOTA geometry-motivated generative models.
\end{abstract}

\section{Introduction}
Diffusion models and flow matching methods are two classes of popular generative modeling approaches. 
Denoising diffusion probabilistic model \citep{Ho_2020} is a discrete Markov model that converts a complex data distribution into a simple tractable distribution by gradually adding noise in a forward process and then trains a generative model to remove the noise in a reverse process.
\cite{Song_2021} generalize it to a continuum of distributions and reformulate these two processes using SDEs.
In the SDE framework, the Fokker–Planck equation governs the evolution of the density path, which induces an ODE that enables a deterministic flow with marginal densities the same as SDE \citep{Maoutsa_2020}. 
The continuous normalizing flow \citep[CNF][]{Chen_2018} directly models the vector field of the probabilistic flow that generates the probability path from noise to data distributions, bypassing the noise-adding forward process. Recently, flow matching \citep[FM][]{Lipman_2023} introduces a diffusion-style training objective to predict the vector field and improves efficiency by avoiding expensive ODE simulation in CNF training.

Despite the connection to diffusion processes, dominant FM methods are based on probability paths as an optimal transport (OT) interpolant \citep{MCCANN_1997} between Gaussian distributions. OT paths are known to be geodesics in the 2-Wasserstein space \citep{Benamou_2000, Villani_2009}. They may not necessarily be optimal for complex data distributions, e.g. non-Gaussian distributions with heavy tails.
Heavy-tailed priors using student-t distributions have recently been studied in diffusion models \citep{Pandey_2025}.
In this work, we generalize FM by deviating from the default choice of Gaussian distributions and explore a broader location-scale family. In particular, we focus on exponential power distributions \citep{Gomez_1998}, which have the flexibility of imposing regularization and modeling tail probabilities through a parameter $q>0$ and embrace Gaussian distribution as a special case $q=2.0$.
To search for the optimal probability path on the manifold of exponential power distributions, we adopt the Fisher-Rao metric and solve the geodesic under this metric. 
The proposed FM with this geodesic path is hence named \emph{probabilistic geodesic flow matching}, to differentiate it from the existing Riemannian (geodesic) FM \citep{Chen_2024}.

\paragraph{Connection to Existing Literature}
Our work directly generalizes the FM method \citep{Lipman_2023}. In addition to adopting broader location-scale distributions, our framework, developed from probability path to vector field and then sample flow, is reverse to \cite{Lipman_2023}, but provides a natural foundation of manifold perspective of probability paths.
We want to highlight that the Riemannian FM \citep{Chen_2024} has its geodesics defined in the input space of most applications. However, our \emph{probabilistic} geodesics are defined on the manifold of exponential power and more general location-scale distributions.
Fisher-Rao metric is also considered in \cite{Cheng_2024} and \cite{Davis_2024} for discrete data. Other manifold based FM includes \cite{Kapusniak_2024} with a learned metric and \cite{Atanackovic_2025} with a $2$-Wasserstein. Despite the different applications, their constructions either take advantage of generic spherical geometry or learn from data, while our derivation is more native to exponential power distributions with \textbf{analytic} solution (see Table \ref{tab:fm-comparison} for a more comprehensive comparison).
Our work makes multi-fold contributions: 
\begin{enumerate}[noitemsep]
\item It generalizes FM with a broader location-scale family beyond Gaussian distributions.
\item It interprets optimal probability paths as energy-minimizing trajectories in the probability space with appropriate metrics.
\item It proposes probabilistic geodesic FM demonstrating improved or comparable performance.
\end{enumerate}

The remainder of the paper is organized as follows. Section \ref{sec:background} reviews the background on FM and exponential power distributions. Section \ref{sec:location-scale} introduces the flexible FM framework in location-scale family. Then we derive the probabilistic geodesics on the manifold of exponential power distributions in Section \ref{sec:PG}. We demonstrate the numerical advantages in Section \ref{sec:numerics} and conclude in Section \ref{sec:conclusion} with some discussion on limitations and future directions.

\section{Background Review}
\label{sec:background}

\subsection{(Conditional) Flow Matching}
To learn the distribution $q(\cdot)$ of a given dataset and generate synthetic samples $\bx\sim q(\cdot)$, flow-based approaches evolve noisy data $\bx_0\sim p_0(\cdot)$ to clean data $\bx_1\sim p_1(\cdot)\approx q(\cdot)$ along a flow mapping $\phi_t: \mbR^d \to \mbR^d$ defined by the following ODE with given vector field $\bv_t:\mbR^d \to \mbR^d$ for $t\in[0, 1]$:
\begin{equation}\label{eq:flow}
\begin{aligned}
    \frac{d}{dt} \phi_t(\bx) &= \bv_t(\phi_t(\bx)),  \\
    \phi_0(\bx) &= \bx_0, 
\end{aligned}
\end{equation}
where $\bx_t := \phi_t(\bx_0)$ follows a probability path $p_t(\cdot)$ given by the push-forward equation
\begin{equation}\label{eq:pushfwd}
    p_t = [\phi_t]_\sharp p_0, \quad [\phi_t]_\sharp p_0(\bx) = p_0(\phi_t^{-1}(\bx)) \det\left[\frac{\pa \phi_t}{\pa \bx}\right]^{-1} .
\end{equation}

A vector field $\bv_t$ is said to generate the probability path $p_t$ if and only if the continuity equation holds:
\begin{equation}\label{eq:continuity}
    \pa_t p_t(\bx) = -\nabla_\bx \cdot (\bv_t(\bx) p_t(\bx)) ,
\end{equation}
which induces the following equation evolving the logarithm of density:
\begin{equation}\label{eq:logpdf_CNF}
\frac{d}{dt} \log p_t(\bx_t) = -\tr\left(\frac{\pa\bv_t}{\pa\bx_t }\right) = -\nabla\cdot \bv_t(\bx_t) .
\end{equation}
Once the vector field $\bv_t$ is learned from the data, a synthetic sample can be generated by solving the ODE \eqref{eq:flow} up to time $t=1$: $\bx_1 = \phi_1(\bx_0) \sim p_1(\cdot)\approx q(\cdot)$.


The FM considers a conditional probability path $p_t(\cdot|\bx_1)$ for a fixed data point $\bx_1\in\mbR^d$ such that $p_0(\cdot)=p_0(\cdot|\bx_1)$ and $p_1(\cdot) = \int p_1(\cdot|\bx_1)q(\bx_1)d\bx_1\approx q(\cdot)$. 
A common choice is $p_0(\cdot|\bx_1)=\mN(\cdot ;\bzero, \bI)$ and $p_1(\cdot|\bx_1)=\mN(\cdot ;\bx_1, \sigma_{\min}^2\bI)$ for some small $\sigma_{\min}>0$.
A conditional vector field $\bu_t(\cdot |\bx_1)$ is then constructed to generate this conditional probability path $p_t(\cdot|\bx_1)$ by the flow $\frac{d}{dt} \psi_t(\bx) = \bu_t(\psi_t(\bx)|\bx_1)$. Hence, conditional FM (CFM) trains a neural network $\bv_t(\cdot; \theta)$ to match $\bu_t(\cdot|\bx_1)$ by minimizing the loss:
\begin{equation}\label{eq:CFM_loss}
    \mL_\textrm{CFM}(\theta) = \mbE_{t, q(\bx_1), p_t(\bx|\bx_1)}\Vert \bv_t(\bx;\theta) - \bu_t(\bx|\bx_1)\Vert_2^2 .
\end{equation}
where $t\sim \unif[0,1]$, $\bx_1\sim q(\cdot)$, and $\bx\sim p_t(\cdot|\bx_1)$.

\emph{Starting with the conditional flow} $\psi_t$ as a canonical affine transformation in the family of Gaussian distributions $\mN(\bmu_t, \sigma^2_t\bI)$:
\begin{equation}\label{eq:affine_path}
\psi_t(\bx) = \bmu_t(\bx_1) + \sigma_t(\bx_1) \bx, 
\end{equation}
\cite{Lipman_2023} prove that the corresponding conditional vector field $\bu_t(\cdot|\bx_1)$ in the following unique format generates the desired conditional probability path $p_t(\bx|\bx_1)$:
\begin{equation}\label{eq:conditional_vf}
\bu_t(\bx|\bx_1) = \dot \bmu_t(\bx_1) + \frac{\dot\sigma_t(\bx_1)}{\sigma_t(\bx_1)}(\bx - \bmu_t(\bx_1)) .
\end{equation}

In particular, \cite{Lipman_2023} propose and advocate the conditional flow based on the optimal transport (OT) displacement map $\psi$ pushing from $p_0(\cdot|\bx_1)$ to $p_1(\cdot|\bx_1)$ \citep{MCCANN_1997}:
\begin{equation}\label{eq:OT_flow}
    \psi_t = (1-t)\mathrm{id} + t \psi .
\end{equation}
The OT displacement between $\mN(\bmu_0, \sigma_0^2\bI)$ and $\mN(\bmu_1, \sigma_1^2\bI)$ is known to be $\psi(\bx) = \bmu_1 + \frac{\sigma_1}{\sigma_0}(\bx - \bmu_0)$ \citep{Villani_2009}. 
With the previous Gaussian conditionals $p_0(\cdot|\bx_1)=\mN(\cdot; \bzero, \bI)$ and $p_1(\cdot|\bx_1)=\mN(\cdot; \bx_1, \sigma^2_{\min}\bI)$,
\cite{Lipman_2023} develop OT based conditional \emph{sample flow ($\psi_t$) $\to$ vector field ($\dot\psi_t$) $\to$ probability path ($p_t(\cdot|\bx_1)$)} in turn as follows:
\begin{equation}\label{eq:OT_interp}
\begin{aligned}
\psi_t(\bx)=(1-(1-\sigma_{\min})t)\bx + t\bx_1, &\quad \bu_t(\bx|\bx_1) = \frac{\bx_1 - (1-\sigma_{\min})\bx}{1-(1-\sigma_{\min})t}, \\
p_t(\cdot|\bx_1) = \mN(\cdot; \bmu_t, \sigma_t\bI), &\quad  \bmu_t = (1-t)\bmu_0 + t\bmu_1, \; \sigma_t = (1-t)\sigma_0 + t\sigma_1 ,
\end{aligned}
\end{equation}
where $\bmu_0=\bzero, \bmu_1=\bx_1$, $\sigma_0=1$ and $\sigma_1 = \sigma_{\min}>0$.

\subsection{Exponential Power Distribution}
Although Gaussian distribution is the default choice of many FM models, it is just a member of a larger \emph{location-scale} family that also include Cauchy, Laplace, and Student's-t distributions.
\begin{dfn}[location-scale family]
A location-scale family is a family of probability distributions parametrized by a location parameter $\bmu\in\mbR^d$ and a positive scale parameter $\sigma>0$. Suppose $Z$ is a random variable following distribution in this family corresponding to $\bmu=\bzero$ and $\sigma=1$ with cumulative distribution function (CDF) $F_0(z)$. Then $X=\bmu+\sigma Z$ has CDF $F(X) = F_0((X-\bmu)/\sigma)$.
\end{dfn}

In this paper, we introduce another more flexible member, the \emph{exponential power distribution}, a.k.a. generalized normal distribution, to flow matching methods for broader choices beyond Gaussian.
\cite{Gomez_1998} has a multivariate version of the exponential power distribution defined below.
\begin{dfn}
The \emph{exponential power distribution}, denoted as $\mP_q(\bmu, \bSigma)$, has the following probability density
\begin{equation}
p(\bx)= \frac{q\Gamma(\frac{d}{2})}{2\Gamma(\frac{d}{q})} 2^{-\frac{d}{q}}\pi^{-\frac{d}{2}} |\bSigma|^{-\half} \exp\left\{-\frac{r^\frac{q}{2}}{2}\right\}, \quad r(\bx) = \tp{(\bx-\bmu)} \bSigma^{-1} (\bx-\bmu) .
\end{equation}
\end{dfn}
\begin{rk}
When $q=2$, the exponential power distribution reduces to multivariate normal distribution, i.e. $\mP_2(\bmu, \bSigma) = \mN(\bmu, \bSigma)$.
Generally, the parameter $q$ controls the regularization around center and tail behavior: the smaller $q$, the sharper regularization and the heavier tail it induces, as illustrated in Figure \ref{fig:exponential_power}.
\end{rk}

\begin{wrapfigure}{r}{0.48\textwidth}
  \centering
   \vspace{-12pt}
  \includegraphics[width=1\linewidth, height=.3\linewidth]{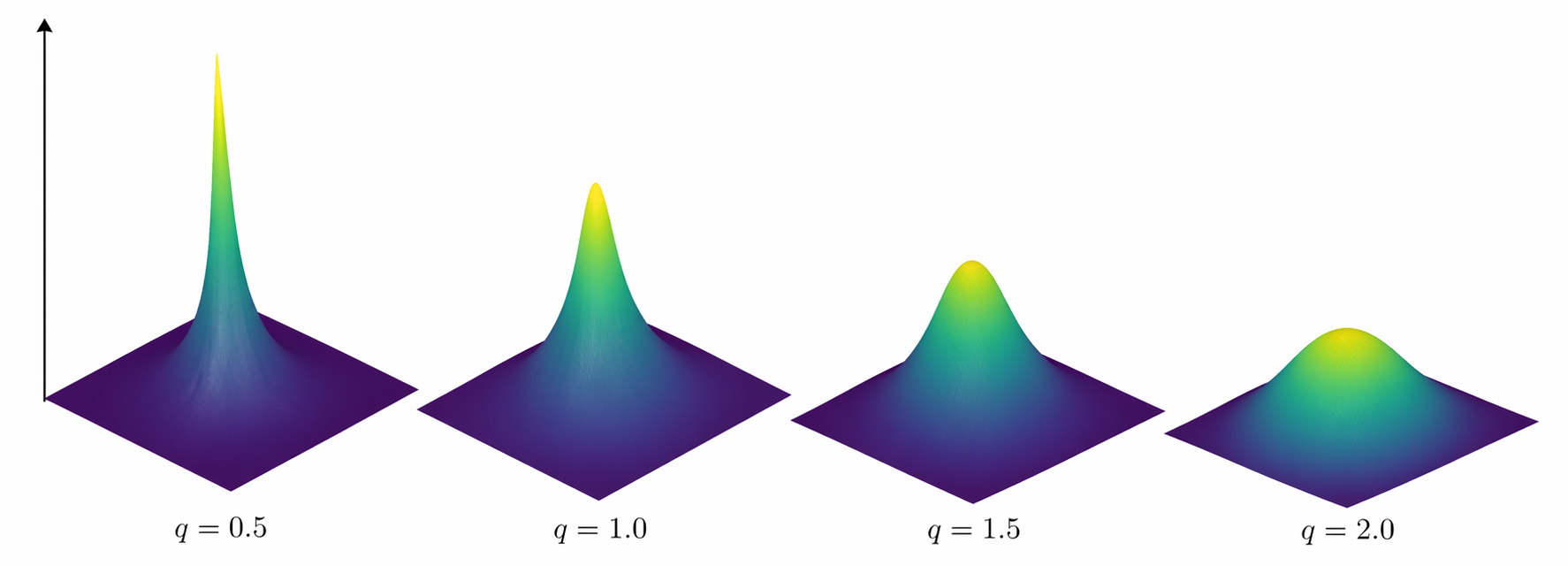}
  \vspace{-15pt}
  \caption{Probability densities of exponential power distributions $\mP_q(\bzero, \bI)$ for selected $q$'s.}
  \label{fig:exponential_power}
\end{wrapfigure}

As a special elliptic contoured distribution \citep{Johnson_1987,fang1990generalized}, the exponential power distribution is location-scale, i.e. if $Z\sim \mP_q(\bzero, \bI)$, then $X=\bmu+\sigma Z\sim \mP_q(\bmu, \sigma^2\bI)$.
In this work, we focus on the isotropic case $\bSigma=\sigma^2\bI$ and anisotropic case $\bSigma = \diag(\bsigma^2)$.

The following proposition synthesizes the properties of exponential power distribution \citep{Gomez_1998, Li_2023} useful for the development of new flow matching methods.
\begin{prop}\label{prop:epd_prop}
For a $d$-dimensional random variable $X\sim \mP_q(\bmu, \sigma^2\bI)$, we have
\begin{enumerate}[nosep]
\item $\E(X)=\bmu$, $\mCv(X) = s_{q,d}\bSigma$ with a scaling factor $s_{q,d} = \frac{2^{\frac{2}{q}}\Gamma(\frac{d+2}{q})}{d\Gamma(\frac{d}{q})}$.
\item $X \overset{L}{=} \bmu + \sqrt{r} \bL U_S$, where $r = \tp{(X-\bmu)} \bSigma^{-1} (X-\bmu)$ and $r^{\frac{q}{2}} \sim \Gamma(\alpha=\frac{d}{q}, \beta=\half)$, $\bL\tp{\bL}=\bSigma$, and $\sqrt{r}\perp U_S$ with $U_S\sim \unif(\mS^{d-1})$ uniformly distributed on the sphere.
\end{enumerate}
\end{prop}

\section{Flexible Flow Path in Location-Scale Family}\label{sec:location-scale}
In this section, we develop a series of novel flow matching methods in the broad location-scale family and propose new algorithms based on exponential power distributions. In addition, we also explore different probability paths by minimizing flow energies under various metric tensors, including Euclidean and Fisher-Rao (Section \ref{sec:PG}).

In the OT theory \citep{Villani_2009}, the OT displacement between two univariate probability distributions $p_0(\cdot|x_1)$ and $p_1(\cdot|x_1)$ is $\psi(x) = F_1^{-1}(F_0(x))$ with $F_0$ and $F_1$ being their CDFs respectively.
We immediately have OT interpolant path based on the conditional flow \eqref{eq:OT_flow} in the one dimensional case.
For general dimensions, we may have the following OT-based conditional flow
\begin{equation}\label{eq:OT_qep}
\psi_t(\bx) = (1-t)\bx + t \psi(\bx), \quad
    \psi(\bx) = [\psi(x_i)]_{i=1}^d, \quad \psi(x_i) = F_1^{-1}(F_0(x_i)) .
\end{equation}
Because the exponential power distribution is location-scale, we have $F_1(x)=F_0((x-x_1)/\sigma_{\min})$.
Substituting 
it
into \eqref{eq:OT_qep} yields the same OT map
    $\psi(\bx) = \bx_1 + \sigma_{\min} \bx$ .

This observation motivates us to \emph{start the search with any probability path in the broad location-scale family}, instead of being restricted to Gaussian distributions and the canonical affine path \eqref{eq:affine_path}.
Note that probability paths $\{p_t\}$ in the location-scale family parametrized by location and scale coordinates $(\bmu_t, \sigma_t)$ form a Riemannian manifold endowed with an information-geometry metric \citep{amari00}, which will be explored in Section \ref{sec:PG}.
The following theorem states that any probability path $p_t(\cdot ; \bmu_t, \sigma_t)$ from the location-scale family can be generated by a determined vector field.

\begin{restatable}{thm}{vf}
\label{thm:vf}
Any probability path $p_t(\cdot ; \bmu_t, \sigma_t)$ in the location-scale family with pdf $p_0$ for the standard random variable $\bZ = (\bX-\bmu_t)/\sigma_t$ has the following vector field $\bv_t$ that satisfies the continuity equation \eqref{eq:continuity}:
\begin{equation}\label{eq:VF_locationscale}
\bv_t(\bx) = \dot\bmu_t + \frac{\dot\sigma_t}{\sigma_t}(\bx-\bmu_t) . 
\end{equation}
\end{restatable}
\begin{proof}
See Appendix \ref{apx:vf} .
\end{proof}
\begin{rk}
The proof of this theorem is distribution-agnostic in the location-scale family.
Though the format of the vector field $\bv_t$ to generate probability $p_t$ is the same as in \citep[Theorem 3 of][]{Lipman_2023}, our theorem generalizes it to the location-scale family from which we can use the exponential power distribution $\mP_q$ beyond Gaussian $\mP_2=\mN$.
\end{rk}

Starting with the probability path $p_t(\cdot ; \bmu_t, \sigma_t)$, Theorem \ref{thm:vf} determines the conservation law \eqref{eq:continuity} \eqref{eq:logpdf_CNF} respecting vector field.
The following theorem states that such vector field \eqref{eq:VF_locationscale} prescribes the flow.

\begin{restatable}{thm}{path}
\label{thm:path}
The flow equation \eqref{eq:flow} with vector field $\bv_t$ \eqref{eq:VF_locationscale} has the unique solution as an affine transformation:
\begin{equation}\label{eq:affine_flow}
    \phi_t(\bx) = \bmu_t + \sigma_t \bx .
\end{equation}
\end{restatable}
\begin{proof}
See Appendix \ref{apx:path} .
\end{proof}
\begin{rk}
Although we end up with the same affine flow,
unlike \cite{Lipman_2023}, which searches for proper flow first, we can start with any probability path in the vast location-scale family, e.g. Normal, Laplace, Student's t, elliptical distributions, then specify the vector field \eqref{eq:VF_locationscale} compatible with the continuity law \eqref{eq:continuity}, and determine the flow \eqref{eq:affine_flow} at last.
Supported by Theorem \ref{thm:vf} and Theorem \ref{thm:path}, the workflow \emph{probability path ($p_t$) $\to$ vector field ($\bv_t$) $\to$ sample flow ($\phi_t$)} is reverse to that of \cite{Lipman_2023}, which lends us greater flexibility in choosing flow in the FM framework.
\end{rk}

We emphasize the novelty and importance of the order \textbf{probability path ($p_t$) $\xrightarrow{Th \ref{thm:vf}}$ vector field ($\bv_t$) $\xrightarrow{Th \ref{thm:path}}$ sample flow ($\phi_t$)}, which would otherwise no longer guarantee $p_t = [\phi_t]_\sharp p_0$ to have tractable geodesics when $p_0$ moves away from Gaussian, a property lacking in \cite{Kapusniak_2024, Davis_2024}.
Because the conditional probability path $p_t(\cdot|\bx_1)$ in the location-scale family is parametrized by location $\bmu_t$ and scale $\sigma_t$, what remains in the CFM framework is to determine them with boundary conditions $(\bmu_0, \sigma_0)=(\bzero, 1)$ and $(\bmu_1, \sigma_1)=(\bx_1, \sigma_{\min})$ to drive the conditional vector field $\bu_t(\cdot|\bx_1)$ \eqref{eq:conditional_vf} and the conditional flow $\psi_t$ \eqref{eq:affine_path}.
Then a neural network is trained to regress $\bu_t(\cdot|\bx_1)$ along the flow $\psi_t(\bx_0)$, i.e. $\bu_t(\psi_t(\bx_0)|\bx_1)=\frac{d}{dt}\psi_t(\bx_0)$, with the following loss:
\begin{equation}\label{eq:loss_epd}
\mL_\textrm{CFM}(\theta) = \mbE_{t, q(\bx_1), p(\bx_0)}\Vert \bv_t(\psi_t(\bx_0);\theta) - \frac{d}{dt}\psi_t(\bx_0)\Vert_2^2 ,
\end{equation}
where $t\sim \unif[0,1]$, $\bx_1\sim q(\cdot)$, and $\bx_0\sim p_0(\cdot)=\mP_q(\cdot; \bzero, \bI)$.

In the following, we specify the coordinate path $\{(\bmu_t, \sigma_t): t\in[0,1]\}$ on the manifold of location-scale probability distributions $\mM=\{p(\cdot; \bmu, \sigma)\}$ by minimizing the flow energy.
We will revisit and introduce three probability paths in the family of exponential power distributions: i) minimal energy path; ii) non-linear interpolant path; and iii) (probabilistic) geodesic path.

\subsection{Minimal Energy Path}
OT with quadratic cost is known to be an energy minimization problem under the Wasserstein-2 metric \citep{Benamou_2000}.
To determine the best conditional vector field, we adopt the energy criterion for the flow:
\begin{equation}\label{eq:energy}
    e(\bv) = \mbE_{p_t(\bx)} \int_0^1 \Vert \dot\bx_t\Vert^2 dt = \int_{\mbR^d} \int_0^1 p_t(\bx) \Vert \bv_t(\bx)\Vert_2^2 dt d\bx .
\end{equation}
Minimizing the energy \eqref{eq:energy} leads to geodesic in the Euclidean space, as stated below.

\begin{restatable}{prop}{minengy}
\label{prop:minengy}
The following OT interpolant path minimizes the flow energy \eqref{eq:energy}:
\begin{equation}\label{eq:OT_probpath}
    \bmu_t = (1-t)\bmu_0 + t\bmu_1 =t\bx_1, \quad \sigma_t = (1-t)\sigma_0 + t\sigma_1 = 1-t + t\sigma_{\min} .
\end{equation}
The minimal energy attained is $e_{\min} = \Vert \bmu_1-\bmu_0\Vert^2 + d|\sigma_1-\sigma_0|^2 .$
\end{restatable}
\begin{proof}
See Appendix \ref{apx:minengy} .
\end{proof}

Therefore, we have the conditional flow and vector field along the minimal energy path for loss \eqref{eq:loss_epd}:
\begin{equation*}
    \psi_t(\bx) = (1-(1-\sigma_{\min})t)\bx + t\bx_1, \quad \frac{d}{dt}\psi_t(\bx) = \bx_1 - (1-\sigma_{\min})\bx .
\end{equation*}


\subsection{Nonlinear Interpolant Path}
The OT-interpolant \eqref{eq:OT_probpath} is a linear interpolation between the starting $(\bmu_0, \sigma_0)$ and ending $(\bmu_1, \sigma_1)$ points on the manifold. 
A popular nonlinear alternative involves the functions $\sin$ and $\cos$ \citep{Lipman_2024}:
\begin{equation}\label{eq:sino_interp}
    \bx_t = \psi_t(\bx) = \bx \cos(\pi t/2) + \bx_1 \sin(\pi t/2) .
\end{equation}
By $\frac{d}{dt}\psi_t(\bx) = \bu_t(\psi_t(\bx)|\bx_1)$, we can derive the conditional vector field as
    $\bu_t(\bx|\bx_1) = \frac{\pi}{2} \frac{\bx_1 - \bx\sin(\pi t/2)}{\cos(\pi t/2)}$ ,
which 
generates the conditional probability path $p_t(\cdot; \bmu_t, \sigma_t)$ in the location-scale family with coordinates:
$\bmu_t(\bx_1) = \bx_1\sin(\pi t/2), \quad \sigma_t(\bx_1) = \cos(\pi t/2)$,
although it does not minimize the energy \eqref{eq:energy}.
Later we will show that it is actually a special case of the probabilistic geodesic path on the manifold of exponential power distributions with Fisher-Rao metric (Section \ref{sec:PG}).

\section{Probabilistic Geodesic Flow}
\label{sec:PG}
To search for the optimal probability path $\{p_t(\cdot; \bmu_t, \sigma_t): t\in[0,1]\}$ in the location-scale family, it is natural to adopt the Riemannian manifold of probability distributions $\mM=\{p(\cdot; \bmu, \sigma)\}$ with Fisher-Rao metric $g$ \citep{amari00}.
From the manifold perspective, we identify $(\bmu_t, \sigma_t)\leftrightarrow p_t(\cdot; \bmu_t, \sigma_t)$.
Although the methodology introduced applies to general location-scale distributions, we focus on the exponential power distribution $\mP_q(\bmu, \sigma^2\bI)$. 
Lemma \ref{lem:frmetric} shows that the Fisher-Rao metric \eqref{eq:fisher_metric} for this manifold is $ds^2 = \frac{c_{\mu}\Vert d\bmu\Vert^2+c_{\sigma} d\sigma^2}{\sigma^2}$ with $c_{\mu}(d,q) = \frac{2^{-\frac{2}{q}}}{d} \frac{\Gamma(\frac{d-2}{q} +2)}{\Gamma(\frac{d}{q})} q^2$, and $c_{\sigma}=qd$.
Note that if $X\sim \mP_q(\bmu_t, \sigma_t^2\bI)$, then $\bv_t(X) \sim \mP_q(\dot\bmu_t, \dot\sigma_t^2\bI)$ by \eqref{eq:VF_locationscale}.
The flow energy \eqref{eq:energy} can be redefined by replacing the Euclidean metric with the Fisher-Rao metric \eqref{eq:fisher_metric}:
\begin{equation}\label{eq:energy_fisher}
    e(\bv) = \mbE_{p_t(\bx)} \int_0^1 \Vert \dot\bx_t\Vert^2_g dt = \int_{\mbR^d} \int_0^1 p_t(\bx) \Vert \bv_t(\bx)\Vert^2_g dt d\bx = \int_0^1 \frac{c_{\mu}\Vert\dot\bmu_t\Vert^2+c_\sigma \dot\sigma_t^2}{\sigma_t^2} dt.
\end{equation}
By the calculus of variation, we get the geodesic as a semi-circle on the manifold of exponential power distributions $\mP_q(\bmu_t, \sigma_t^2\bI)$ in the following theorem.

\begin{restatable}{thm}{geod}
\label{thm:geod}
The following geodesic path minimizes the variational energy \eqref{eq:energy_fisher}:
\begin{subequations}\label{eq:geod_semicirc}
\begin{align}
\bmu_t &= \bmu_0 + x(t) \be, \quad x(t) = \mu_c + R\cos\theta(t), \quad R = \sqrt{\mu_c^2+\lambda^2\sigma_0^2}  ;\label{eq:geod_semicirc-1}\\
\sigma_t &= \frac{R}{\lambda} \sin\theta(t), \quad \theta(t) = 2\arctan\left(e^{Lt}\tan\frac{\theta_0}{2}\right), \quad L = \log\frac{\tan(\theta_1/2)}{\tan(\theta_0/2)} , \label{eq:geod_semicirc-2}
\end{align}
\end{subequations}
where $\lambda = \sqrt{c_\sigma/c_\mu}$, $\be= (\bmu_1-\bmu_0)/\Vert \bmu_1-\bmu_0\Vert_2$, $\mu_c = \frac{\Vert \bmu_1-\bmu_0\Vert^2+\lambda^2(\sigma_1^2-\sigma_0^2)}{2\Vert \bmu_1-\bmu_0\Vert}$, and $\theta_0 = \atan2(\lambda\sigma_0, -\mu_c)$, $\theta_1 = \atan2(\lambda\sigma_1, \Vert \bmu_1-\bmu_0\Vert-\mu_c)$.
The minimal energy attained is $e_{\min}=c_{\sigma} L^2$, where $\sqrt{c_{\sigma}}|L|$ is the Fisher-Rao distance between $\mP_q(\bmu_0, \sigma_0^2\bI)$ and $\mP_q(\bmu_1, \sigma_1^2\bI)$, for which $|L|$ can be rewritten in the original parameters as
\begin{equation}\label{eq:Fisher-Rao-distance}
|L| = \arcosh\left(1+ \frac{c_{\mu}\Vert \bmu_1-\bmu_0\Vert^2+c_{\sigma}|\sigma_1-\sigma_0|^2}{2c_{\sigma}\sigma_0\sigma_1}\right)   .
\end{equation}
\end{restatable}
\begin{proof}
See Appendix \ref{apx:geod}.
\end{proof}

Although mathematically complicated, this \textbf{analytically tractable} formula has numerical complexity $\mathcal{O}(d)$, which becomes negligible compared to network training in higher dimensions, as evidenced by the training time reported in Section \ref{sec:numerics} and Section \ref{apx:numerics}.
We also have an anisotropic version that decouples the isotropic geodesic into $d$ independent one-dimensional geodesics (Appendix \ref{apx:aniso}).

In Theorem \ref{thm:CFMbound}, the CFM loss \eqref{eq:loss_epd} along the geodesic in Theorem \ref{thm:geod} decomposes into the FM loss and the variance of conditional vector field and is bounded from below.

In particular, we get the vector field $\bv_t$ \eqref{eq:VF_locationscale} that generates the exponential power probability path $\mP_q(\bmu_t, \sigma_t^2\bI)$, and the corresponding flow $\phi_t(\bx)$ \eqref{eq:affine_flow}.
Letting $\bmu_0=\bzero, \bmu_1=\bx_1$, $\sigma_0=1, \sigma_1 = \sigma_{\min}$, we have 
the conditional flow and vector field along the probabilistic geodesic path to feed in \eqref{eq:loss_epd}:
\begin{equation}\label{eq:pgvf}
\begin{aligned}
\psi_t(\bx) &= \bmu_t(\bx_1) + \sigma_t(\bx_1) \bx = x(t) \bx_1/\Vert\bx_1\Vert +  \frac{R}{\lambda} \sin\theta(t)\bx, \\
\frac{d}{dt} \psi_t(\bx) &= \bu_t(\psi_t(\bx)|\bx_1) = \left[-\sin\theta(t) \frac{\bx_1}{\Vert\bx_1\Vert} + \frac{1}{\lambda}\cos\theta(t) \bx\right] RL\sin\theta(t) .
\end{aligned}
\end{equation}

Interestingly, when $R=\lambda=\Vert\bx_1\Vert$, letting $\sigma_{\min}\downarrow 0$, we have $\mu_c=\frac{\Vert \bx_1\Vert^2+\lambda^2(\sigma_{\min}^2-1)}{2\Vert \bx_1\Vert} \to 0$. Therefore $x(t)\to R\cos\theta(t)$. The conditional flow along this special geodesic path reduces to
\begin{equation*}
\psi_t(\bx) = x(t) \bx_1/\Vert\bx_1\Vert +  \frac{R}{\lambda} \sin\theta(t)\bx \to
\cos\theta(t) \bx_1 + \sin\theta(t) \bx ,
\end{equation*}
which is exactly the sinusoidal path \eqref{eq:sino_interp} with $\theta(t) = \pi(1-t)/2$.


FM methods transfer samples such that the distribution they represent changes from pure noise (standard normal or exponential power) to the data distribution. The ultimate goal is to evolve distributions, not to transport samples. \emph{OT based approaches move samples along straight lines in the Euclidean space. However, they do not necessarily travel the shortest distances in the space of probability distributions.}
The leftmost panel of Figure \ref{fig:epd_geods} shows that the Fisher-Rao distance may be smaller than the Euclidean distance between two exponential power distributions $\mP_q(\bzero, \bI)$ and $\mP_q(\bm{1}, \sigma^2_{\min}\bI)$ for certain $q$'s. The other panels of Figure \ref{fig:epd_geods} demonstrate selected probabilistic geodesic paths (second from left) and contrast their corresponding flow paths for $q=1.0$ vs $q=2.0$ (rightmost two) -- smaller $q$ leads to a more stringent movement than larger $q$.

\paragraph{Limitation}
There is a numerical issue on the conditional vector field that aggravates with increasing dimensions. When $\sigma_1\to 0$, the geodesic length $|L|\to \infty$ as in \eqref{eq:Fisher-Rao-distance}. As $t\to 1$, $\sigma_t \to \sigma_{\min}\ll 1$, and the conditional vector field $\dot\psi_t$ in \eqref{eq:pgvf} fluctuates greatly in magnitude as the flow moves towards the data end. Therefore, the vector field learned by neural network would cause the resulted ODE stiff and numerically difficult to solve. Luckily, this limitation can be alleviated by a hybrid (HB):
\begin{equation*}
    \psi_t(\bx) =
    \begin{cases}
        \bmu_t(\bx_1) + \sigma_t(\bx_1) \bx, & t < t_c,  \\
        t \bx_1 + (1-t) \bx, & t \geq t_c .
    \end{cases}
\end{equation*}
where we use $t_c=0.85$ in numerical experiments.

\section{Numerical Results}\label{sec:numerics}

In this section, we compare the proposed probabilistic geodesic (PG) FM and its hybrid variant (HB) against baseline models with variance preserving (VP), optimal transport (OT) interpolant, and sinusoidal (sino) probability paths. To fully test their performance, we also include the state-of-the-art models Riemannian FM \citep[RFM][]{Chen_2024}, Metric FM \citep[MFM][]{Kapusniak_2024}, and Fisher FM \citep[FFM][]{Davis_2024} in the comparison. Three 2d benchmark examples including Swiss roll, moons, and checkerboard, higher dimensional simulations of an augmented eight-Gaussian and Neal's funnel distribution, handwritten digits, four scientific datasets (Appendix \ref{apx:scidata}), and a molecular generation problem are used in the test.
We refer to Wasserstein-2 distance ($W_2$), 
energy distance (ED), 
and maximum mean discrepancy (MMD) 
to measure the discrepancy between synthetic sample distribution and real data distribution.
One can refer to Appendix \ref{apx:setup} for more details of the setup.
All the computer code will be released.

\subsection{Simulations}

\subsubsection{2d Benchmarks: swiss roll, moons, and checkerboard}

\begin{figure}[t]
\centering
\includegraphics[height=0.2\linewidth, width=.195\linewidth]{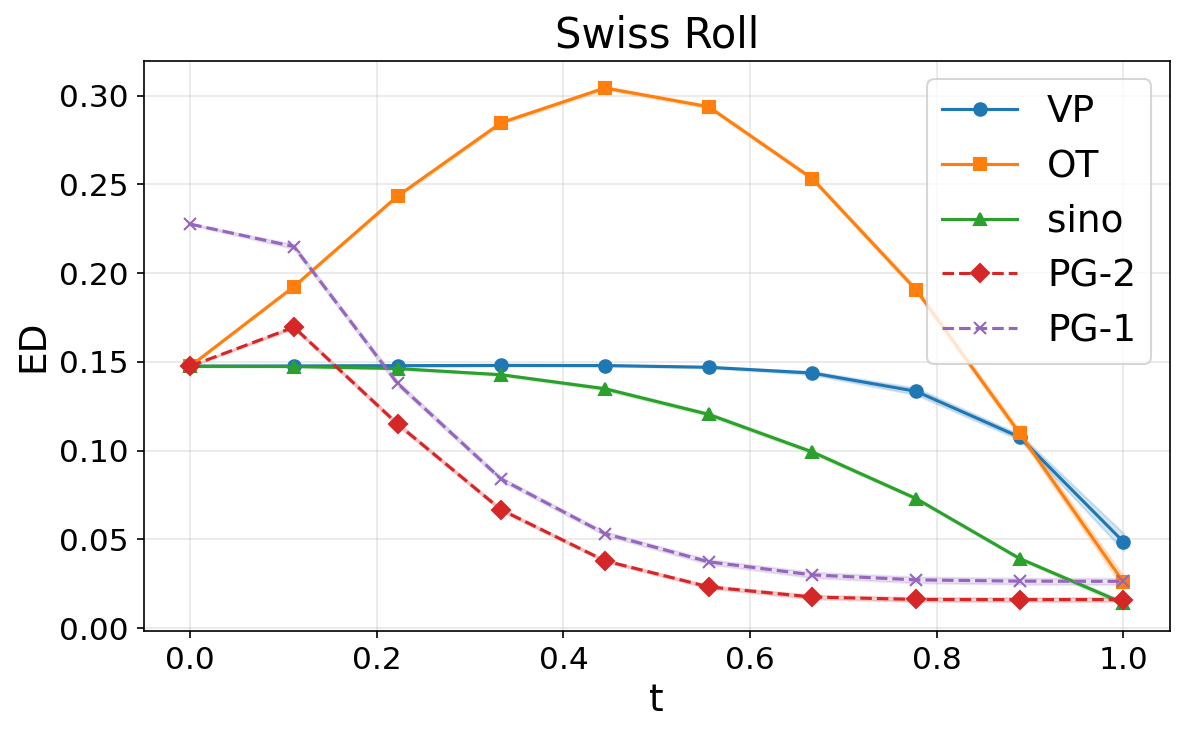}
\includegraphics[height=0.2\linewidth, width=.195\linewidth]{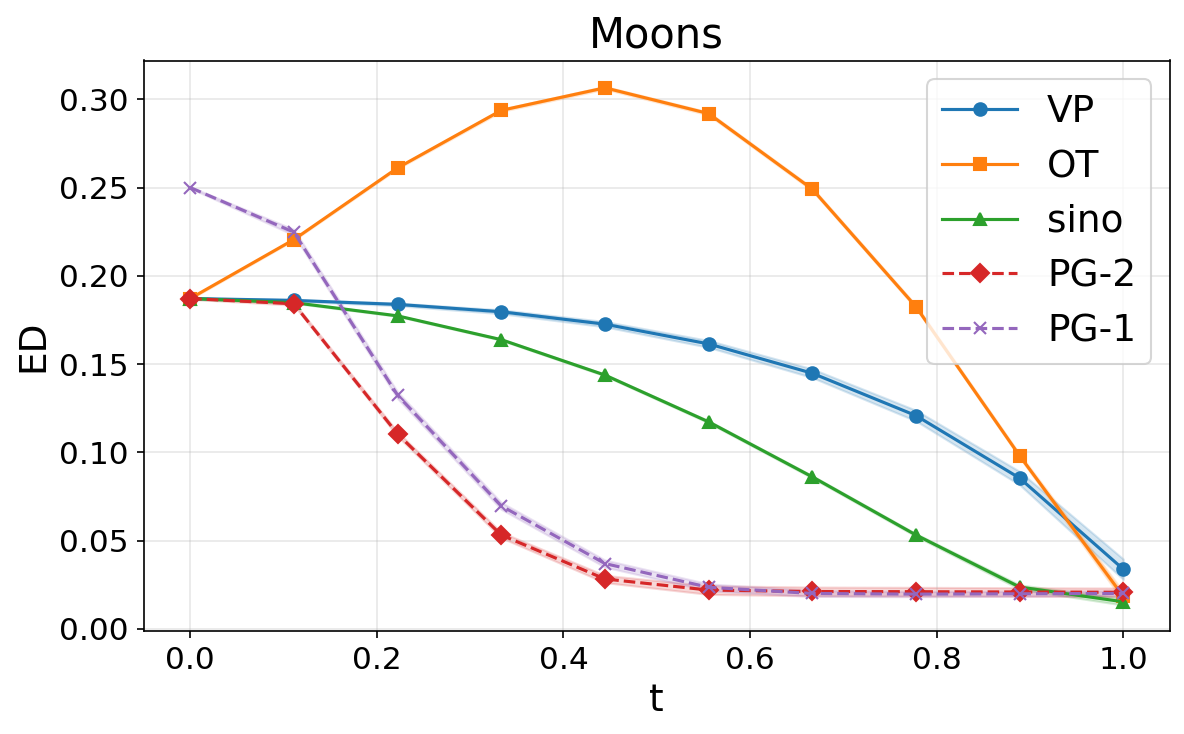}
\includegraphics[height=0.2\linewidth, width=.195\linewidth]{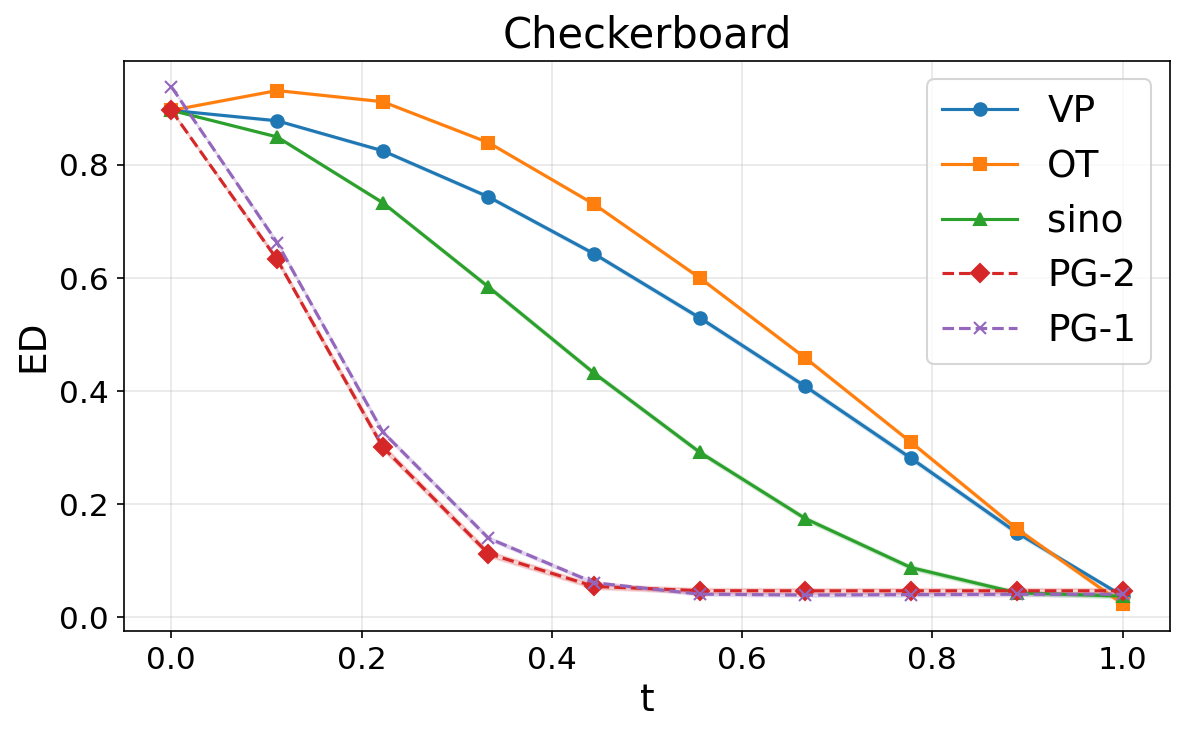}
\includegraphics[height=0.2\linewidth, width=.195\linewidth]{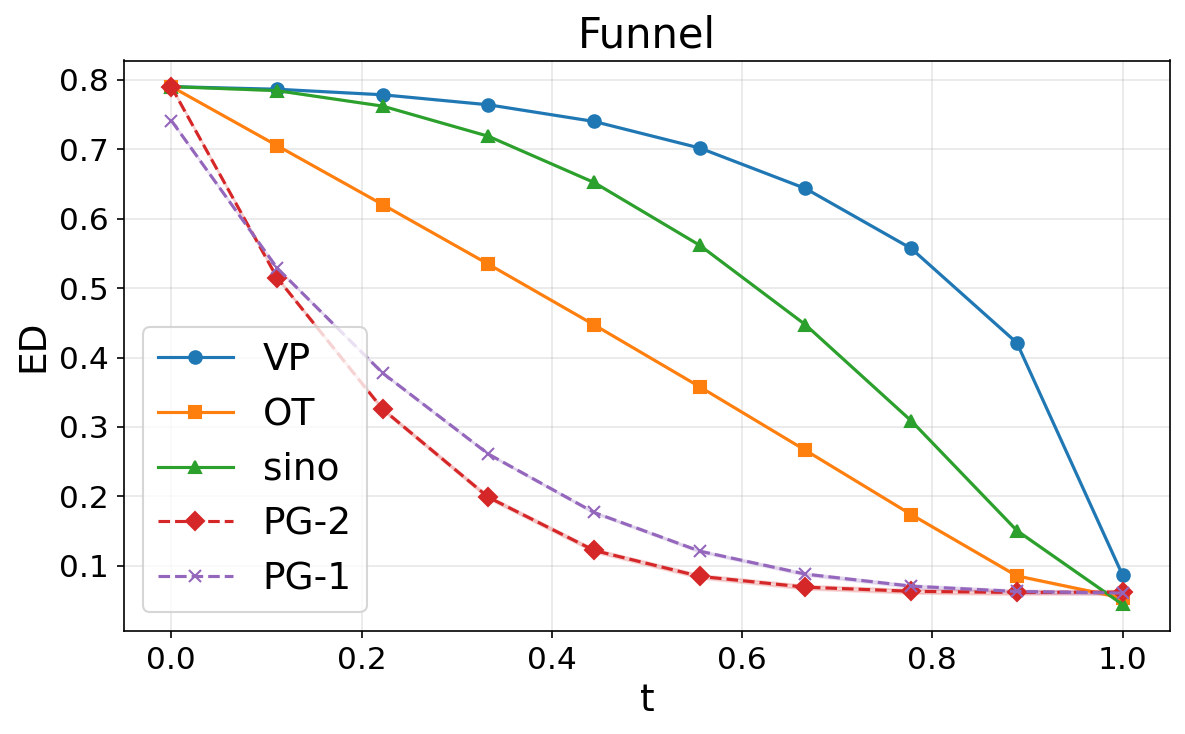}
\includegraphics[height=0.2\linewidth, width=.19\linewidth]{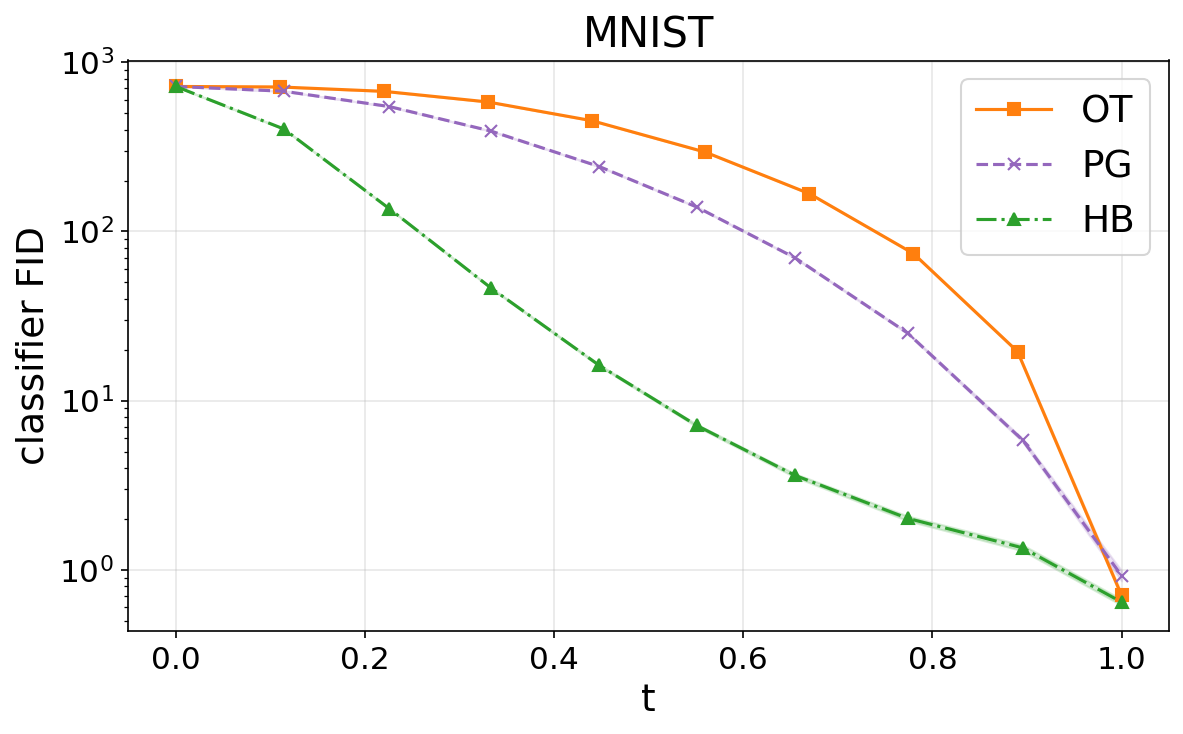}
\caption{Contrasting trajectories of energy distance (ED) in the Swiss roll, moons, checkerboard, and funnel examples (left four) and class Frechet inception distance (clf-FID) in MNIST (rightmost).}
\label{fig:ED_traj}
\end{figure}

First, we test VP, OT, sino, and PG-2 (on Gaussian distributions) and PG-1 (on exponential power distributions $\mP_1(\bmu_t, \sigma_t)$) on three 2d benchmark datasets: Swiss roll, moons and checkerboard \citep{grathwohl_2019}.
Figure \ref{fig:swissroll_moons_sample} demonstrates their sample paths on the Swiss roll dataset. It is clear that PG methods are more effective in transporting samples in which the rolling patterns emerge in earlier stages. Starting with the more concentrated standard exponential power noise, PG-1 learns the data feature even faster than the others based on standard Gaussian noise. Similar quick feature learning also shows in the other two examples (Figures \ref{fig:swissroll_moons_sample} and \ref{fig:checkerboard_funnel_sample}).
The metrics in terms of $W_2$, ED, and MMD in Tables \ref{tab:swissroll} and \ref{tab:moons_checkerboard} also support the numerical advantage (smaller generative discrepancy) of PG methods, of which some are one order of magnitude smaller. 
The time overhead (time/epoch) fades off as the dimension increases (see Tables \ref{tab:funnel100}, \ref{tab:funnel}, and \ref{tab:8gaussians}).














Figure \ref{fig:ED_traj} plots the trajectories of ED on these benchmark datasets and clf-FID on MNIST. OT interpolant with independent coupling ($\bx_0\perp \bx_1$) features an ``energy hump" (refer to Remark \ref{rk:energy_hump}) in the middle of the path due to the well-known midpoint (marginal) variance contraction \citep{Albergo_2023, Tong_2024}. This does not present an issue for our PG paths as their energy distances to the target distribution decrease quickly and monotonically.

\paragraph{Ablation on $q$.}
As illustrated in Figures \ref{fig:exponential_power} and Figure \ref{fig:epd_geods}, the parameter $q$ in the exponential power distribution plays a role of regularization: smaller $q$ induces a more strict movement of the sample. However, that does not necessarily lead to better performance when training becomes more difficult in the non-convex regime $(q<1.0)$. In Table \ref{tab:simulation_qs}, these discrepancy metrics generally decrease and then increase as $q$ changes from above 2 to below 1.

\subsubsection{Higher Dimensions: 36d eight-Gaussians, 100d Funnel Distribution}

To test the robustness of PG methods to higher dimensionality, we extend the classical eight-Gaussians example 
\citep{grathwohl_2019} from 2d to $d=36$ dimensions by filling in the remaining $d-2$ coordinates with isotropic Gaussian noise.
To test the capability of PG methods of handling heavy-tail data, we investigate the more challenging Neal's funnel distribution \citep{Neal_2003}
whose density $f(\bx, \nu) = \mN_d(\bx;\bzero, e^{\nu/2}\bI) \mN(\nu; 0, 3)$ features a funnel (Figure \ref{fig:checkerboard_funnel_sample}).

\begin{figure}[t]
\centering
\includegraphics[height=0.055\linewidth, width=.495\linewidth]{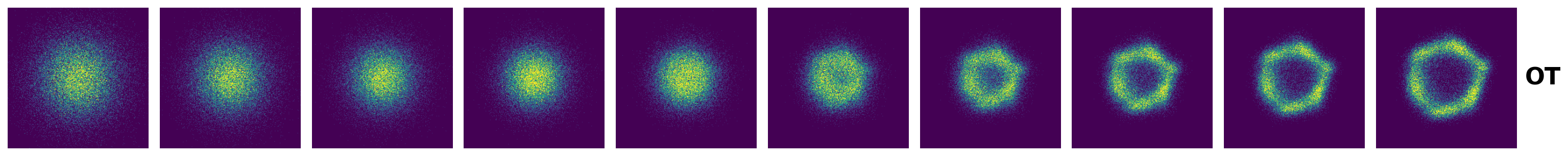}
\includegraphics[height=0.055\linewidth, width=.495\linewidth]{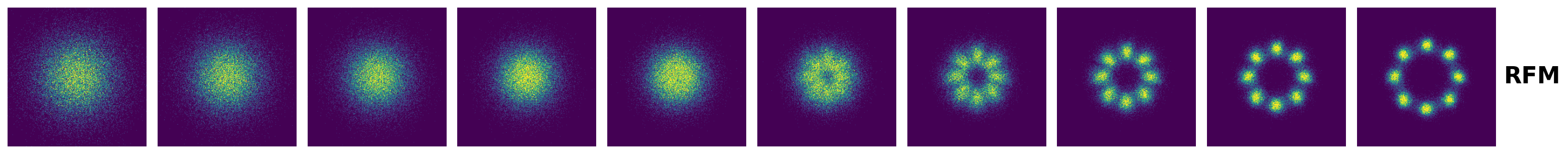}
\includegraphics[height=0.055\linewidth, width=.495\linewidth]{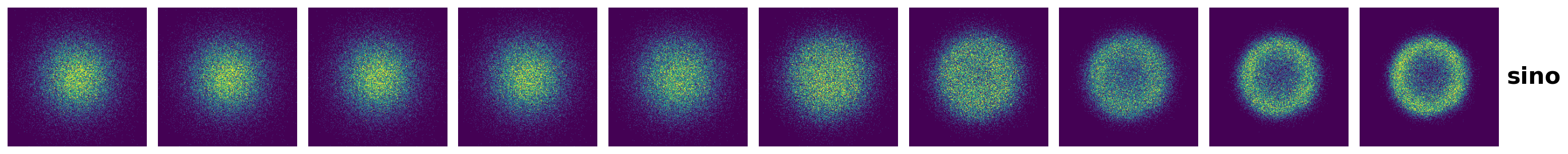}
\includegraphics[height=0.055\linewidth, width=.495\linewidth]{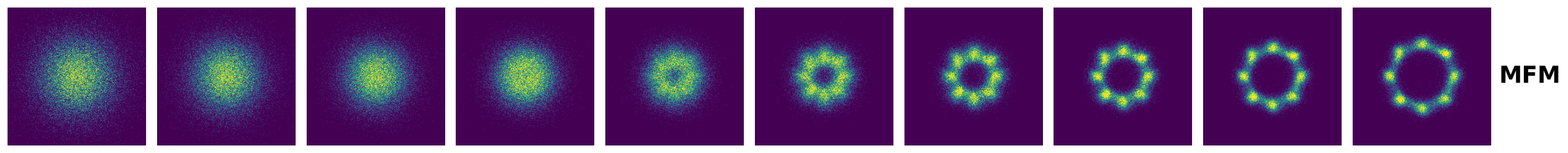}
\includegraphics[height=0.055\linewidth, width=.495\linewidth]{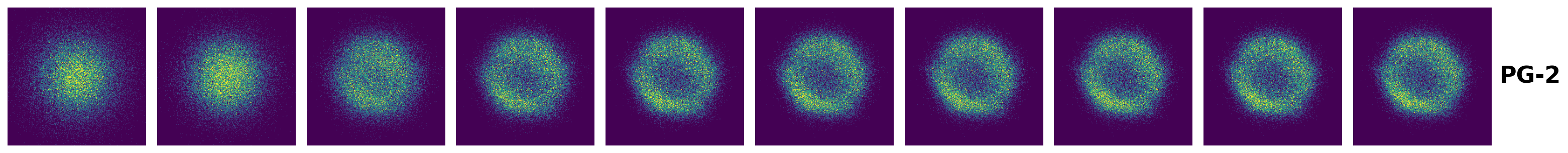}
\includegraphics[height=0.055\linewidth, width=.495\linewidth]{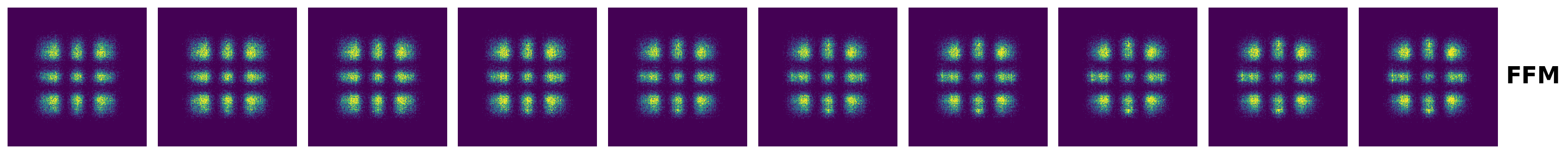}
\includegraphics[height=0.055\linewidth, width=.495\linewidth]{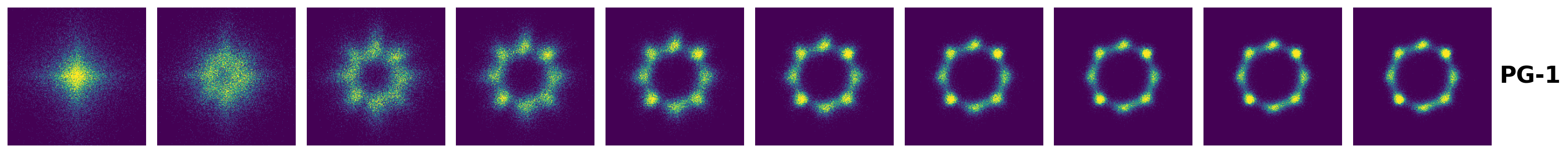}
\includegraphics[height=0.055\linewidth, width=.495\linewidth]{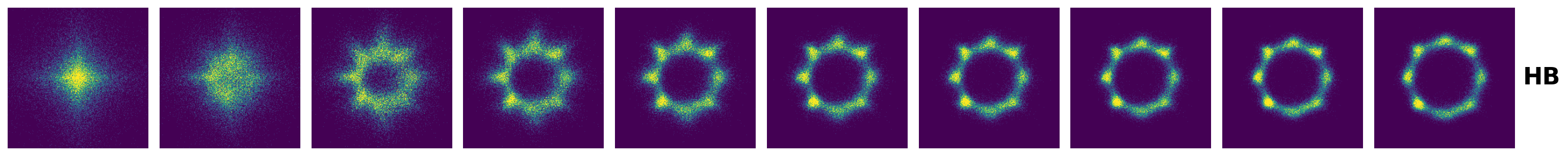}
\caption{Sample paths in the 8-Gaussians (36d) example. Only the first two dimensions are plotted.}
\label{fig:8gaussians_sample}
\end{figure}


\begin{table}[htbp]
\centering
\caption{Neal's funnel (100d): comparison in $W_2$, ED, and MMD (mean $\pm$ std over 10 seeds).}
\label{tab:funnel100}
\resizebox{\columnwidth}{!}{
\begin{tabular}{l l c c c c}
\toprule
DIST & Method & $W_2$ $\downarrow$ & ED $\downarrow$ & MMD $\downarrow$ & time/epoch \\
\midrule





q=1.0 & OT
& $3.73e{0} \pm 2.49e{-1}$
& $1.43e{-1} \pm 1.35e{-2}$
& $3.52e{-3} \pm 1.07e{-3}$
& $5.90e{-2} \pm 6.85e{-5}$ \\


q=2.0 & OT
& $3.74e{0} \pm 2.00e{-1}$
& $1.41e{-1} \pm 1.18e{-2}$
& $3.32e{-3} \pm 7.60e{-4}$
& $5.76e{-2} \pm 8.21e{-5}$ \\





\hline

q=1.0 & RFM
& $3.65e{0} \pm 2.01e{-1}$
& $1.27e{-1} \pm 1.33e{-2}$
& \cellcolor{lightgray}$2.30e{-3} \pm 7.17e{-4}$
& $7.68e{-2} \pm 6.59e{-4}$ \\

q=2.0 & RFM
& $3.69e{0} \pm 2.12e{-1}$
& $1.29e{-1} \pm 1.10e{-2}$
& $2.48e{-3} \pm 5.29e{-4}$
& $7.11e{-2} \pm 5.74e{-4}$ \\

\hline

q=1.0 & MFM
& $4.05e{0} \pm 2.39e{-1}$
& \cellcolor{lightgray}$1.26e{-1} \pm 1.16e{-2}$
& $2.65e{-3} \pm 7.57e{-4}$
& $4.36e{-1} \pm 1.57e{-2}$ \\

q=2.0 & MFM
& $4.17e{0} \pm 1.83e{-1}$
& $1.27e{-1} \pm 1.31e{-2}$
& $2.80e{-3} \pm 5.95e{-4}$
& $4.26e{-1} \pm 1.62e{-2}$ \\

\hline

q=1.0 & FFM
& $6.18e{0} \pm 5.90e{-1}$
& $1.00e{0} \pm 7.25e{-2}$
& $1.16e{-1} \pm 8.27e{-3}$
& $1.00e{-2} \pm 2.52e{-5}$ \\

q=2.0 & FFM
& $5.55e{0} \pm 6.64e{-1}$
& $8.99e{-1} \pm 7.95e{-2}$
& $1.01e{-1} \pm 8.49e{-3}$
& $8.69e{-3} \pm 1.38e{-5}$ \\





\hline

q=1.0 & PG
& $3.76e{0} \pm 3.06e{-1}$
& $1.82e{-1} \pm 1.04e{-2}$
& $5.94e{-3} \pm 8.64e{-4}$
& $6.51e{-2} \pm 1.39e{-4}$ \\


q=2.0 & PG
& $4.58e{0} \pm 6.98e{-1}$
& $1.62e{-1} \pm 2.64e{-2}$
& $4.48e{-3} \pm 1.34e{-3}$
& $6.33e{-2} \pm 1.08e{-4}$ \\

\hline

q=1.0 & HB
& \cellcolor{lightgray}$2.99e{0} \pm 2.38e{-1}$
& \cellcolor{lightgray}$1.26e{-1} \pm 1.69e{-2}$
& $2.39e{-3} \pm 6.59e{-4}$
& $6.92e{-2} \pm 2.69e{-4}$ \\


q=2.0 & HB
& $3.67e{0} \pm 3.27e{-1}$
& $1.84e{-1} \pm 1.25e{-1}$
& $8.34e{-3} \pm 1.55e{-2}$
& $6.61e{-2} \pm 3.42e{-4}$ \\

\bottomrule
\end{tabular}}
\end{table}

\paragraph{Why non-Gaussianity matters?}
Figure \ref{fig:8gaussians_sample} compares the sample paths of the baseline (OT, sino) and SOTA (RFM, MFM, FFM) models with the proposed PG and HB models.
It challenges all the methods to discover eight 2d Gaussian blobs in 36 dimensional ambient space, yet PG-1 and HB are the fastest to identify all the modes. What is more, the exponential power distribution's heavier tail property (with smaller $q$ as in PG-1) helps identify and isolate those modes faster and better than Gaussian ($q=2$ as in PG-2) which tends to smooth and blur the pattern.  This explains why most $q=1.0$ rows have lower discrepancy metrics than $q=2.0$ (except RFM and MFM) in Table \ref{tab:8gaussians} for which PG-1
achieves the lowest ED and MMD.  
In Table \ref{tab:funnel100},
except for a few cells, $q=1.0$ almost unanimously beats $q=2.0$, which again justifies the adoption of exponential power distribution.
HB attains the lowest $W_2$ and ED for $q=1.0$.
Note that MFM reaches the same lowest ED at a much higher time cost ($4.36e-1$ vs $6.92e-2$ seconds per epoch for HB).

\subsection{MNIST Handwritten Digits}
In this subsection, we consider the MNIST dataset \citep{Lecun_1998} consisting of 60,000 training handwritten digits and 10,000 testing handwritten digits with $d=28\times 28=784$. 
The sharp contrast along the edge of these digits challenges FM models.
Here we focus on comparing the OT, PG, and HB paths.

\vspace{4pt}
\noindent\begin{minipage}[c]{0.44\textwidth}

\centering
\captionsetup{hypcap=false}
\includegraphics[width=1\linewidth, height=0.6\linewidth]{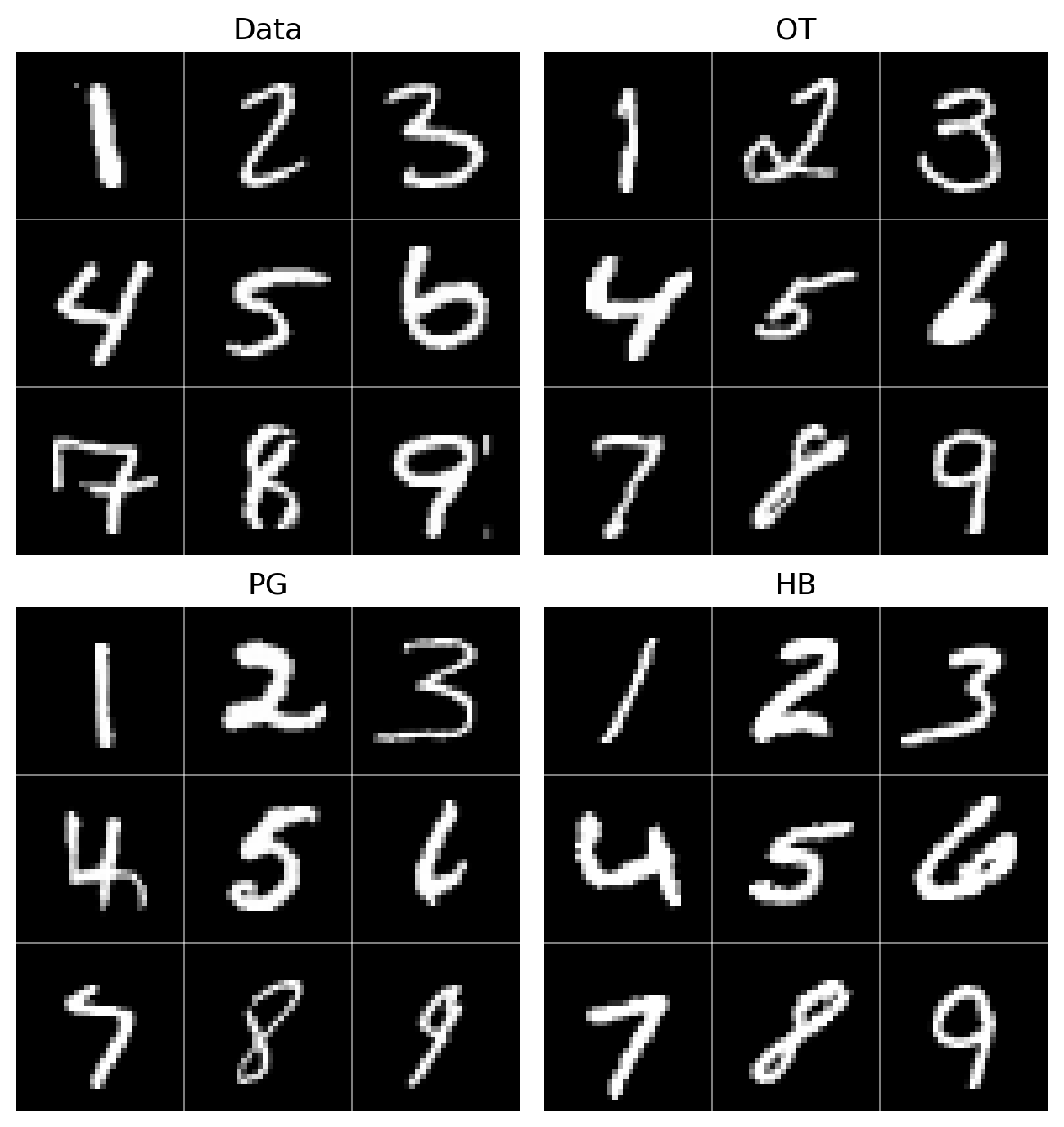}
 \vspace{-20pt}
\captionof{figure}{sample handwritten digits.}
\label{fig:mnist_sample}

\end{minipage}
\hfill
\begin{minipage}[c]{0.55\textwidth}
\centering
\captionsetup{hypcap=false}
\captionof{table}{Comparison of generation quality metrics across different path
methods. Values are reported as mean $\pm$ standard deviation over
$n=5$ runs. Higher IS is better, while lower FID and clf-FID are better.}
\vspace{-10pt}
\label{tab:mnist}
\resizebox{\linewidth}{!}{%
\begin{tabular}{llccc}
\toprule
DIST & PATH & IS $\uparrow$ & FID $\downarrow$ & clf-FID $\downarrow$ \\
\midrule

q=1.0 & OT
& $2.099 \pm 0.008$
& $1.606 \pm 0.037$
& $0.718 \pm 0.108$ \\

q=2.0 & OT
& \cellcolor{lightgray}$2.107 \pm 0.003$
& \cellcolor{lightgray}$1.597 \pm 0.019$
& $0.770 \pm 0.081$ \\

\hline

q=1.0 & PG
& $2.095 \pm 0.009$
& $1.671 \pm 0.013$
& $0.909 \pm 0.139$ \\

q=2.0 & PG
& $2.083 \pm 0.006$
& $2.281 \pm 0.097$
& $1.843 \pm 0.279$ \\

\hline

q=1.0 & HB
& \cellcolor{lightgray}$2.107 \pm 0.008$
& $1.727 \pm 0.029$
& $0.567 \pm 0.119$ \\

q=2.0 & HB
& $2.102 \pm 0.008$
& $1.803 \pm 0.040$
& \cellcolor{lightgray}$0.556 \pm 0.071$ \\

\bottomrule
\end{tabular}}

\end{minipage}\par
\vspace{4pt}

In Table~\ref{tab:mnist}, we evaluate the sample quality of OT, PG, and HB
under both $q=1$ and $q=2$ in terms of the inception score (IS) and the
Fr\'echet inception distance (FID), based on ImageNet-pretrained
Inception-V3 activations. Additionally, we report a classifier FID
(clf-FID), computed identically but on penultimate-layer features from
a LeNet-5~\cite{Lecun_1998} classifier trained on MNIST, which
makes the metric sensitive to domain-specific structure that inception
features do not capture. 
While OT achieves the highest IS and the lowest FID, HB attains the highest IS (same) and the  lowest clf-FID. 
Figure \ref{fig:mnist_sample} illustrates some generated digits.

\subsection{Molecular Generation}
Lastly, we consider a real data example of thousands of dimensions -- de novo molecule generation \citep[Section 4.4 of][]{Davis_2024}, which aims to generate molecules positions unconditionally based on QM9 data \citep{Ruddigkeit_2012, Ramakrishnan_2014}. 
Refer to Appendix \ref{apx:molecule} for a detailed description of the problem.
Table \ref{tab:molecular-generation} reports the percentage of stable atoms within molecules (Atoms S), valid molecules (Mols Val), and stable molecules (Mols. S) based on 1,000 molecules (first 4 rows quoted from Table 3 of \cite{Davis_2024}).
For a fare comparison, all the noise distributions are Gaussian. PG/HB is slightly lower in the percentage of valid molecules (Mols Val), but has higher percentages of stable atoms within molecules (Atoms S) and stable molecules (Mols. S) than OT, FFM and FlowMol. 
Figure \ref{fig:QM9_sample} shows synthetic molecules generated by PG.

\vspace{4pt}
\noindent\begin{minipage}[c]{0.44\textwidth}

\centering
\captionsetup{hypcap=false}
\includegraphics[width=1\linewidth, height=0.6\linewidth]{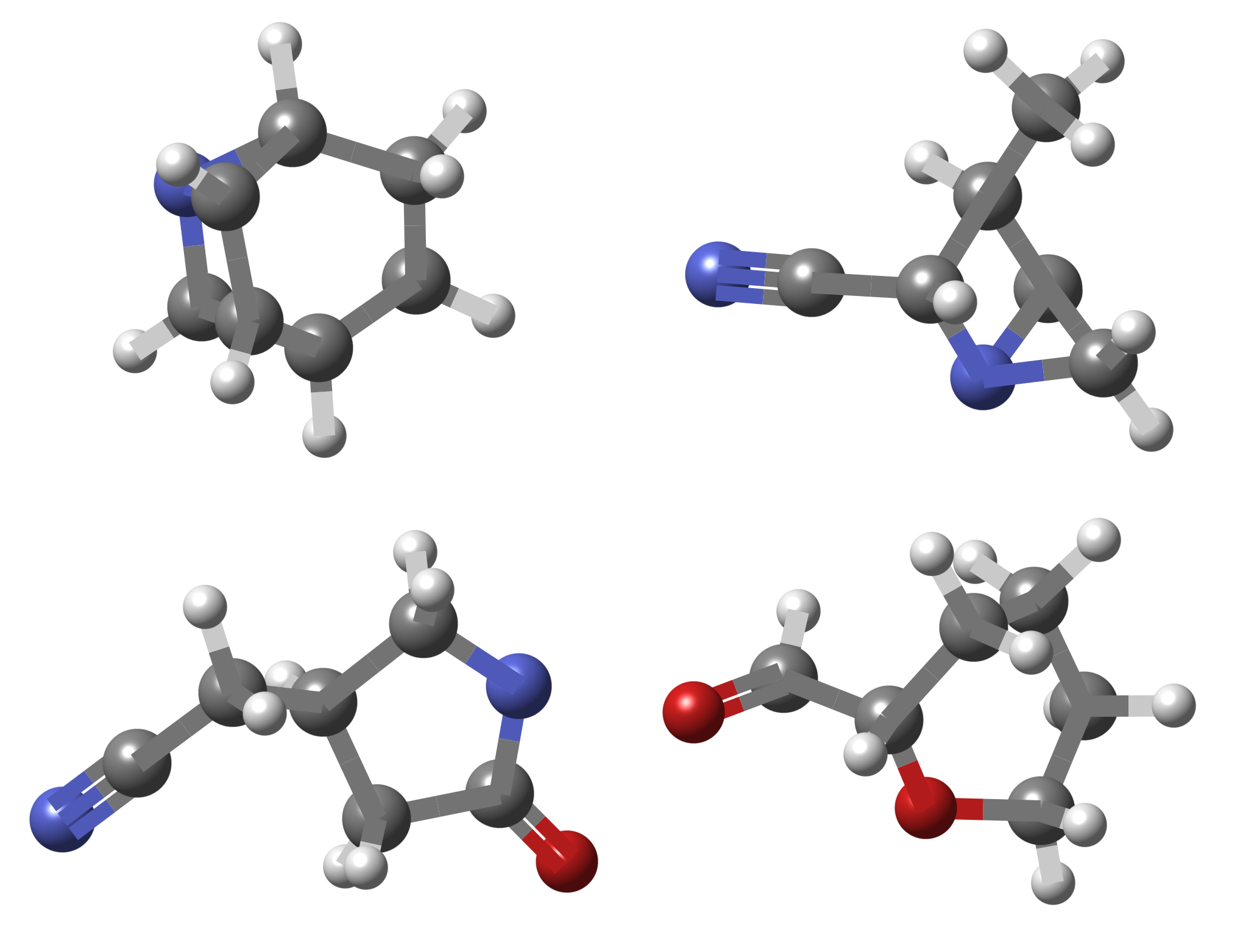}
 \vspace{-20pt}
\captionof{figure}{Generated molecules by PG-2.}
\label{fig:QM9_sample}

\end{minipage}
\hfill
\begin{minipage}[c]{0.55\textwidth}
\centering
\captionsetup{hypcap=false}
\captionof{table}{Results of molecular generation on QM9: the percentage of stable atoms within molecules (Atoms S), valid molecules (Mols Val), and stable molecules (Mols. S) based on 1,000 molecules.}
\vspace{-10pt}
\label{tab:molecular-generation}
\resizebox{\linewidth}{!}{%
\begin{tabular}{lccc}
\toprule
\textbf{Method} &
\textbf{Atoms S. (\%) $\uparrow$} &
\textbf{Mols. Val. (\%) $\uparrow$} &
\textbf{Mols. S. (\%) $\uparrow$} \\
\midrule
Fisher-Flow & 98.6           & 95.3          & 88.2          \\
JODO        & \textbf{99.4}  & \textbf{98.9} & \textbf{98.7} \\
EquiFM      & 99.4           & 94.4          & 93.2          \\
FlowMol     & 98.9           & \textbf{96.9} & 84.2          \\
\hline
OT          & 99.35          & 89.80         & 91.60         \\
\textit{PG} & \textbf{99.49} & 90.63         & \textbf{93.75}\\
\textit{HB} & 99.39          & 89.84         & 92.97         \\
\bottomrule
\end{tabular}}
\end{minipage}\par

\section{Conclusion}\label{sec:conclusion}
In this paper, we extend the conditional flow matching from Gaussian to a broader location-scale family with a particular focus on exponential power distributions that embrace Gaussian as a special case $q=2$. We propose a class of novel probability paths as geodesics on a manifold of probability distributions with Fisher-Rao metric. 
The proposed methods demonstrate fast feature learning capability with numerical performance superior or comparable to SOTA FM models.

The probabilistic geodesic path is not as conceptually straightforward as the OT interpolant path. 
The computational overhead diminishes as the dimension increases, compared to the cost of network training.
Future directions include integration of
the proposed probabilistic geodesic path into recent one-step methods \citep{Frans_2025, Geng_2026} to further boost the efficiency and effectiveness of FM.

\clearpage
\subsection*{AI use statement}


In this work, we used generative AI tools for polishing presentation of the paper (refining sentences and paragraphs).
We have not used generative AI tools for writing the whole paper.
In addition, we used generative AI tools for assistance in coding and debuging. 
We have reviewed all AI-assisted work. For example, the LLM-generated code was verified and tested for correctness by all co-authors. We take responsibility for the final content of this work, including text, claims or artifacts produced with the aid of generative AI.





\subsection*{Reproducibility statement}



All results are reproducible using \texttt{python} codes that will be published after publication. We include a \texttt{jupyter} notebook of a demonstrative example in the submitted supplementary materials. 



\def\url#1{}
\def\urlprefix{}
\def\doi#1{}


\medskip

{
\small

 \bibliographystyle{plain}
\bibliography{ref}

\begin{thebibliography}{10}

\bibitem{Albergo_2019}
M.~S. Albergo.
\newblock Flow-based generative models for markov chain monte carlo in lattice
  field theory.
\newblock {\em Physical Review D}, 100(3), 2019.

\bibitem{Albergo_2023}
Michael~Samuel Albergo and Eric Vanden-Eijnden.
\newblock Building normalizing flows with stochastic interpolants.
\newblock In {\em The Eleventh International Conference on Learning
  Representations}, 2023.

\bibitem{amari00}
S.~Amari and H.~Nagaoka.
\newblock {\em Methods of Information Geometry}, volume 191 of {\em
  Translations of Mathematical monographs}.
\newblock Oxford University Press, 2000.

\bibitem{Atanackovic_2025}
Lazar Atanackovic, Xi~Zhang, Brandon Amos, Mathieu Blanchette, Leo~J Lee,
  Yoshua Bengio, Alexander Tong, and Kirill Neklyudov.
\newblock Meta flow matching: Integrating vector fields on the wasserstein
  manifold.
\newblock In {\em The Thirteenth International Conference on Learning
  Representations}, 2025.

\bibitem{Benamou_2000}
Jean-David Benamou and Yann Brenier.
\newblock A computational fluid mechanics solution to the monge-kantorovich
  mass transfer problem.
\newblock {\em Numerische Mathematik}, 84(3):375--393, Jan 2000.

\bibitem{Chen_2024}
Ricky T.~Q. Chen and Yaron Lipman.
\newblock Flow matching on general geometries.
\newblock In {\em The Twelfth International Conference on Learning
  Representations}, 2024.

\bibitem{Chen_2018}
Ricky T.~Q. Chen, Yulia Rubanova, Jesse Bettencourt, and David~K Duvenaud.
\newblock Neural ordinary differential equations.
\newblock In S.~Bengio, H.~Wallach, H.~Larochelle, K.~Grauman, N.~Cesa-Bianchi,
  and R.~Garnett, editors, {\em Advances in Neural Information Processing
  Systems}, volume~31. Curran Associates, Inc., 2018.

\bibitem{Cheng_2024}
Chaoran Cheng, Jiahan Li, Jian Peng, and Ge~Liu.
\newblock Categorical flow matching on statistical manifolds.
\newblock In {\em Advances in Neural Information Processing Systems 37},
  NeurIPS 2024, pages 54787--54819. Neural Information Processing Systems
  Foundation, Inc. (NeurIPS), 2024.

\bibitem{Davis_2024}
Oscar Davis, Samuel Kessler, Mircea Petrache, Ismail~Ilkan Ceylan, Michael~M.
  Bronstein, and Joey Bose.
\newblock Fisher flow matching for generative modeling over discrete data.
\newblock In {\em The Thirty-eighth Annual Conference on Neural Information
  Processing Systems}, 2024.

\bibitem{fang1990generalized}
K.~Fang and Y.T. Zhang.
\newblock {\em Generalized Multivariate Analysis}.
\newblock Science Press, 1990.

\bibitem{Frans_2025}
Kevin Frans, Danijar Hafner, Sergey Levine, and Pieter Abbeel.
\newblock One step diffusion via shortcut models.
\newblock In {\em The Thirteenth International Conference on Learning
  Representations}, 2025.

\bibitem{Geng_2026}
Zhengyang Geng, Mingyang Deng, Xingjian Bai, J~Zico Kolter, and Kaiming He.
\newblock Mean flows for one-step generative modeling.
\newblock In {\em The Thirty-ninth Annual Conference on Neural Information
  Processing Systems}, 2026.

\bibitem{Gomez_1998}
E.~G{\'{o}}mez, M.A. Gomez-Viilegas, and J.M. Mar{\'{\i}}n.
\newblock A multivariate generalization of the power exponential family of
  distributions.
\newblock {\em Communications in Statistics - Theory and Methods},
  27(3):589--600, jan 1998.

\bibitem{grathwohl_2019}
Will Grathwohl, Ricky T.~Q. Chen, Jesse Bettencourt, and David Duvenaud.
\newblock Scalable reversible generative models with free-form continuous
  dynamics.
\newblock In {\em International Conference on Learning Representations}, 2019.

\bibitem{Ho_2020}
Jonathan Ho, Ajay Jain, and Pieter Abbeel.
\newblock Denoising diffusion probabilistic models.
\newblock In H.~Larochelle, M.~Ranzato, R.~Hadsell, M.F. Balcan, and H.~Lin,
  editors, {\em Advances in Neural Information Processing Systems}, volume~33,
  pages 6840--6851. Curran Associates, Inc., 2020.

\bibitem{Johnson_1987}
Mark~E. Johnson.
\newblock {\em Multivariate Statistical Simulation}, chapter 6 Elliptically
  Contoured Distributions, pages 106--124.
\newblock Probability and Statistics. John Wiley \& Sons, Ltd, 1987.

\bibitem{Kapusniak_2024}
Kacper Kapusniak, Peter Potaptchik, Teodora Reu, Leo Zhang, Alexander Tong,
  Michael Bronstein, Avishek Bose, and Francesco~Di Giovanni.
\newblock Metric flow matching for smooth interpolations on the data manifold.
\newblock In {\em Advances in Neural Information Processing Systems 37},
  NeurIPS 2024, pages 135011--135042. Neural Information Processing Systems
  Foundation, Inc. (NeurIPS), 2024.

\bibitem{Lecun_1998}
Y.~Lecun, L.~Bottou, Y.~Bengio, and P.~Haffner.
\newblock Gradient-based learning applied to document recognition.
\newblock {\em Proceedings of the IEEE}, 86(11):2278--2324, 1998.

\bibitem{Li_2023}
Shuyi Li, Michael O\textquotesingle~Connor, and Shiwei Lan.
\newblock Bayesian learning via q-exponential process.
\newblock In A.~Oh, T.~Naumann, A.~Globerson, K.~Saenko, M.~Hardt, and
  S.~Levine, editors, {\em Proceedings of the 37th Conference on Neural
  Information Processing Systems}, volume~36, pages 72867--72887. Curran
  Associates, Inc., 2023.

\bibitem{Lipman_2023}
Yaron Lipman, Ricky T.~Q. Chen, Heli Ben-Hamu, Maximilian Nickel, and Matthew
  Le.
\newblock Flow matching for generative modeling.
\newblock In {\em The Eleventh International Conference on Learning
  Representations}, 2023.

\bibitem{Lipman_2024}
Yaron Lipman, Marton Havasi, Peter Holderrieth, Neta Shaul, Matt Le, Brian
  Karrer, Ricky T.~Q. Chen, David Lopez-Paz, Heli Ben-Hamu, and Itai Gat.
\newblock Flow matching guide and code.
\newblock {\em arXiv:2412.06264}, 12 2024.

\bibitem{Maoutsa_2020}
Dimitra Maoutsa, Sebastian Reich, and Manfred Opper.
\newblock Interacting particle solutions of fokker--planck equations through
  gradient--log--density estimation.
\newblock {\em Entropy}, 22(8):802, July 2020.

\bibitem{MCCANN_1997}
Robert~J. McCann.
\newblock A convexity principle for interacting gases.
\newblock {\em Advances in Mathematics}, 128(1):153--179, 1997.

\bibitem{Neal_2003}
Radford~M. Neal.
\newblock Slice sampling.
\newblock {\em The Annals of Statistics}, 31(3), jun 2003.

\bibitem{Noe_2019}
Frank No{\'e}, Simon Olsson, Jonas K{\"o}hler, and Hao Wu.
\newblock Boltzmann generators: Sampling equilibrium states of many-body
  systems with deep learning.
\newblock {\em Science}, 365(6457):eaaw1147, 2019.

\bibitem{Pandey_2025}
Kushagra Pandey, Jaideep Pathak, Yilun Xu, Stephan Mandt, Michael Pritchard,
  Arash Vahdat, and Morteza Mardani.
\newblock Heavy-tailed diffusion models.
\newblock In {\em The Thirteenth International Conference on Learning
  Representations}, 2025.

\bibitem{Ramakrishnan_2014}
Raghunathan Ramakrishnan, Pavlo~O. Dral, Matthias Rupp, and O.~Anatole von
  Lilienfeld.
\newblock Quantum chemistry structures and properties of 134 kilo molecules.
\newblock {\em Scientific Data}, 1(1), August 2014.

\bibitem{Ruddigkeit_2012}
Lars Ruddigkeit, Ruud van Deursen, Lorenz~C. Blum, and Jean-Louis Reymond.
\newblock Enumeration of 166 billion organic small molecules in the chemical
  universe database gdb-17.
\newblock {\em Journal of Chemical Information and Modeling},
  52(11):2864--2875, November 2012.

\bibitem{Song_2021}
Yang Song, Jascha Sohl-Dickstein, Diederik~P Kingma, Abhishek Kumar, Stefano
  Ermon, and Ben Poole.
\newblock Score-based generative modeling through stochastic differential
  equations.
\newblock In {\em International Conference on Learning Representations}, 2021.

\bibitem{Tong_2024}
Alexander Tong, Kilian FATRAS, Nikolay Malkin, Guillaume Huguet, Yanlei Zhang,
  Jarrid Rector-Brooks, Guy Wolf, and Yoshua Bengio.
\newblock Improving and generalizing flow-based generative models with
  minibatch optimal transport.
\newblock {\em Transactions on Machine Learning Research}, 2024.
\newblock Expert Certification.

\bibitem{Villani_2009}
C{\'e}dric Villani.
\newblock {\em Optimal Transport}.
\newblock Springer Berlin Heidelberg, 2009.

\bibitem{Wolf_2018}
F.~Alexander Wolf, Philipp Angerer, and Fabian~J. Theis.
\newblock Scanpy: large-scale single-cell gene expression data analysis.
\newblock {\em Genome Biology}, 19(1):15, Feb 2018.

\end{thebibliography}

}



\clearpage
\appendix

\renewcommand{\theHequation}{appendix.\thesection.\arabic{equation}}

\renewcommand\theHtable{Appendix.\thetable}
\renewcommand{\theHfigure}{Appendix.\thefigure}
\counterwithin{table}{section}
\counterwithin{figure}{section}


\begin{table}[htbp]
\centering
\caption{Comparison of flow-matching methods.}
\label{tab:fm-comparison}
\resizebox{\columnwidth}{!}{
\begin{tabular}{l l l l l}
\toprule
\textbf{Method} &
\textbf{Manifold of} &
\textbf{Metric} &
\textbf{Geodesic} &
\textbf{Data type} \\
\midrule
Riemannian FM \citep{Chen_2024} &
general (mainly data) &
standard Riemannian &
general &
non-Euclidean \\

Metric FM \citep{Kapusniak_2024} &
Data &
learned &
general &
general \\

Meta FM \citep{Atanackovic_2025} &
probability densities &
$2$-Wasserstein &
general &
gmbedded in graph \\

Categorical FM \citep{Cheng_2024} &
probability measures &
Fisher-Rao &
circle &
discrete \\

Fisher FM \citep{Davis_2024} &
probability densities &
Fisher-Rao &
circle &
discrete \\

\emph{PG FM (ours)} &
probability densities &
Fisher-Rao &
Poincar\'e semicircle &
continuous, Euclidean \\
\bottomrule
\end{tabular}}
\end{table}

\section{Proofs}

\vf*
\begin{proof}[Proof of Theorem \ref{thm:vf}]
\label{apx:vf}
Because $\bZ\sim p_0$, we have
\begin{equation}\label{eq:prob_locscal}
\bX_t = \bmu_t + \sigma_t \bZ \sim p_t, \quad
p_t(\bx) = p_0((\bx-\bmu_t)/\sigma_t)\sigma_t^{-d} .
\end{equation}
Along this trajectory, we have
\begin{equation*}
\dot\bX_t = \dot\bmu_t + \dot\sigma_t \bZ = \dot\bmu_t + \frac{\dot\sigma_t}{\sigma_t} (\bX_t-\bmu_t) .
\end{equation*}

Let $\bv_t(\bx) = \dot\bmu_t + \frac{\dot\sigma_t}{\sigma_t} (\bx-\bmu_t)$.
The flow \eqref{eq:flow} with $\bv_t$ then follows a probability path $p_t$ \eqref{eq:prob_locscal} given by the pushforward \eqref{eq:pushfwd}.
This is equivalent to the continuity equation \eqref{eq:continuity}, which can also be verified.

On the one hand,
\begin{equation*}
\begin{aligned}
\pa_t p_t(\bx) &= -d\sigma_t^{-d-1}\dot\sigma_t p_0(\bz) + \sigma_t^{-d} \nabla p_0(\bz) \cdot (-\dot\bmu_t\sigma_t^{-1}-(\bx-\bmu_t)\sigma_t^{-2}\dot\sigma_t) \\
&= -d \frac{\dot\sigma_t}{\sigma_t} p_t(\bx) - \sigma_t^{-d-1}\nabla p_0(\bz) \cdot (\dot\bmu_t + \dot\sigma_t \bz), 
\end{aligned}
\end{equation*}
On the other hand,
\begin{equation*}
\begin{aligned}
-\nabla\cdot(\bv_tp_t) &= -(\nabla\cdot \bv_t)p_t(\bx) - \bv_t\cdot \nabla p_t(\bx) \\
&= -d \frac{\dot\sigma_t}{\sigma_t} p_t(\bx) - (\dot\bmu_t + \dot\sigma_t \bz)\cdot (\sigma_t^{-d-1}\nabla p_0(\bz))
\end{aligned}
\end{equation*}
Therefore, the two sides of \eqref{eq:continuity} are equal. Hence, the proof is completed.
\end{proof}

\path*
\begin{proof}[Proof of Theorem \ref{thm:path}]
\label{apx:path}
Let $\bu_t = \phi_t(\bx) - \bmu_t$. 
With the vector field $\bv_t$ \eqref{eq:VF_locationscale},
we substitute it into \eqref{eq:flow} to get
\begin{equation*}
\dot\bu_t = \dot\phi_t - \dot\bmu_t = \bv_t(\phi_t) - \dot\bmu_t = \dot\bmu_t + \frac{\dot\sigma_t}{\sigma_t}(\phi_t-\bmu_t) - \dot\bmu_t = \frac{\dot\sigma_t}{\sigma_t} \bu_t,
\end{equation*}
which already decouples across coordinates.
Integrating it as a vector ODE yields 
\begin{equation*}
\bu_t = \frac{\sigma_t}{\sigma_0} \bu_0 .
\end{equation*}
Since $\phi_0(\bx)=\bx$, $\bmu_0=\bzero$ and $\sigma_0=1$, we get $\bu_0=\bx$. Therefore $\bu_t=\sigma_t\bx$ and $\phi_t(\bx)=\bmu_t + \sigma_t \bx$.
Hence the proof is completed.
\end{proof}

\minengy*
\begin{proof}[Proof of Proposition \ref{prop:minengy}]
\label{apx:minengy}
We compute the energy \eqref{eq:energy} of the flow path
\begin{equation*}
\begin{aligned}
e(\bv) &= \sum_{i=1}^d \int_{\mbR^d} \int_0^1 p_t(\bx) v_t^2(x_i) dt d\bx \\
&= \sum_{i=1}^d \int_{\mbR} \int_0^1 p_t(x_i) [\dot\mu_{t,i}+(x_i-\mu_{t,i})\dot\sigma_t/\sigma_t]^2 dtdx_i \\
&= \int_0^1 \sum_{i=1}^d [(\dot\mu_{t,i})^2 + (\dot\sigma_t)^2] dt \\
&= \int_0^1 L(\bmu_t, \sigma_t, \dot\bmu_t, \dot\sigma_t) dt, \quad L(\bmu_t, \sigma_t, \dot\bmu_t, \dot\sigma_t) = \Vert\dot\bmu_t\Vert^2 + d (\dot\sigma_t)^2
\end{aligned}
\end{equation*}
By the calculus of variation, the Euler–Lagrange equations
\begin{equation*}
    \frac{d}{dt}\frac{\pa L}{\pa\dot\bmu_t} - \frac{\pa L}{\pa\bmu_t} = 0, \quad
\frac{d}{dt}\frac{\pa L}{\pa\dot\sigma_t} - \frac{\pa L}{\pa\sigma_t} = 0    
\end{equation*}
Solving these equations yields the desired geodesic path affine in $t$.

Substituting \eqref{eq:OT_probpath} into the above energy, we get the minimal value
\begin{equation*}
    \min e(\bv) = \Vert \bmu_1-\bmu_0\Vert^2 + d|\sigma_1-\sigma_0|^2 .
\end{equation*}
\end{proof}

\begin{lem}
\label{lem:frmetric}
The exponential power distribution $\mP_q(\bmu, \sigma^2\bI)$ has density
\begin{equation*}
\begin{aligned}
p(\bx) &= \frac{q\Gamma(\frac{d}{2})}{2\Gamma(\frac{d}{q})}
2^{-\frac{d}{q}}\pi^{-\frac{d}{2}} \sigma^{-d}
\exp\left\{-\frac{r^{q/2}}{2}\right\}, \\
r(\bx) &= \frac{\Vert\bx-\bmu\Vert^2}{\sigma^2}.
\end{aligned}
\end{equation*}
Its Fisher information matrix and Fisher-Rao metric are
\begin{equation}\label{eq:fisher_metric}
    g(\bmu, \sigma) = \frac{1}{\sigma^2}
    \begin{bmatrix}
    c_{\mu}\bI & \tp\bzero\\
    \bzero & c_{\sigma}
    \end{bmatrix}, \quad
 and\quad  ds^2 = \frac{c_{\mu}\Vert d\bmu\Vert^2+c_{\sigma} d\sigma^2}{\sigma^2} ,
\end{equation}
where $c_{\mu}(d,q) = \frac{2^{-\frac{2}{q}}}{d} \frac{\Gamma(\frac{d-2}{q} +2)}{\Gamma(\frac{d}{q})} q^2$, and $c_{\sigma}=qd$.
\end{lem}

\begin{proof}
\label{apx:frmetric}
By property 2 of Proposition \ref{prop:epd_prop},
$S=r^{\frac{q}{2}}=(\Vert\bx-\bmu\Vert/\sigma)^q\sim \Gamma(\alpha=\frac{d}{q}, \beta=\half)$ \citep{Gomez_1998, Li_2023}.
The log-density and their derivatives are
\begin{equation*}
\begin{aligned}
    l(\bmu, \sigma) =& 
    -d\log \sigma - \half r^{\frac{q}{2}} + \textrm{const.}\\
    \frac{\pa l}{\pa\bmu}(\bmu, \sigma) =& 
    \frac{q}{2\sigma} S^{1-\frac{1}{q}} \bu, \quad \bu=\frac{\bx-\bmu}{\Vert \bx-\bmu\Vert} \sim \mathrm{unif}(\mS^{d-1}) \\
    \frac{\pa l}{\pa\sigma}(\bmu, \sigma) =& 
    \frac{1}{\sigma}\left[-d + \frac{q}{2}S\right]\\
    \frac{\pa^2 l}{\pa\bmu\pa\sigma} =& -\frac{q^2}{2\sigma^2}S^{1-\frac{1}{q}} \bu\\
    \frac{\pa^2 l}{\pa\sigma^2}(\bmu, \sigma) =&
    \frac{1}{\sigma^2}\left[d - \frac{q}{2}(1+q) S\right]
\end{aligned}
\end{equation*}
By the moments of gamma distribution $\mbE[S^\alpha] = 2^\alpha\frac{\Gamma(\frac{d}{q}+\alpha)}{\Gamma(\frac{d}{q})}$, we have the following Fisher information matrix $\mI$:
\begin{equation*}
\begin{aligned}
\mI_{\bmu\bmu} &= \mbE\left[\frac{\pa l}{\pa\bmu} \tp{\frac{\pa l}{\pa\bmu}}\right] = \frac{c_{\mu}}{\sigma^2}\bI, \quad c_{\mu}(d,q) = \frac{2^{-\frac{2}{q}}}{d} \frac{\Gamma(\frac{d-2}{q} +2)}{\Gamma(\frac{d}{q})} q^2 \\
\mI_{\bmu\sigma} &= -\mbE\left[\frac{\pa^2 l}{\pa\bmu\pa\sigma}\right] = \bzero, \quad
\mI_{\sigma\sigma} =  \mbE\left[\left(\frac{\pa l}{\pa\sigma}\right)^2\right] =  -\mbE\left[\frac{\pa^2 l}{\pa\sigma^2}\right] = \frac{c_{\sigma}}{\sigma^2}, \;c_{\sigma}=qd
\end{aligned}
\end{equation*}
Therefore, the Fisher metric as written below is a hyperbolic metric
\begin{equation*}
    g(\bmu, \sigma) = \frac{1}{\sigma^2}
    \begin{bmatrix}
    c_{\mu}\bI & \tp\bzero\\
    \bzero & c_{\sigma}
    \end{bmatrix}, \quad
 or\quad  ds^2 = \frac{c_{\mu}\Vert d\bmu\Vert^2+c_{\sigma} d\sigma^2}{\sigma^2} .
\end{equation*}
\end{proof}

\geod*
\begin{rk}
When $q=2$, $c_{\mu}\equiv1$ and $c_{\sigma}=2d$. The above Fisher-Rao formula yields the well-known geodesic distance between two Gaussians $\mN(\bmu_0, \sigma_0^2\bI)$ and $\mN(\bmu_1, \sigma_1^2\bI)$:
\begin{equation*}
d_{FR} = \sqrt{2d}\arcosh\left(1+ \frac{\Vert \bmu_1-\bmu_0\Vert^2+2d|\sigma_1-\sigma_0|^2}{4d\sigma_0\sigma_1}\right) .
\end{equation*}
\end{rk}

\begin{proof}[Proof of Theorem \ref{thm:geod}]
\label{apx:geod}
The geodesic equation under the hyperbolic metric $ds^2 = \frac{c_{\mu}\Vert d\bmu\Vert^2+c_{\sigma} d\sigma^2}{\sigma^2}$ \eqref{eq:fisher_metric} can be solved in either semi-circle parametrization or hyperbolic parametrization, which are detailed as follows.

\paragraph{Poincar\'e semi-circle parametrization}

The metric \eqref{eq:fisher_metric} can be rescaled to a standard Poincar\'e metric which leads to a semi-circle in the Poincar\'e upper half-plane.
Denote $\lambda = \sqrt{c_\sigma/c_\mu}$ and $\be= (\bmu_1-\bmu_0)/\Vert \bmu_1-\bmu_0\Vert_2$.
Use the scalar projection $x(t) = \langle \bmu(t)-\bmu_0, \be\rangle$ and rescale $y(t) = \lambda \sigma(t)$. 
The metric \eqref{eq:fisher_metric} reduces to
\begin{equation*}
    ds^2 = \frac{c_{\sigma}}{y^2} (dx^2 + dy^2).
\end{equation*}
By calculus of variation ($L(x, y, \dot{x}, \dot{y}) = \frac{c_{\sigma}(\dot{x}^2+\dot{y}^2)}{y^2}$), we can get
\begin{equation*}
    (x-\mu_c)^2 + y^2 = R^2, \quad \mu_c = \frac{\Vert \bmu_1-\bmu_0\Vert^2+\lambda^2(\sigma_1^2-\sigma_0^2)}{2\Vert \bmu_1-\bmu_0\Vert}, \quad R = \sqrt{\mu_c^2+\lambda^2\sigma_0^2} .
\end{equation*}
where $\mu_c$ and $R$ are determined by the boundary conditions $(x_0, y_0) = (0, \lambda\sigma_0)$ and $(x_1, y_1)=(\Vert \bmu_1-\bmu_0\Vert, \lambda\sigma_1)$.
Then the semi-circle can be parametrized by angle $\theta(t)\in(0,\pi)$: $x(t) = \mu_c+ R\cos\theta(t), y(t) = R\sin\theta(t)$ with two end points $\theta(0)=\theta_0$ and $\theta(1)=\theta_1$ corresponding to $(x_0, y_0) = (0, \lambda\sigma_0)$ and $(x_1, y_1)=(\Vert \bmu_1-\bmu_0\Vert, \lambda\sigma_1)$ respectively:
\begin{equation*}
    \theta_0 = \atan2(\lambda\sigma_0, -\mu_c), \quad \theta_1 = \atan2(\lambda\sigma_1, \Vert \bmu_1-\bmu_0\Vert-\mu_c), 
\end{equation*}
where $\atan2(y, x) = \arg(x+ i y) \in(-\pi, \pi)$.

Note the arc-length along the Poincar\'e semi-circle has constant speed $L$: $ds = \frac{\sqrt{dx^2+dy^2}}{y} = \frac{R|d\theta|}{R\sin\theta} = L dt$.
Therefore, we have
\begin{equation*}
    \theta(t) = 2\arctan\left(e^{Lt}\tan\frac{\theta_0}{2}\right), \quad L = \log\frac{\tan(\theta_1/2)}{\tan(\theta_0/2)} .
\end{equation*}

Therefore, the final geodesic equation in the semi-circle parameterization of Poincar\'e becomes
\begin{align*}
\bmu(t) &= \bmu_0 + x(t) \be, \quad x(t) = \mu_c + R\cos\theta(t)   ;\\
\sigma(t) &= \frac{R}{\lambda} \sin\theta(t). 
\end{align*}

Substituting the geodesic solution back to the energy equation, we have
\begin{equation*}
e(\bv) = \int_{\gamma} ds^2 = c_{\sigma}\int_0^1 \frac{\dot{x}^2 + \dot{y}^2}{y^2} dt = c_{\sigma} \int_0^1 \frac{\dot{\theta}^2(t)}{\sin^2\theta(t)}dt = c_{\sigma} L^2 ,
\end{equation*}
where $|L|$ is the Fisher-Rao distance between two distributions, which can be rewritten in the original parameters as
\begin{equation*}
|L| 
= \arcosh\left(1+ \frac{c_{\mu}\Vert \bmu_1-\bmu_0\Vert^2+c_{\sigma}|\sigma_1-\sigma_0|^2}{2c_{\sigma}\sigma_0\sigma_1}\right) .
\end{equation*}

\paragraph{Hyperbolic parametrization}

With Einstein's notation, $g_{\mu_i\mu_j} = \frac{c_{\mu}}{\sigma^2}\delta_{ij}$, $g_{\sigma\sigma} = \frac{c_{\sigma}}{\sigma^2}$, we compute the Christoffel symbols $\Gamma_{\mu\nu}^{\lambda} = \half g^{\lambda\rho}[\pa_{\mu}g_{\nu\rho}+\pa_{\nu}g_{\mu\rho}-\pa_{\rho}g_{\mu\nu}]$ with inverse metric $g^{\mu_i\mu_j} = \frac{\sigma^2}{c_{\mu}}\delta_{ij}$, $g^{\sigma\sigma} = \frac{\sigma^2}{c_{\sigma}}$.
The only nonzero partial derivatives come from $\pa_{\sigma} g= -\frac{2}{\sigma}g$, giving
\begin{equation*}
    \Gamma_{\mu_j\sigma}^{\mu_i} = -\frac{1}{\sigma}\delta_{ij}, \quad \Gamma_{\mu_i\mu_j}^{\sigma} = \frac{c_{\mu}}{c_{\sigma}\sigma}\delta_{ij}, \quad \Gamma_{\sigma\sigma}^{\sigma} = -\frac{1}{\sigma} .
\end{equation*}

Therefore, the geodesic equations $\ddot{x}^{\lambda}+\Gamma_{\mu\nu}^{\lambda} \dot{x}^{\mu} \dot{x}^{\nu} = 0$ become
\begin{subequations}\label{eq:geod_qep}
\begin{align}
    \ddot{\mu}_i - \frac{2}{\sigma} \dot{\mu}_i \dot{\sigma}
 &=0 , \label{eq:geod_qep-1}\\
 \ddot{\sigma} + \frac{c_{\mu}}{c_{\sigma}\sigma}\Vert \dot{\bmu}\Vert^2 - \frac{1}{\sigma} \dot{\sigma}^2 &=0 . \label{eq:geod_qep-2}
 \end{align}
\end{subequations}
Solving \eqref{eq:geod_qep-1} yields $\dot{\bmu} = \bc \sigma^2$. Substituting it into \eqref{eq:geod_qep-2} gives $\ddot{\sigma} + \frac{c_{\mu}}{c_{\sigma}}\Vert \bc\Vert^2 \sigma^3 - \frac{1}{\sigma} \dot{\sigma}^2 =0$ .
By changing the variable $\sigma=e^\tau$, \eqref{eq:geod_qep-2} becomes a standard autonomous second-order ODE
\begin{equation*}
    \ddot\tau + \frac{c_{\mu}}{c_{\sigma}}\Vert \bc\Vert^2 e^{2\tau} = 0 ,
\end{equation*}
which implies the following conserved energy:
\begin{equation}\label{eq:geod_energy}
    \dot{\tau}^2 + \frac{c_{\mu}}{c_{\sigma}}\Vert \bc\Vert^2 e^{2\tau} = \frac{E}{c_{\sigma}}, \quad
    or\quad \frac{c_{\mu}\Vert \bc\Vert^2\sigma^4+c_{\sigma}\dot\sigma^2}{\sigma^2} = E .
\end{equation}

Denote $\omega = \sqrt{E/c_{\sigma}}$, $\kappa = \Vert \bc\Vert\sqrt{c_{\mu}/c_{\sigma}}$.
This energy equation \eqref{eq:geod_energy} can be rewritten as a separable ODE $\dot\sigma^2 = \omega^2\sigma^2 - \kappa^2\sigma^4$ which has the following solution
\begin{equation*}
    \sigma(t) = \frac{\omega}{\kappa}\sech(\omega t +\alpha) ,
\end{equation*}
where $\omega$, $\kappa$ and $\alpha$ can be determined by the boundary conditions.
Using \eqref{eq:geod_qep-1}, or equivalent $\dot{\bmu} = \bc \sigma^2$, one can solve
\begin{equation*}
    \bmu(t) = \bmu_0 + \bc \frac{\omega}{\kappa^2} (\tanh(\omega t+\alpha) - \tanh(\alpha) ) ,
\end{equation*}
where $\bc = (\bmu_1-\bmu_0)\frac{\kappa^2}{\omega(\tanh(\omega +\alpha) - \tanh(\alpha))}$.
We also have the conserved energy $E = c_{\sigma} \omega^2$ along the final geodesic: 
\begin{subequations}\label{eq:geodsoln_qep}
\begin{align}
\bmu(t) &= \bmu_0 + (\bmu_1-\bmu_0) \frac{\tanh(\omega t+\alpha) - \tanh(\alpha)}{\tanh(\omega +\alpha) - \tanh(\alpha)} ;\label{eq:geodsoln_qep-1}\\
\sigma(t) &= \frac{\omega}{\kappa}\sech(\omega t +\alpha) ,\label{eq:geodsoln_qep-2}
 \end{align}
\end{subequations}
where $\omega, \kappa, \alpha$ can be numerically solved from the following algebraic equations:
\begin{equation*}
\frac{\omega}{\kappa}\sech(\alpha) = \sigma_0, \quad \frac{\omega}{\kappa}\sech(\omega +\alpha) = \sigma_1, \quad
\kappa = \Vert\bmu_1-\bmu_0\Vert\frac{\kappa^2\sqrt{c_{\mu}/c_{\sigma}}}{\omega(\tanh(\omega +\alpha) - \tanh(\alpha))} .
\end{equation*}

\end{proof}

\begin{rk}
    The Poincar\'e semicircle and hyperbolic parametrizations are equivalent through the following identities:
    \begin{equation*}
        \sin(2\arctan e^{-s}) = \sech(s), \quad \cos(2\arctan e^{-s}) = \tanh(s), \quad s=\omega t + \alpha, \quad R = \frac{\omega}{\kappa} \lambda .
    \end{equation*}
    However, the former has fully explicit constants, and yet the latter contains implicit constants that need to be resolved numerically.

In particular, we get the vector field that generates the probability path by \eqref{eq:VF_locationscale}:
\begin{equation*}
\begin{aligned}
    \bv_t(\bx) &= \dot\bmu_t + \frac{\dot\sigma_t}{\sigma_t}(\bx-\bmu_t)
    =\bc\sigma_t^2 - \omega\tanh(\omega t +\alpha) [\bx -\bmu_0 - \bc \frac{\omega}{\kappa^2} (\tanh(\omega t+\alpha) - \tanh(\alpha) )]\\
    &=\omega \left[\frac{1-\tanh(\omega t +\alpha) \tanh(\alpha)}{\tanh(\omega +\alpha) - \tanh(\alpha)} (\bmu_1-\bmu_0) - \tanh(\omega t +\alpha) (\bx-\bmu_0) \right] .
\end{aligned}
\end{equation*}

Letting $\bmu_0=\bzero, \bmu_1=\bx_1$, $\sigma_0=1, \sigma_1 = \sigma_{\min}$, we have the conditional vector field
\begin{equation*}
    \bu_t(\bx|\bx_1) = \omega \left[\frac{1-\tanh(\omega t +\alpha) \tanh(\alpha)}{\tanh(\omega +\alpha) - \tanh(\alpha)} \bx_1 - \tanh(\omega t +\alpha) \bx \right] ,
\end{equation*}
and the conditional flow 
\begin{equation*}
    \psi_t(\bx) = \bx_1 \frac{\tanh(\omega t+\alpha) - \tanh(\alpha)}{\tanh(\omega +\alpha) - \tanh(\alpha)} + \frac{\omega}{\kappa}\sech(\omega t +\alpha) \bx .
\end{equation*}

\end{rk}

Let $p_t(\cdot|\bx_1)\sim \mP_q(\bmu_t(\bx_1), \sigma_t^2(\bx_1)\bI)$  
with $(\bmu_t, \sigma_t)$ as in Theorem \ref{thm:geod}. 
We have the marginal probability path $p_t(\bx) = \int p_t(\bx|\bx_1)q(\bx_1)d\bx_1$,
the \emph{marginal} vector field and the FM loss:
\begin{equation*}
\bu_t(\bx) = \mbE[\bu_t(\bx|\bx_1)|\bx_t=\bx]=\frac{\int \bu_t(\bx|\bx_1)p_t(\bx|\bx_1)q(\bx_1)d\bx_1}{p_t(\bx)}, \;
\mL_\textrm{FM}(\theta) = \mbE_{t, p_t(\bx)} \Vert \bv_t(\bx;\theta) - \bu_t(\bx)\Vert_2^2 .
\end{equation*}
The following theorem \ref{thm:CFMbound} states that the CFM loss \eqref{eq:loss_epd} along the geodesic in Theorem \ref{thm:geod} decomposes into the FM loss and the variance of conditional vector field and is bounded from below.

\begin{thm}
\label{thm:CFMbound}
For every parameter $\theta$ in the neural network $\bv_t(\cdot;\theta)$, we have
\begin{align*}
\mL_\textrm{CFM}(\theta) &= \mL_\textrm{FM}(\theta) + C, \quad
C:= \mbE_{t, p_t(\bx)}\V_{\bx_1|\bx_t=\bx}[\bu_t(\bx|\bx_1)] , \\
C &= \inf_{\theta} \mL_\textrm{CFM}(\theta) \leq \sigma^2_{\max} \kappa_{q,d} \mbE_{q(\bx_1)} L^2(\mP_q(\bzero, \bI), \mP_q(\bx_1,\sigma^2_{\min}\bI)) ,
\end{align*}
where $C$ does not depend on the neural network $\bv_t(\cdot;\theta)$, 
$\sigma_{\max}$, $\kappa_{q,d}$ are constants depending on the geodesic,
and $L^2(\mP_q(\bzero, \bI), \mP_q(\bx_1,\sigma^2_{\min}\bI))$ is the scaled square Fisher-Rao distance between $\mP_q(\bzero, \bI)$ and $\mP_q(\bx_1,\sigma^2_{\min}\bI)$.
\end{thm}

\begin{proof}
\label{apx:CFMbound}
Denote $\bx_t:=\psi_t(\bx_0)$ as the conditional flow \eqref{eq:affine_path}.
With $q(\bx_1)p_t(\bx_t|\bx_1) = p_t(\bx_t)p(\bx_1|\bx_t)$, the CFM loss \eqref{eq:loss_epd} can be rewritten
\begin{align*}
\mL_\textrm{CFM}(\theta) &= \mbE_t\mbE_{q(\bx_1)}\mbE_{p_t(\bx_t|\bx_1)}\Vert \bv_t(\bx_t;\theta)-\bu_t(\bx_t|\bx_1)\Vert_2^2 \\
&= \mbE_t\mbE_{p_t(\bx_t)}\mbE_{p(\bx_1|\bx_t)}\Vert \bv_t(\bx_t;\theta)-\bu_t(\bx_t|\bx_1)\Vert_2^2 \\
&= \underbrace{\mbE_t\mbE_{p_t(\bx_t)} \Vert \bv_t(\bx_t;\theta)-\bu_t(\bx_t)\Vert_2^2}_{\mL_\textrm{FM}(\theta)} + \underbrace{\mbE_t\mbE_{p_t(\bx_t)} \mbE_{p(\bx_1|\bx_t)}\Vert \bu_t(\bx_t|\bx_1) - \bu_t(\bx_t)\Vert_2^2}_{C} .
\end{align*}
where $C=\mbE_{t, p_t(\bx)}\V_{\bx_1|\bx_t=\bx}[\bu_t(\bx|\bx_1)]$ based on the definition of marginal vector field $\bu_t(\bx)$.
Since $\mL_\textrm{FM}(\theta)\geq 0$, we get $\mL_\textrm{CFM}(\theta)\geq C$, with equality achieved iff $\Vert\bv_t(\bx_t;\theta)-\bu_t(\bx_t)\Vert=0$ for $p_t$-a.e $(\bx, t)$.
Therefore, $C$ is the lower bound of the CFM loss \eqref{eq:loss_epd}, regardless of the network $\bv_t(\cdot; \theta)$.

To get the upper bound of $C$ for the geodesic path, we notice that with conditional vector field \eqref{eq:conditional_vf}
\begin{align*}
C &\leq \mbE_t\mbE_{q(\bx_1)}\mbE_{p_t(\bx|\bx_1)} \Vert\bu_t(\bx|\bx_1)\Vert^2 \\
&= \mbE_t\mbE_{q(\bx_1)}[\Vert\dot\bmu_t(\bx_1)\Vert^2 + (\dot\sigma_t(\bx_1)/\sigma_t(\bx_1))^2\mbE_{p_t(\bx|\bx_1)}\Vert\bx-\bmu_t(\bx_1)\Vert^2]\\
&\overset{prop\ref{prop:epd_prop}}{=} \mbE_{q(\bx_1)}\int_0^1[\Vert\dot\bmu_t(\bx_1)\Vert^2+s_{q,d}d\dot\sigma_t^2(\bx_1)] dt \\
&\leq \mbE_{q(\bx_1)} \int_0^1 \sigma_t^2\max\left\{\frac{1}{c_{\mu}}, \frac{s_{q,d}d}{c_{\sigma}}\right\} \frac{c_{\mu}\Vert\dot\bmu_t\Vert^2+c_\sigma \dot\sigma_t^2}{\sigma_t^2} dt\\
&\leq \sigma^2_{\max} \kappa_{q,d} \mbE_{q(\bx_1)} L^2(\mP_q(0, \bI), \mP_q(\bx_1,\sigma^2_{\min}\bI)) ,
\end{align*}
where $\sigma_{\max}:=\sup_{t\in[0,1], \bx_1\in\mathrm{supp}(q(\bx_1))}\sigma_t^2(\bx_1)$, and $\kappa_{q,d}=\max\left\{\frac{1}{c_{\mu}}, \frac{s_{q,d}d}{c_{\sigma}}\right\}$.
\end{proof}

\begin{rk}
Note that the CFM loss bound $C$ depends only on the probability path $p_t(\cdot|\bx)$ and the data distribution $q(\cdot)$, not the neural network $\bv(\cdot, \theta)$.
\end{rk}
\begin{rk}\label{rk:energy_hump}
With the OT-interpolant path \eqref{eq:OT_interp}, the CFM loss bound collapses to
\begin{equation*}
C=\mbE_{t, p_t(\bx)}\V_{\bx_1|\bx_t}[\bx_1]
\left(1+\frac{(1-\sigma_{\min})t}{\sigma_t}\right)^2.
\end{equation*}
This quantity peaks around $t\approx \half$, where the posterior $p(\bx_1|\bx_t)$ is most uncertain. Correspondingly, an ``energy hump" may appear, as illustrated in Figure \ref{fig:ED_traj}.
\end{rk}
\begin{rk}
The proof suggests a more natural choice of loss based on the energy \eqref{eq:energy_fisher}, i.e. replacing the $L_2$ norm with the Fisher-Rao norm in \eqref{eq:loss_epd}. That would yield a tighter bound for $C$.
\end{rk}

\begin{figure}[t]
\centering
\includegraphics[height=0.2\linewidth, width=.24\linewidth]{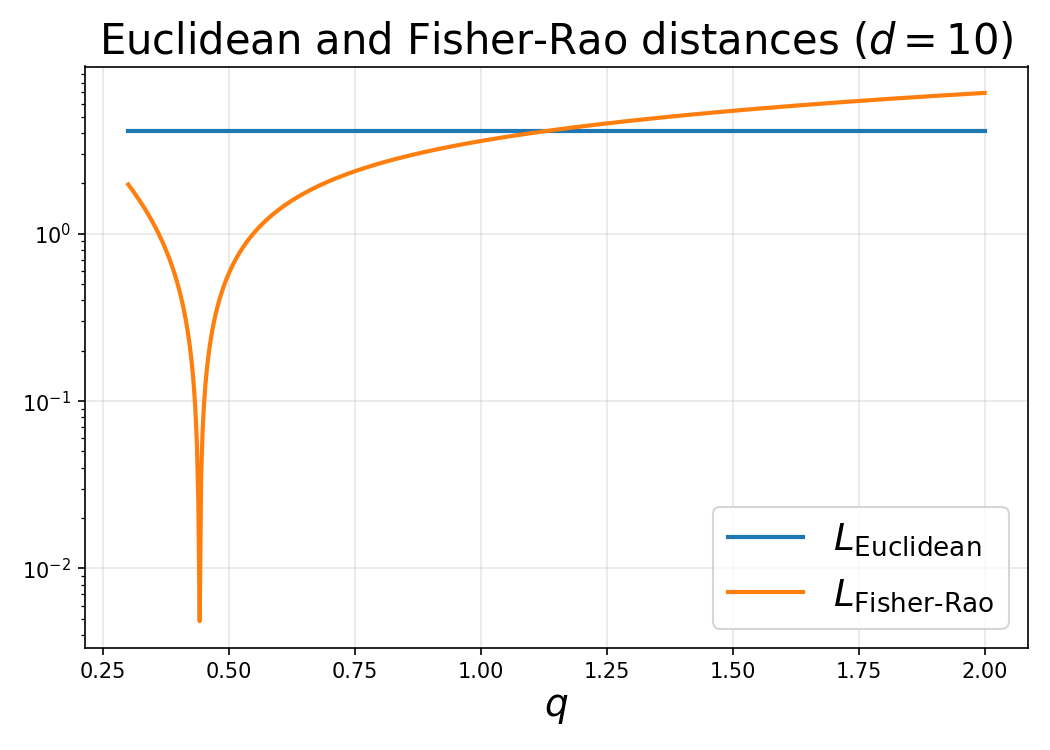}
\includegraphics[height=0.2\linewidth, width=.24\linewidth]{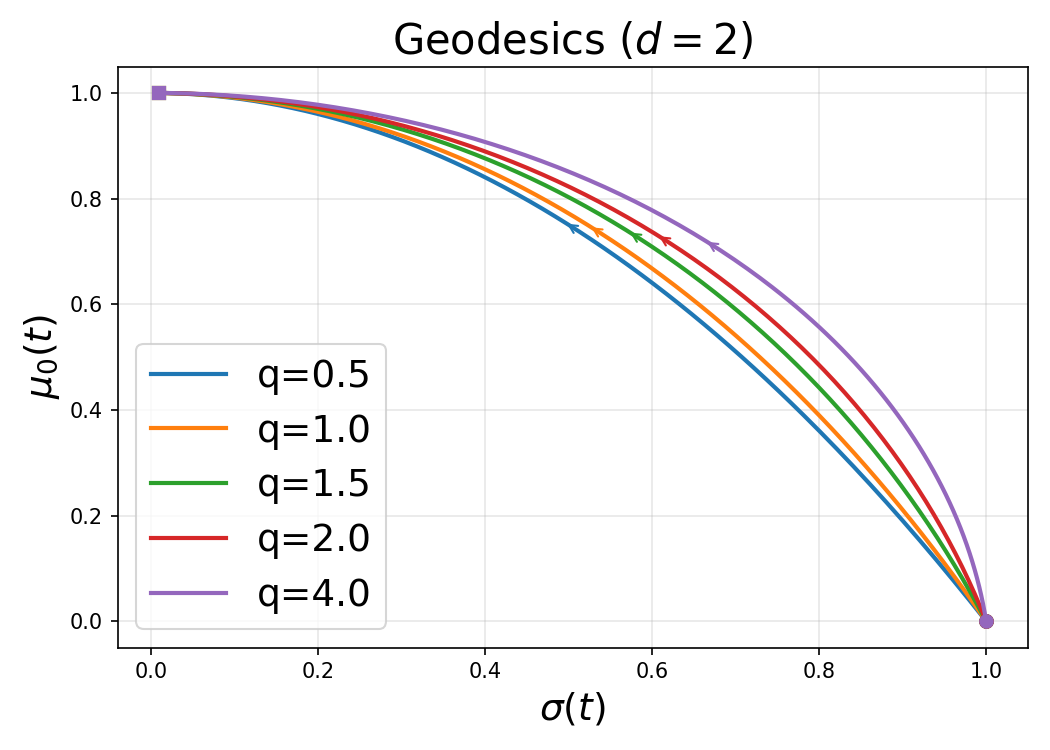}
\includegraphics[height=0.2\linewidth, width=.24\linewidth]{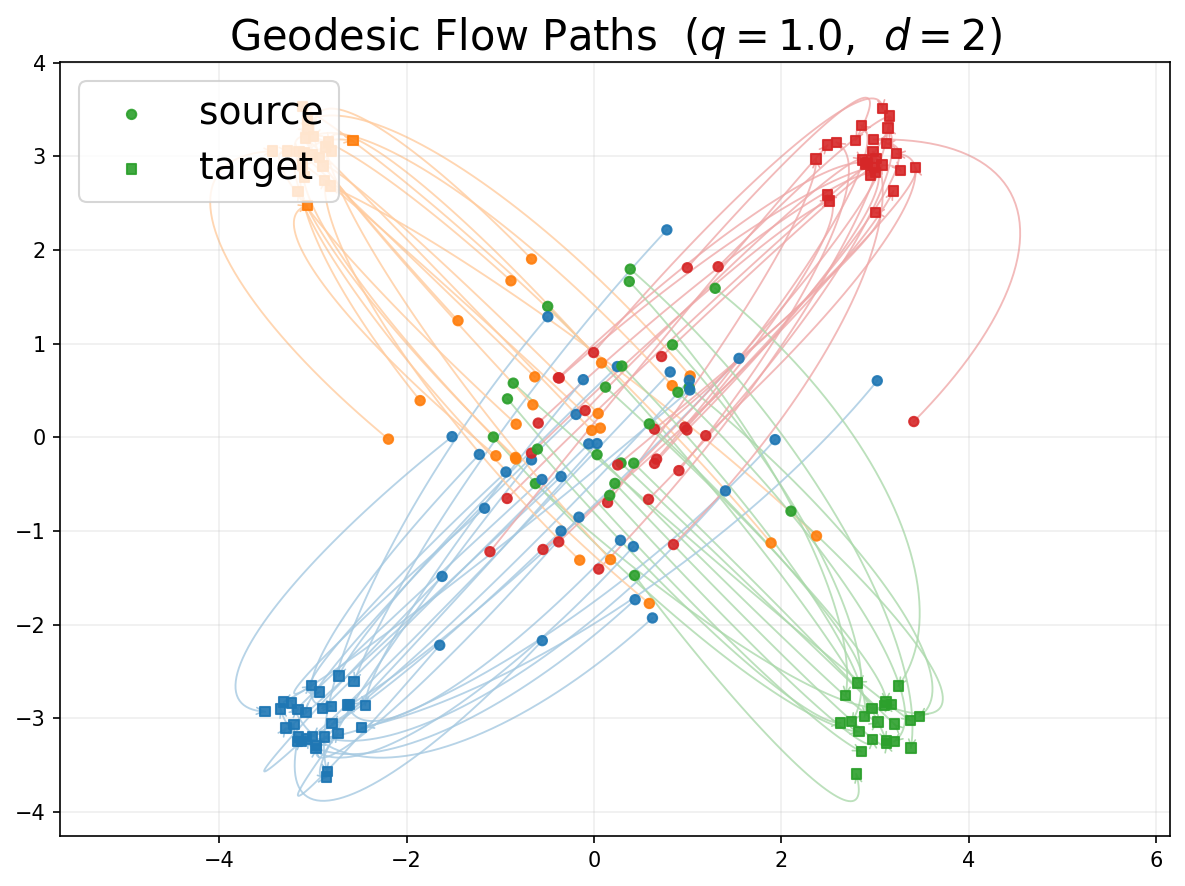}
\includegraphics[height=0.2\linewidth, width=.24\linewidth]{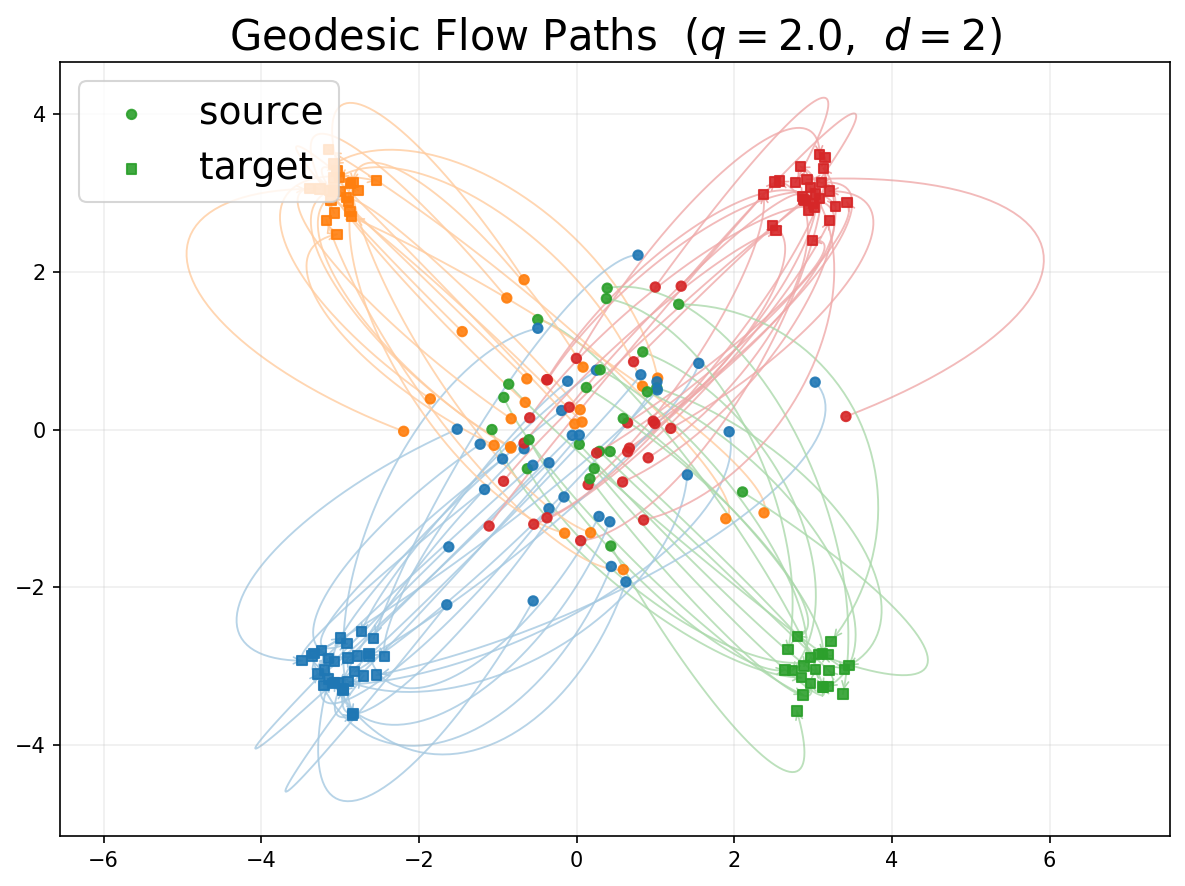}
\caption{Euclidean vs Fisher-Rao distances, (probabilistic) geodesic probability paths, flow paths for $q=1.0$ and $q=2.0$ (from left to right) on the manifold of exponential power distributions.}
\label{fig:epd_geods}
\end{figure}

\section{More Numerical Results}\label{apx:numerics}

\subsection{Setup for Numerical Experiments}\label{apx:setup}
We follow the framework of \cite[CFM]{Lipman_2023}. For all simulations, we use a 5-layer dense neural network with sinusoidal (Fourier) time encoding and 512 hidden dimensions and train it with the ADAM optimizer with learning rate $1e-4$. We set $\sigma_{\min}=1e-3$ for PG/HB and $1e-3$ for other baselines. 
In the sampling/inference stage, we adopt the `midpoint' method to solve the flow ODE \eqref{eq:flow} with step size 0.05. 

For image applications (MNIST), we use a classic UNet for MNIST (64 base channels, 2 residual blocks) with sinusoidal time embedding trained by AdamW at a learning rate $2e-4$. The exponential moving average (ema) is also adopted with decay rate 0.999. In sampling, the same `midpoint' method is used for solving ODE with 100 discretized steps. FID scores are evaluated based on 50k samples. 

We refer to Wasserstein-2 distance ($W_2^2(P, Q)=\inf_{\pi\in\Pi(P,Q)}\int_{\mbR^d\times\mbR^d}\Vert x-y\Vert_2^2 d\pi(x, y)$), energy distance ($\mathrm{ED}^2(P, Q)=2\mbE\Vert X-Y\Vert-\mbE\Vert X-X'\Vert-\mbE\Vert Y-Y'\Vert$ for $X,X'\overset{iid}{\sim} P$ and $Y,Y'\overset{iid}{\sim} Q$ independent), and maximum mean discrepancy ($\mathrm{MMD}^2(P,Q;k)=\mbE_{X,X'}k(X,X')+\mbE_{Y,Y'}k(Y,Y')-2\mbE_{X,Y}k(X,Y)$) with radial basis function (rbf) kernel $k$ to measure the discrepancy between synthetic sample distribution and real data distribution.
The training time per epoch, `time/epoch', is reported in Tables \ref{tab:funnel100}, \ref{tab:swissroll}, \ref{tab:8gaussians}, and \ref{tab:funnel}. The sampler for the exponential power distribution is based on Proposition \ref{prop:epd_prop}. 

\subsection{2d Benchmarks}

\begin{figure}[t]
\centering
\includegraphics[height=0.05\linewidth, width=0.49\linewidth]{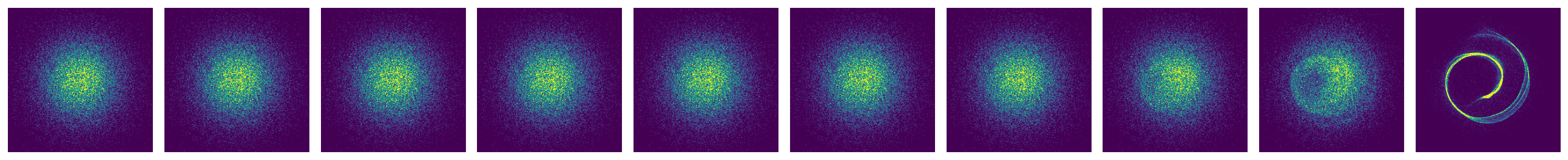}
\includegraphics[height=0.05\linewidth, width=0.49\linewidth]{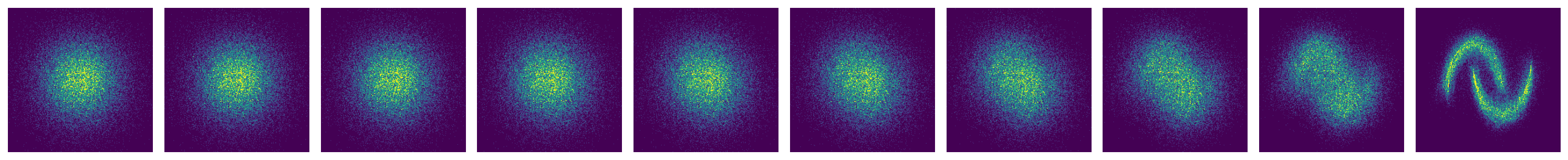}
\includegraphics[height=0.05\linewidth, width=0.49\linewidth]{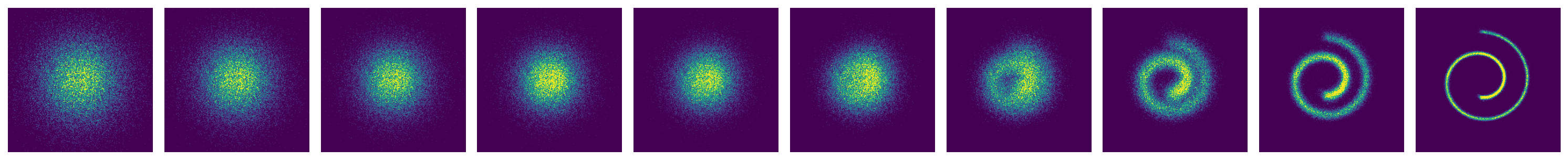}
\includegraphics[height=0.05\linewidth, width=0.49\linewidth]{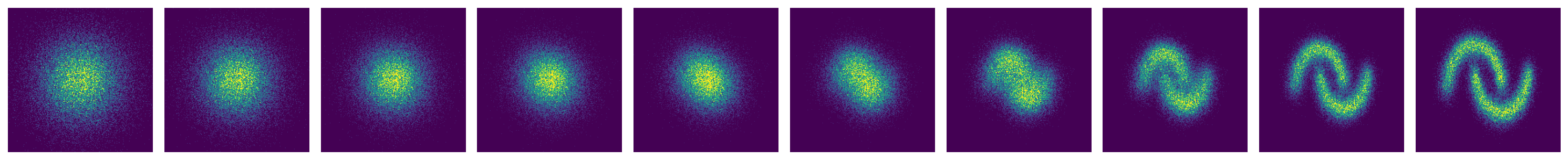}
\includegraphics[height=0.05\linewidth, width=0.49\linewidth]{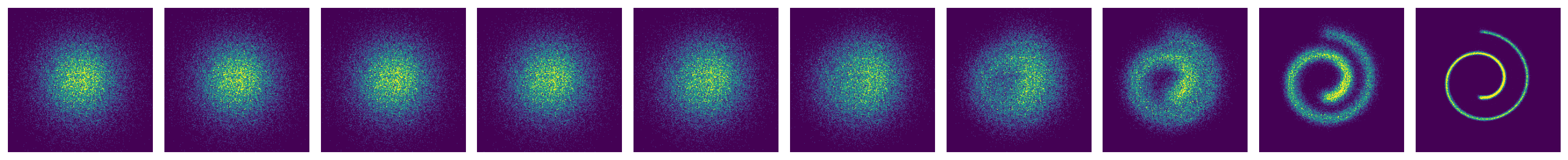}
\includegraphics[height=0.05\linewidth, width=0.49\linewidth]{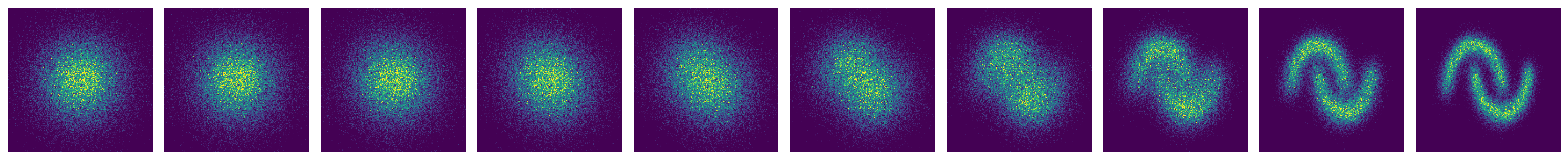}
\includegraphics[height=0.05\linewidth, width=0.49\linewidth]{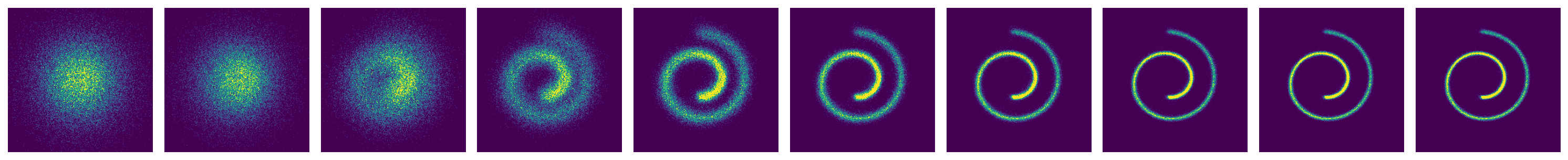}
\includegraphics[height=0.05\linewidth, width=0.49\linewidth]{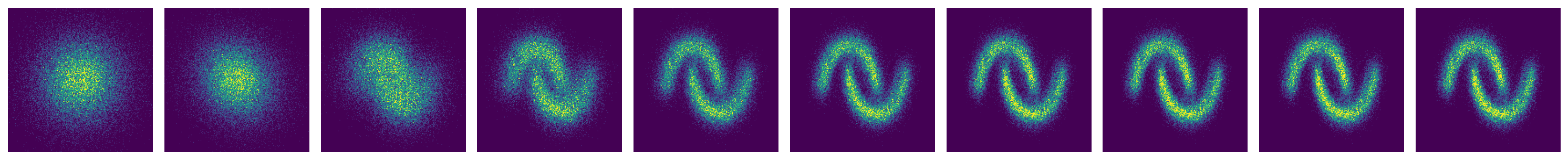}
\includegraphics[height=0.05\linewidth, width=0.49\linewidth]{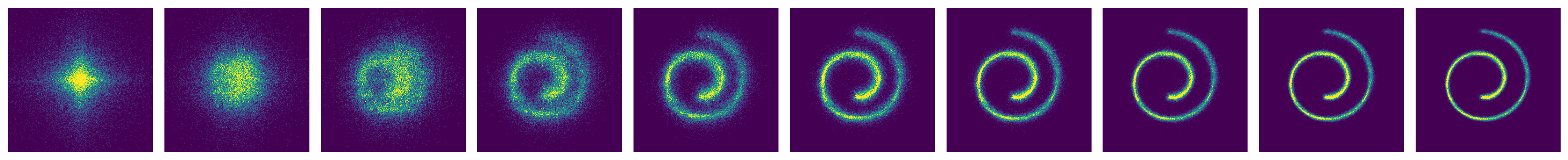}
\includegraphics[height=0.05\linewidth, width=0.49\linewidth]{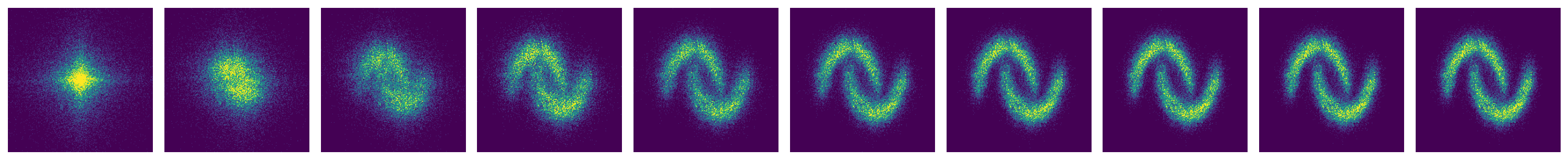}
\caption{Sample paths in the Swiss roll (left) and moons (right) examples: VP, OT, sino, PG-2, and PG-1 (from top to bottom).}
\label{fig:swissroll_moons_sample}
\end{figure}

\begin{table}[htbp]
\centering
\caption{Swiss roll dataset comparison in $W_2$, ED, and MMD (mean $\pm$ std over 10 seeds).}
\label{tab:swissroll}
\resizebox{\columnwidth}{!}{
\begin{tabular}{l l c c c c}
\hline
DIST & PATH & $W_2$ $\downarrow$ & ED $\downarrow$ & MMD $\downarrow$ & time/epoch \\
\hline

q=1.0 & VP 
& $1.49e{-1} \pm 3.5e{-2}$ 
& $3.69e{-2} \pm 1.5e{-2}$ 
& $3.33e{-4} \pm 2.2e{-4}$ 
& $5.84e{-2} \pm 6.1e{-4}$ \\

q=2.0 & VP 
& $1.42e{-1} \pm 4.5e{-2}$ 
& $2.93e{-2} \pm 2.3e{-2}$ 
& $3.06e{-4} \pm 3.6e{-4}$ 
& $6.24e{-2} \pm 5.3e{-4}$ \\

q=1.0 & OT 
& $1.14e{-1} \pm 2.3e{-2}$ 
& $1.56e{-2} \pm 1.1e{-2}$ 
& $7.50e{-5} \pm 9.1e{-5}$ 
& $5.66e{-2} \pm 5.4e{-4}$ \\

q=2.0 & OT 
& $1.12e{-1} \pm 3.0e{-2}$ 
& $1.69e{-2} \pm 1.6e{-2}$ 
& $9.03e{-5} \pm 1.7e{-4}$ 
& $5.40e{-2} \pm 6.2e{-4}$ \\

q=1.0 & sino 
& $1.33e{-1} \pm 2.5e{-2}$ 
& $2.53e{-2} \pm 1.3e{-2}$ 
& $1.87e{-4} \pm 1.2e{-4}$ 
& $5.74e{-2} \pm 3.7e{-4}$ \\

q=2.0 & sino 
& $1.17e{-1} \pm 1.4e{-2}$ 
& $1.24e{-2} \pm 8.3e{-3}$ 
& $6.05e{-5} \pm 5.0e{-5}$ 
& $5.46e{-2} \pm 4.2e{-4}$ \\

\rowcolor{lightgray} q=1.0 & PG 
& $1.05e{-1} \pm 2.4e{-2}$ 
& $9.56e{-3} \pm 7.8e{-3}$ 
& $5.27e{-5} \pm 8.4e{-5}$ 
& $1.05e{-1} \pm 3.2e{-4}$ \\

\rowcolor{lightgray} q=2.0 & PG 
& $1.08e{-1} \pm 2.0e{-2}$ 
& $8.66e{-3} \pm 8.0e{-3}$ 
& $3.55e{-5} \pm 6.3e{-5}$ 
& $1.10e{-1} \pm 9.9e{-3}$ \\

\hline
\end{tabular}}
\end{table}


\begin{table}[htbp]
\centering
\caption{Comparison on Moons and Checkerboard datasets in $W_2$, ED, and MMD (mean $\pm$ std over 10 seeds).}
\label{tab:moons_checkerboard}
\resizebox{\columnwidth}{!}{
\begin{tabular}{l l| ccc |ccc}
\hline
& & \multicolumn{3}{c}{Moons} & \multicolumn{3}{|c}{Checkerboard} \\
\hline
DIST & PATH 
& $W_2$ $\downarrow$ & ED $\downarrow$ & MMD $\downarrow$ 
& $W_2$ $\downarrow$ & ED $\downarrow$ & MMD $\downarrow$ \\
\hline

q=1.0 & VP 
& $1.33e{-1} \pm 2.4e{-2}$ & $2.56e{-2} \pm 1.5e{-2}$ & $2.14e{-4} \pm 1.8e{-4}$
& $3.58e{-1} \pm 9.6e{-2}$ & $4.45e{-2} \pm 2.9e{-2}$ & $3.01e{-4} \pm 3.0e{-4}$ \\

q=2.0 & VP 
& $1.06e{-1} \pm 1.7e{-2}$ & $1.00e{-2} \pm 1.0e{-2}$ & $5.69e{-5} \pm 6.5e{-5}$
& $3.53e{-1} \pm 6.2e{-2}$ & $4.29e{-2} \pm 1.6e{-2}$ & $2.05e{-4} \pm 1.2e{-4}$ \\

q=1.0 & OT 
& $1.17e{-1} \pm 2.1e{-2}$ & $1.94e{-2} \pm 1.2e{-2}$ & $1.08e{-4} \pm 1.2e{-4}$
& $3.15e{-1} \pm 5.3e{-2}$ & $3.23e{-2} \pm 1.7e{-2}$ & $1.63e{-4} \pm 1.1e{-4}$ \\

q=2.0 & OT 
& $9.83e{-2} \pm 1.3e{-2}$ & $9.39e{-3} \pm 8.6e{-3}$ & $5.56e{-5} \pm 6.0e{-5}$
& $2.96e{-1} \pm 4.1e{-2}$ & $2.21e{-2} \pm 1.6e{-2}$ & $9.36e{-5} \pm 7.4e{-5}$ \\

q=1.0 & sino 
& $1.25e{-1} \pm 1.8e{-2}$ & $2.26e{-2} \pm 1.1e{-2}$ & $1.59e{-4} \pm 1.1e{-4}$
& $3.59e{-1} \pm 9.1e{-2}$ & $4.76e{-2} \pm 2.4e{-2}$ & $3.38e{-4} \pm 2.5e{-4}$ \\

q=2.0 & sino 
& $1.06e{-1} \pm 1.5e{-2}$ & $1.22e{-2} \pm 1.0e{-2}$ & $6.25e{-5} \pm 5.5e{-5}$
& $3.01e{-1} \pm 4.5e{-2}$ & $2.32e{-2} \pm 1.6e{-2}$ & $1.06e{-4} \pm 7.0e{-5}$ \\

\rowcolor{lightgray}
q=1.0 & PG 
& $1.02e{-1} \pm 1.3e{-2}$ & $9.26e{-3} \pm 9.2e{-3}$ & $6.94e{-5} \pm 8.4e{-5}$
& $3.17e{-1} \pm 3.8e{-2}$ & $2.74e{-2} \pm 1.6e{-2}$ & $1.24e{-4} \pm 8.3e{-5}$ \\

\rowcolor{lightgray}
q=2.0 & PG 
& $1.00e{-1} \pm 1.4e{-2}$ & $5.62e{-3} \pm 8.0e{-3}$ & $3.68e{-5} \pm 7.3e{-5}$
& $3.05e{-1} \pm 3.2e{-2}$ & $2.56e{-2} \pm 6.8e{-3}$ & $1.04e{-4} \pm 3.0e{-5}$ \\

\hline
\end{tabular}}
\end{table}

\begin{table}[htbp]
\centering
\caption{PG $q$-sweep comparison across Swiss roll, moons, and checkerboard in $W_2$ and ED (mean $\pm$ std over 10 seeds).}
\label{tab:simulation_qs}
\resizebox{\columnwidth}{!}{
\begin{tabular}{l |cc |cc |cc}
\hline
& \multicolumn{2}{c}{Swiss Roll} & \multicolumn{2}{|c}{Moons} & \multicolumn{2}{|c}{Checkerboard} \\
\hline
DIST & $W_2$ $\downarrow$ & ED $\downarrow$ & $W_2$ $\downarrow$ & ED $\downarrow$ & $W_2$ $\downarrow$ & ED $\downarrow$ \\
\hline

q=0.5
& $3.70e{-1} \pm 2.1e{-2}$ & $6.76e{-2} \pm 5.8e{-3}$
& $3.18e{-1} \pm 1.7e{-2}$ & $6.14e{-2} \pm 4.6e{-3}$
& $3.92e{-1} \pm 4.2e{-2}$ & $3.56e{-2} \pm 1.3e{-2}$ \\

q=0.8
& $1.11e{-1} \pm 1.9e{-2}$ & $1.02e{-2} \pm 8.8e{-3}$
& $1.09e{-1} \pm 1.6e{-2}$ & $1.41e{-2} \pm 9.6e{-3}$
& $3.19e{-1} \pm 4.1e{-2}$ & $2.80e{-2} \pm 1.3e{-2}$ \\

q=1.0
& \cellcolor{lightgray}$1.05e{-1} \pm 2.4e{-2}$ & $9.56e{-3} \pm 7.8e{-3}$
& $1.02e{-1} \pm 1.3e{-2}$ & $9.26e{-3} \pm 9.2e{-3}$
& $3.17e{-1} \pm 3.8e{-2}$ & $2.74e{-2} \pm 1.6e{-2}$ \\

q=1.5
& $1.06e{-1} \pm 1.4e{-2}$ & $9.13e{-3} \pm 8.6e{-3}$
& $9.95e{-2} \pm 1.2e{-2}$ & $9.93e{-3} \pm 9.3e{-3}$
& $3.19e{-1} \pm 5.4e{-2}$ & $3.07e{-2} \pm 1.8e{-2}$ \\

q=2.0
& $1.08e{-1} \pm 2.0e{-2}$ & \cellcolor{lightgray}$8.66e{-3} \pm 8.0e{-3}$
& $1.00e{-1} \pm 1.4e{-2}$ & \cellcolor{lightgray}$5.62e{-3} \pm 8.0e{-3}$
& \cellcolor{lightgray}$3.05e{-1} \pm 3.2e{-2}$ & \cellcolor{lightgray}$2.56e{-2} \pm 6.8e{-3}$ \\

q=2.5
& $1.15e{-1} \pm 1.8e{-2}$ & $1.29e{-2} \pm 8.0e{-3}$
& \cellcolor{lightgray}$9.80e{-2} \pm 8.5e{-3}$ & $7.68e{-3} \pm 6.9e{-3}$
& $3.37e{-1} \pm 4.5e{-2}$ & $3.66e{-2} \pm 1.6e{-2}$ \\

\hline
\end{tabular}}
\end{table}

\begin{figure}[t]
\centering
\includegraphics[height=0.05\linewidth, width=0.49\linewidth]{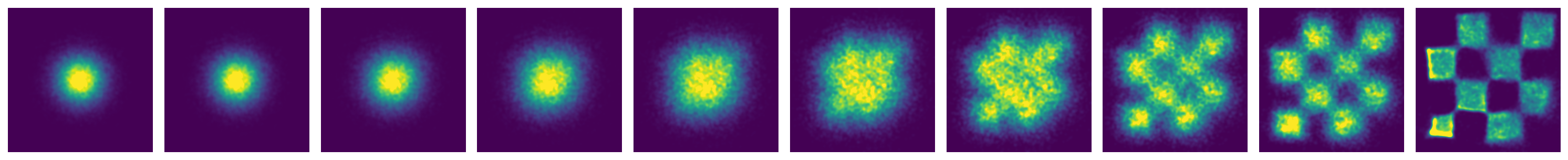}
\includegraphics[height=0.05\linewidth, width=0.49\linewidth]{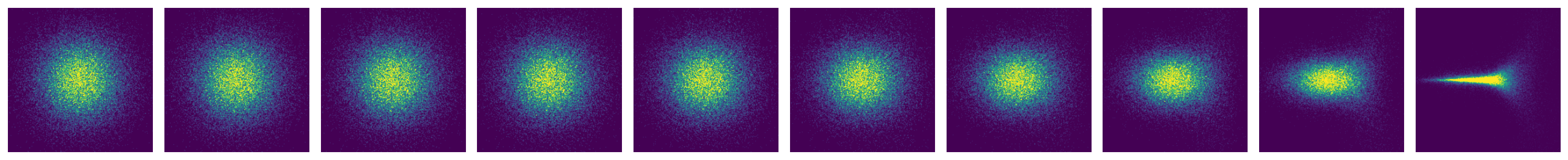}
\includegraphics[height=0.05\linewidth, width=0.49\linewidth]{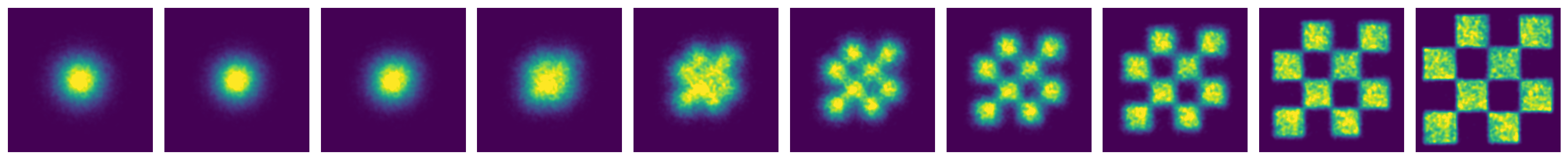}
\includegraphics[height=0.05\linewidth, width=0.49\linewidth]{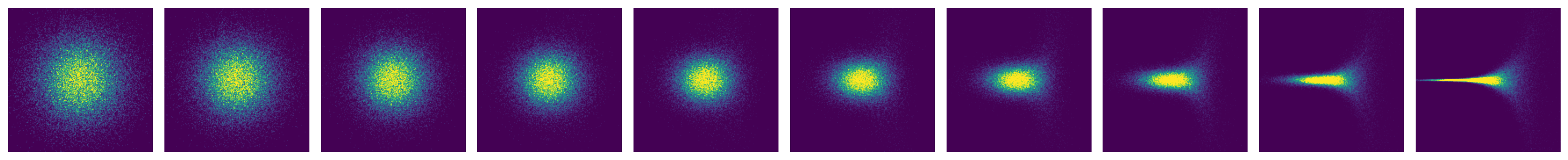}
\includegraphics[height=0.05\linewidth, width=0.49\linewidth]{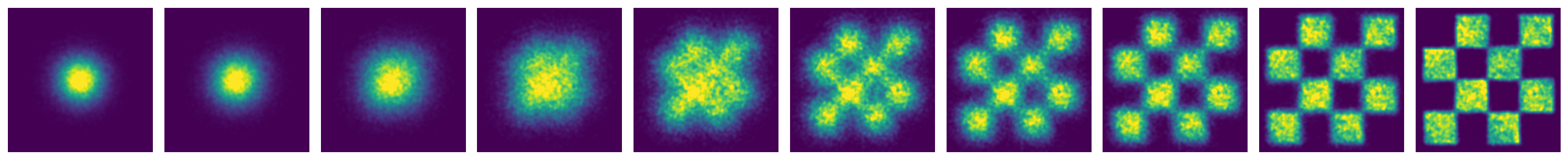}
\includegraphics[height=0.05\linewidth, width=0.49\linewidth]{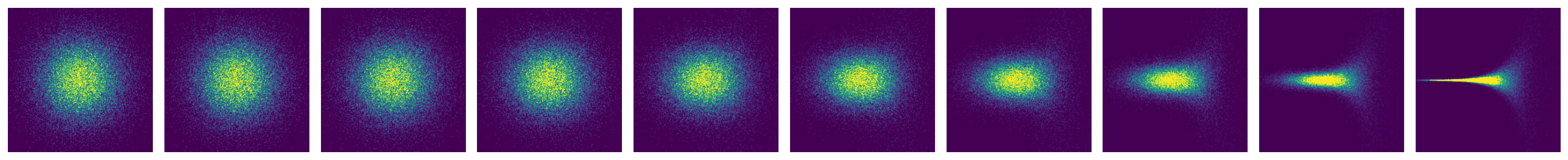}
\includegraphics[height=0.05\linewidth, width=0.49\linewidth]{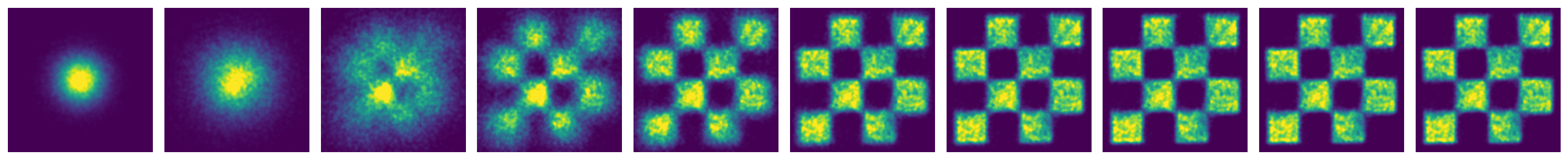}
\includegraphics[height=0.05\linewidth, width=0.49\linewidth]{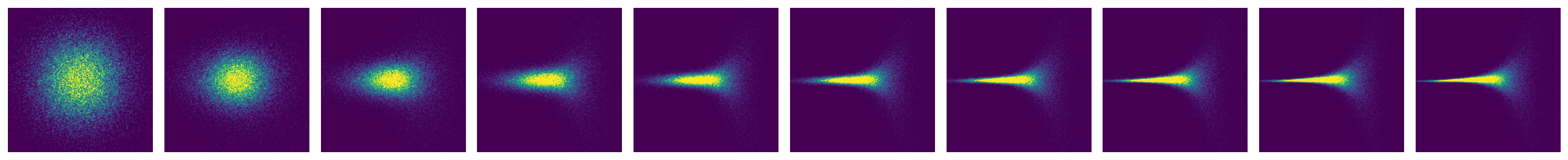}
\includegraphics[height=0.05\linewidth, width=0.49\linewidth]{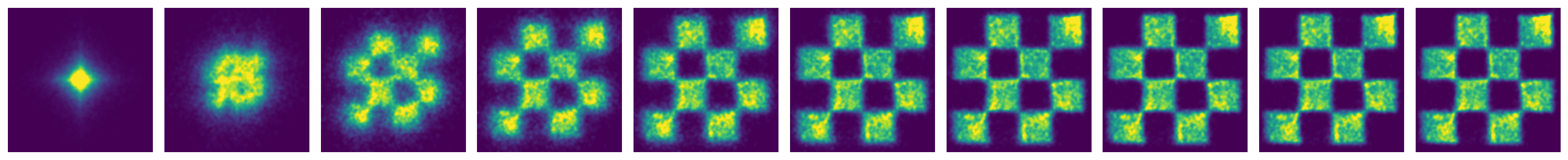}
\includegraphics[height=0.05\linewidth, width=0.49\linewidth]{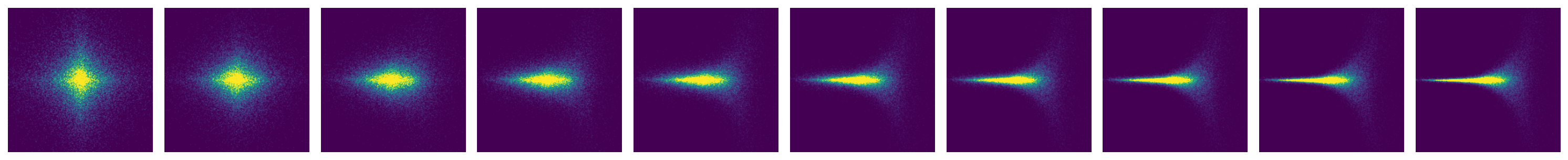}
\caption{Sample paths in the checkerboard (2d, left) and Neal's funnel (10d, right) examples: VP, OT, sino, PG-2, and PG-1 (from top to bottom).}
\label{fig:checkerboard_funnel_sample}
\end{figure}

\begin{table}[htbp]
\centering
\caption{8-Gaussians (36d): comparison in $W_2$, ED, and MMD (mean $\pm$ std over 10 seeds).}
\label{tab:8gaussians}
\resizebox{\columnwidth}{!}{
\begin{tabular}{l l c c c c}
\hline
DIST & Method & $W_2$ $\downarrow$ & ED $\downarrow$ & MMD $\downarrow$ & time/epoch \\
\hline





q=1.0 & OT
& $5.660e{0} \pm 1.43e{-2}$
& $5.93e{-2} \pm 1.25e{-2}$
& $1.39e{-4} \pm 5.09e{-5}$
& $5.92e{-2} \pm 1.27e{-4}$ \\


q=2.0 & OT
& $5.493e{0} \pm 1.41e{-2}$
& $6.62e{-2} \pm 7.93e{-3}$
& $2.47e{-4} \pm 4.99e{-5}$
& $5.70e{-2} \pm 1.41e{-4}$ \\

\hline

q=1.0 & Sino
& $5.499e{0} \pm 1.03e{-2}$
& $4.06e{-2} \pm 6.28e{-3}$
& $9.67e{-5} \pm 2.37e{-5}$
& $5.82e{-2} \pm 1.11e{-4}$ \\


q=2.0 & Sino
& \cellcolor{lightgray}$5.329e{0} \pm 2.86e{-3}$
& $1.29e{-1} \pm 2.66e{-3}$
& $1.17e{-3} \pm 4.85e{-5}$
& $5.69e{-2} \pm 8.26e{-5}$ \\

\hline

q=1.0 & RFM
& $5.49e{0} \pm 8.79e{-3}$
& $7.28e{-2} \pm 6.73e{-3}$
& $2.93e{-4} \pm 4.30e{-5}$
& $7.47e{-2} \pm 3.16e{-3}$ \\

q=2.0 & RFM
& $5.50e{0} \pm 1.61e{-2}$
& $6.41e{-2} \pm 7.89e{-3}$
& $2.24e{-4} \pm 6.10e{-5}$
& $6.89e{-2} \pm 3.31e{-3}$ \\

\hline

q=1.0 & MFM
& $5.58e{0} \pm 1.45e{-2}$
& $6.98e{-2} \pm 1.11e{-2}$
& $2.12e{-4} \pm 6.18e{-5}$
& $4.10e{-1} \pm 1.49e{-2}$ \\

q=2.0 & MFM
& $5.58e{0} \pm 1.32e{-2}$
& $5.68e{-2} \pm 8.40e{-3}$
& $1.48e{-4} \pm 3.73e{-5}$
& $3.66e{-1} \pm 1.18e{-2}$ \\

\hline

q=1.0 & FFM
& $5.79e{0} \pm 2.29e{-2}$
& $4.53e{-2} \pm 3.48e{-2}$
& $2.14e{-4} \pm 2.82e{-4}$
& $8.23e{-3} \pm 1.13e{-4}$ \\

q=2.0 & FFM
& $5.80e{0} \pm 3.78e{-2}$
& $5.56e{-2} \pm 7.05e{-2}$
& $4.89e{-4} \pm 1.27e{-3}$
& $6.76e{-3} \pm 1.31e{-4}$ \\

\hline

q=1.0 & PG-iso
& $5.399e{0} \pm 2.74e{-2}$
& $1.05e{-1} \pm 1.08e{-2}$
& $7.48e{-4} \pm 1.62e{-4}$
& $8.00e{-2} \pm 1.98e{-3}$ \\


q=2.0 & PG-iso
& $5.455e{0} \pm 7.64e{-3}$
& $7.67e{-2} \pm 5.65e{-3}$
& $3.93e{-4} \pm 5.85e{-5}$
& $6.75e{-2} \pm 1.61e{-3}$ \\

\hline

q=1.0 & PG-aniso
& $5.597e{0} \pm 5.51e{-3}$
& \cellcolor{lightgray}$4.04e{-2} \pm 6.23e{-3}$
& \cellcolor{lightgray}$6.57e{-5} \pm 1.73e{-5}$
& $6.59e{-2} \pm 2.54e{-4}$ \\


q=2.0 & PG-aniso
& $5.563e{0} \pm 1.64e{-2}$
& $5.08e{-2} \pm 4.69e{-3}$
& $1.50e{-4} \pm 2.22e{-5}$
& $6.42e{-2} \pm 4.77e{-4}$ \\

\hline

q=1.0 & HB
& $5.638e{0} \pm 1.63e{-2}$
& $4.44e{-2} \pm 1.16e{-2}$
& $7.33e{-5} \pm 3.59e{-5}$
& $7.84e{-2} \pm 9.44e{-4}$ \\


q=2.0 & HB
& $5.571e{0} \pm 1.16e{-2}$
& $5.21e{-2} \pm 5.00e{-3}$
& $1.37e{-4} \pm 2.62e{-5}$
& $6.85e{-2} \pm 3.84e{-4}$ \\

\hline
\end{tabular}}
\end{table}



\begin{figure}[t]
\centering
\includegraphics[width=.49\linewidth, height=0.3\linewidth]{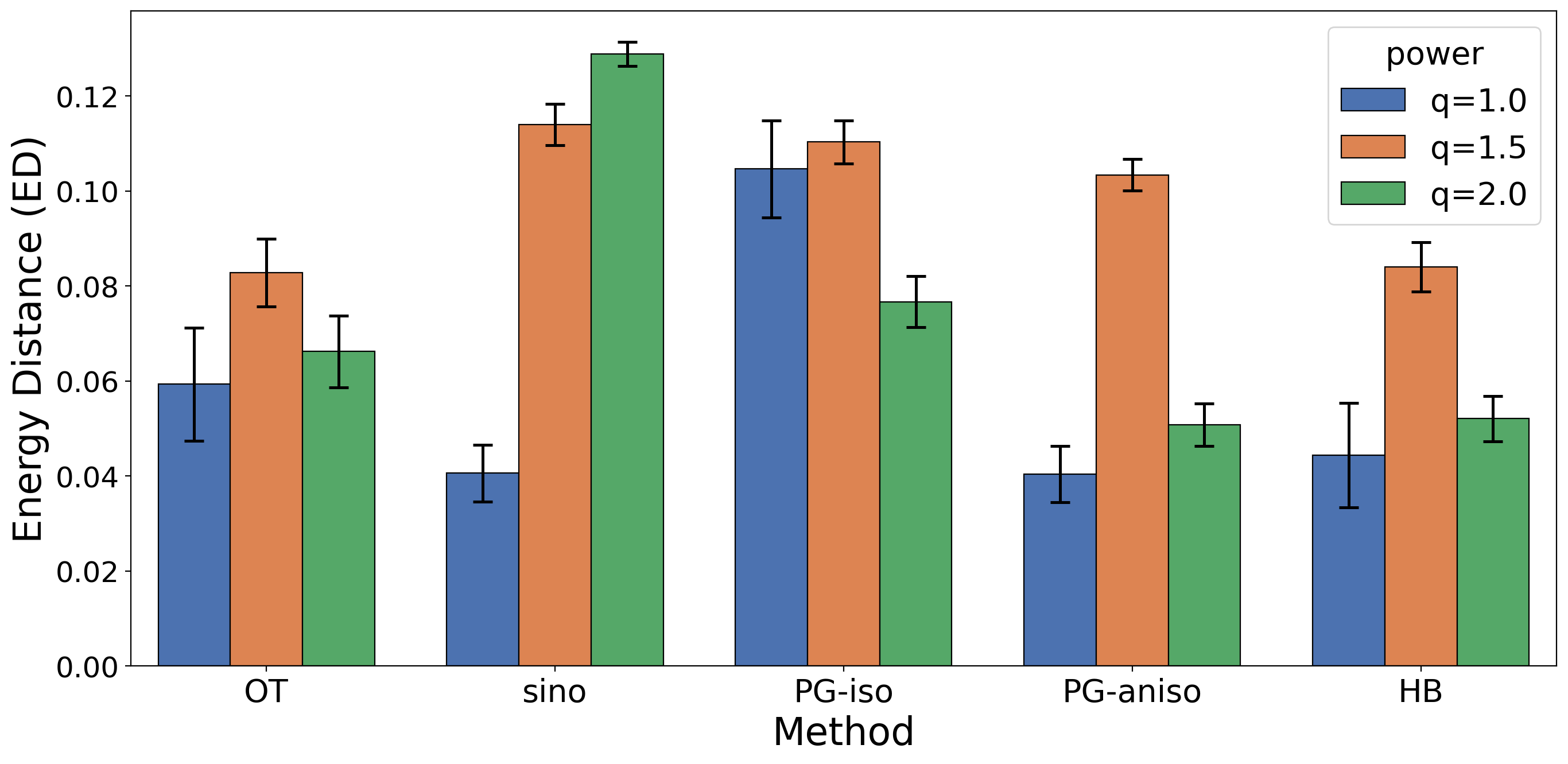}
\includegraphics[width=.49\linewidth, height=0.3\linewidth]{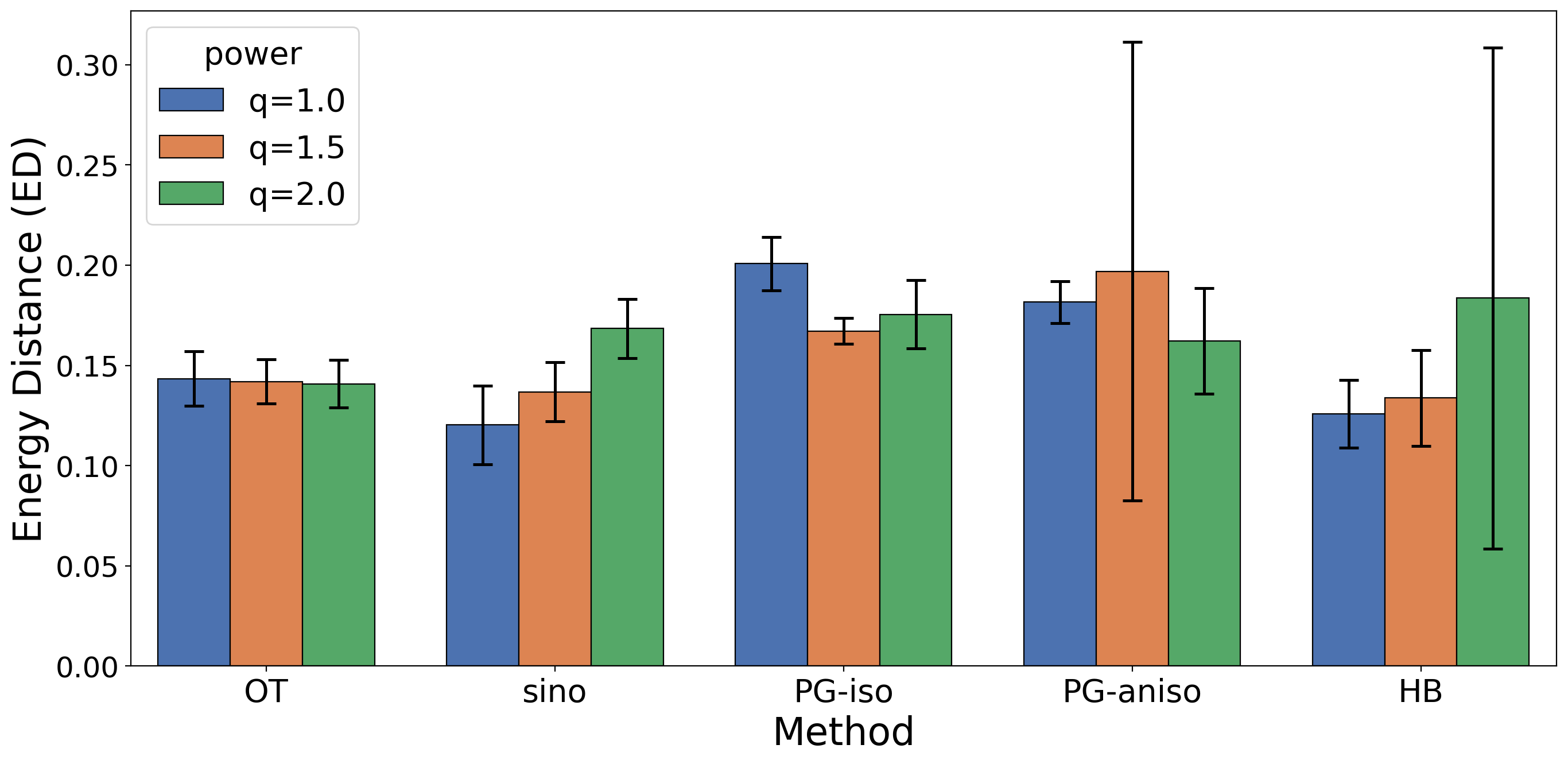}
\caption{Energy distance (ED) of 
OT, sino, PG-2, and PG-1 in the 8-Gaussians (36d, left) and funnel (100d, right) examples.}
\label{fig:ed_bars}
\end{figure}

\subsection{Higher Dimensional Simulations}


\begin{table}[htbp]
\centering
\caption{Funnel example (10d) comparison in $W_2$, ED, and MMD (mean $\pm$ std over 10 seeds).}
\label{tab:funnel}
\resizebox{\columnwidth}{!}{
\begin{tabular}{l l c c c c}
\hline
DIST & PATH & $W_2$ $\downarrow$ & ED $\downarrow$ & MMD $\downarrow$ & time/epoch \\
\hline

q=1.0 & VP 
& $2.55e{0} \pm 1.0e{-1}$ 
& $9.62e{-2} \pm 1.5e{-2}$ 
& $3.04e{-3} \pm 1.2e{-3}$ 
& $5.50e{-2} \pm 4.3e{-4}$ \\

q=2.0 & VP 
& \cellcolor{lightgray}$2.45e{0} \pm 4.6e{-2}$ 
& $9.48e{-2} \pm 3.3e{-2}$ 
& $4.10e{-3} \pm 3.0e{-3}$ 
& $5.27e{-2} \pm 1.2e{-4}$ \\

q=1.0 & OT 
& $2.47e{0} \pm 9.2e{-2}$ 
& $3.25e{-2} \pm 1.6e{-2}$ 
& $2.36e{-4} \pm 1.5e{-4}$ 
& $5.59e{-2} \pm 5.6e{-4}$ \\

q=2.0 & OT 
& $2.54e{0} \pm 1.5e{-1}$ 
& $2.69e{-2} \pm 1.3e{-2}$ 
& $1.63e{-4} \pm 8.5e{-5}$ 
& $5.37e{-2} \pm 4.0e{-4}$ \\



q=1.0 & sino 
& $2.51e{0} \pm 8.7e{-2}$ 
& $3.64e{-2} \pm 1.1e{-2}$ 
& $2.76e{-4} \pm 1.7e{-4}$ 
& $5.64e{-2} \pm 2.0e{-4}$ \\

q=2.0 & sino 
& $2.55e{0} \pm 2.5e{-1}$ 
& $3.07e{-2} \pm 8.6e{-3}$ 
& \cellcolor{lightgray}$1.50e{-4} \pm 6.3e{-5}$ 
& $5.30e{-2} \pm 1.3e{-3}$ \\

q=1.0 & PG-iso 
& $2.63e{0} \pm 2.7e{-1}$ 
& $3.54e{-2} \pm 1.1e{-2}$ 
& $3.36e{-4} \pm 2.2e{-4}$ 
& $8.00e{-2} \pm 4.8e{-4}$ \\

q=2.0 & PG-iso 
& $2.87e{0} \pm 4.2e{-1}$ 
& $3.44e{-2} \pm 1.6e{-2}$ 
& $4.41e{-4} \pm 3.0e{-4}$ 
& $7.78e{-2} \pm 6.5e{-4}$ \\

q=1.0 & PG-aniso 
& $2.61e{0} \pm 1.9e{-1}$ 
& \cellcolor{lightgray}$2.54e{-2} \pm 1.0e{-2}$ 
& $2.53e{-4} \pm 1.4e{-4}$ 
& $6.24e{-2} \pm 2.6e{-4}$ \\

q=2.0 & PG-aniso 
& $3.09e{0} \pm 4.9e{-1}$ 
& $4.54e{-2} \pm 1.2e{-2}$ 
& $5.71e{-4} \pm 2.9e{-4}$ 
& $6.04e{-2} \pm 2.2e{-4}$ \\

\hline
\end{tabular}}
\end{table}

Next, we investigate the more challenging Neal's funnel distribution \citep{Neal_2003}:
\begin{equation*}
f(\bx, \nu) = \mN_d(\bx;\bzero, e^{\nu/2}\bI) \mN(\nu; 0, 3) ,
\end{equation*}
whose density features a cusp on one end and a fat tail on the other (refer to the right panel of Figure \ref{fig:checkerboard_funnel_sample}, which also demonstrates PG methods' faster feature learning), posing difficulties for generative models to sample data from it.
We consider $d=10/ 100$ and generate 10,000 samples to train these FM models.

Tables \ref{tab:funnel100} and \ref{tab:funnel} further support the advantage and dimension robustness of PG methods as they achieve the lowest or comparable discrepancy metrics.
Except a few cells in these tables, $q=1.0$ almost unanimously beats $q=2.0$ which again justifies the adoption of exponential power distribution when handling heavy-tail and inhomogeneity in the data.
Note in Table \ref{tab:funnel100}, MFM reaches the same lowest ED at a much higher time cost ($4.36e-1$ vs $6.92e-2$ seconds per epoch for HB).

\subsection{Scientific Datasets}\label{apx:scidata}

Next, we investigate the following four datasets from molecular science, genetics, and quantum field theory:
\texttt{Alanine dipeptide} (30d) is the standard small-molecule benchmark popular for generative modeling \citep{Noe_2019}.
\texttt{QM9} (87d) contains 133,885 stable small organic molecules with up to nine heavy atoms (C, N, O, F) from GDB-17 \citep{Ruddigkeit_2012, Ramakrishnan_2014}.
\texttt{Single-cell (sc)RNA} (50d) contains 2,700 PBMCs from a healthy donor, freely available from 10x Genomics \citep{Wolf_2018}.
\texttt{Lattice $\phi^4$} (64d) is from the 2D scalar field theory, used as a benchmark for generative models \citep{Albergo_2019}.

These datasets of various dimensions contain complex data structures and serve as competitive benchmarks to test FM generative models.
In Figure \ref{fig:scientific_sample}, we visualize the sample paths of the FM methods on alanine dipeptide and QM9 datasets by projecting them to 2d subspaces using the \mbox{t-SNE} algorithm. Again PG methods maintain high speed in feature discovery. The PG methods demonstrate better (Table \ref{tab:scientific2}) or comparable (Table \ref{tab:scientific1}) performance.

\begin{figure}[t]
\centering
\includegraphics[height=0.09\linewidth, width=0.495\linewidth]{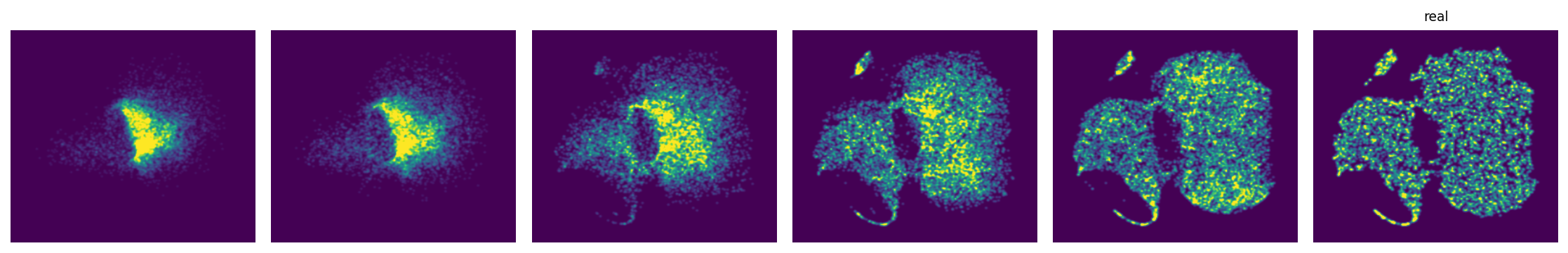}
\includegraphics[height=0.09\linewidth, width=0.495\linewidth]{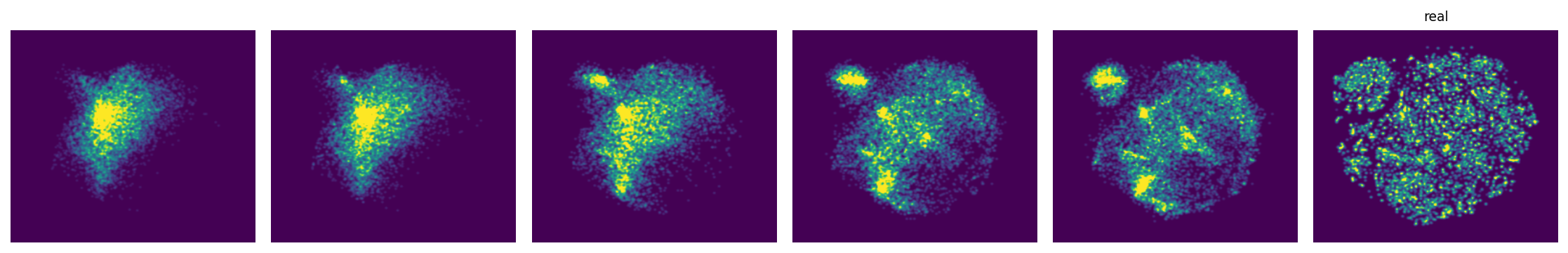}
\includegraphics[height=0.09\linewidth, width=0.495\linewidth]{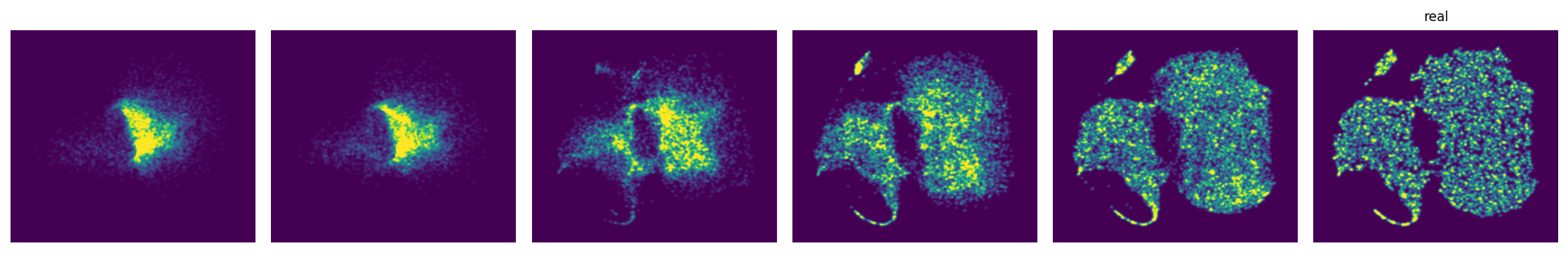}
\includegraphics[height=0.09\linewidth, width=0.495\linewidth]{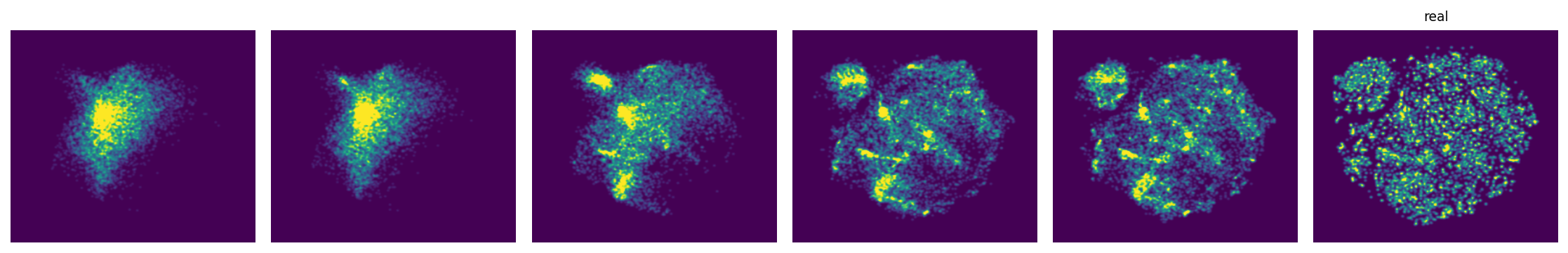}
\includegraphics[height=0.09\linewidth, width=0.495\linewidth]{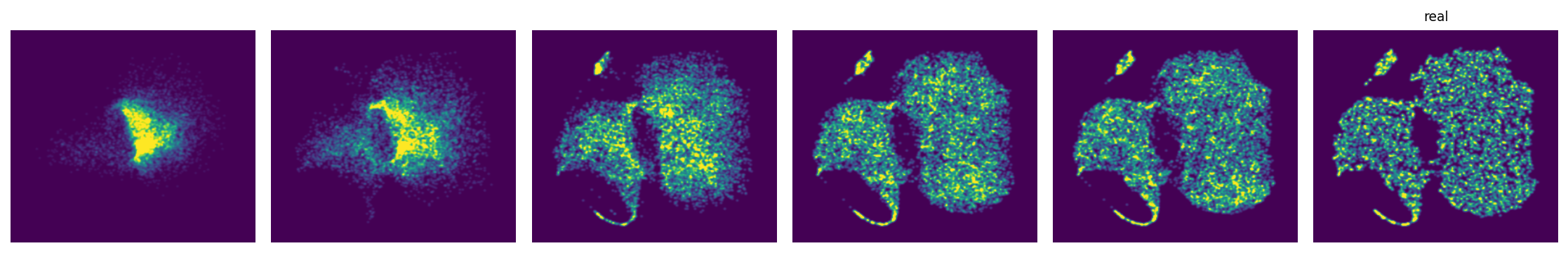}
\includegraphics[height=0.09\linewidth, width=0.495\linewidth]{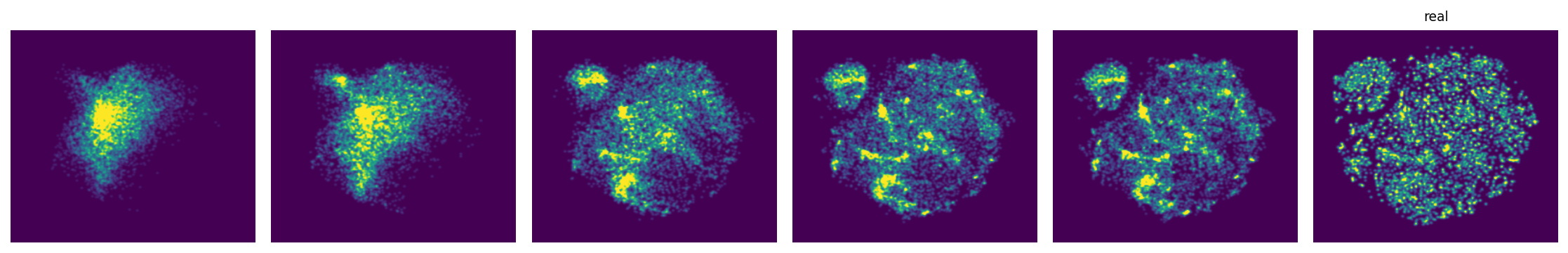}
\includegraphics[height=0.09\linewidth, width=0.495\linewidth]{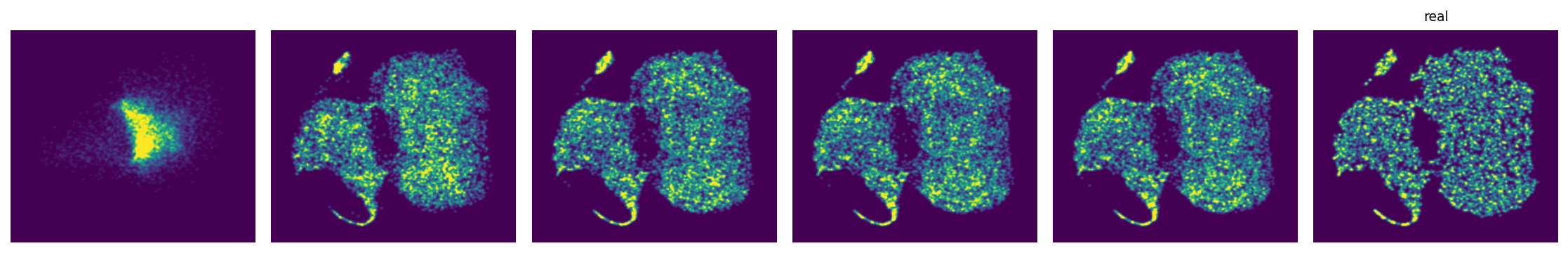}
\includegraphics[height=0.09\linewidth, width=0.495\linewidth]{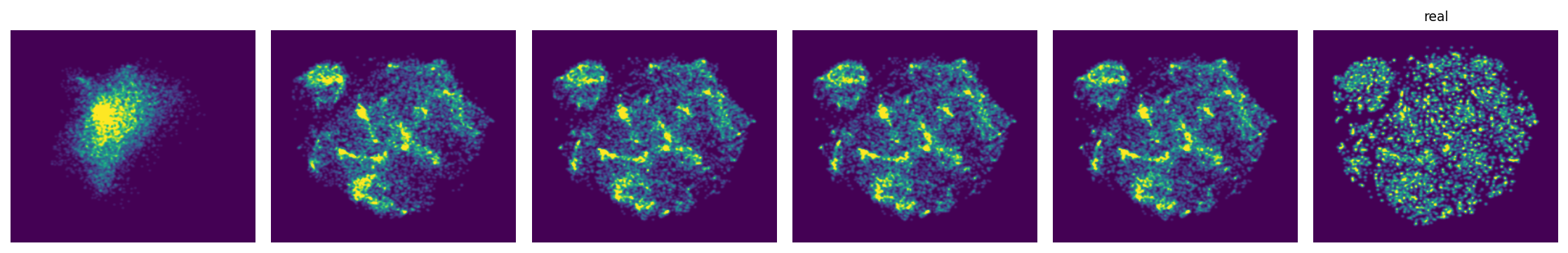}
\includegraphics[height=0.09\linewidth, width=0.495\linewidth]{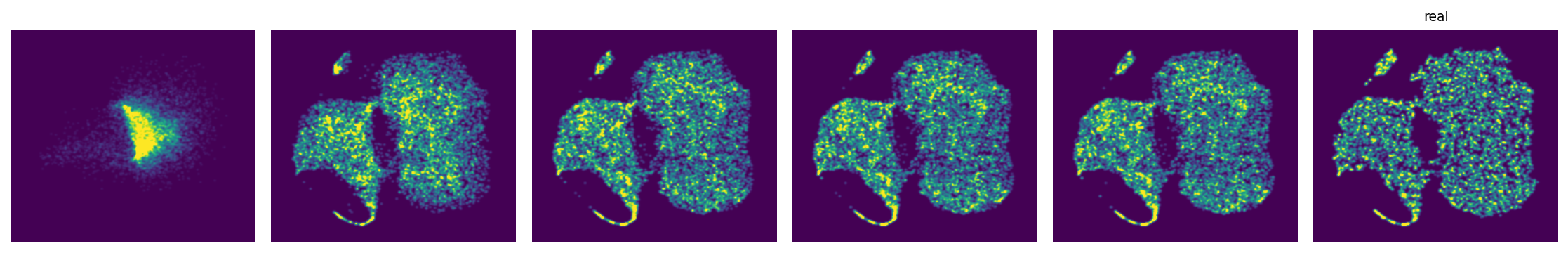}
\includegraphics[height=0.09\linewidth, width=0.495\linewidth]{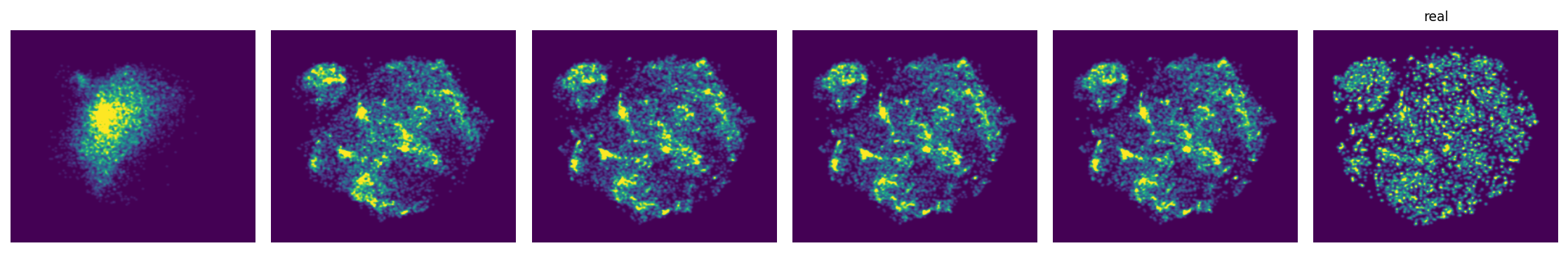}
\caption{Sample paths in the alanine dipeptide (30d, left) and QM9 (87d, right): VP, OT, sino, PG-2, and PG-1 (from top to bottom). The last column of each panel represents the real data. High dimensions have been projected to 2d using t-SNE for visualization.}
\label{fig:scientific_sample}
\end{figure}

\begin{table}[htbp]
\centering
\caption{Comparison across Alanine dipeptide (30d) and QM9 (87d) in $W_2$, ED, and MMD (mean $\pm$ std over 10 seeds).}
\label{tab:scientific1}
\resizebox{\columnwidth}{!}{
\begin{tabular}{l l |c c c |c c c}
\hline
 &  & \multicolumn{3}{c}{Alanine dipeptide (30d)} & \multicolumn{3}{|c}{QM9 (87d)} \\
\hline
DIST & PATH & $W_2$ $\downarrow$ & ED $\downarrow$ & MMD $\downarrow$ & $W_2$ $\downarrow$ & ED $\downarrow$ & MMD $\downarrow$ \\
\hline



$1.0$ & OT 
& $1.49e{0} \pm 1.4e{-1}$ & $7.75e{-2} \pm 3.1e{-2}$ & $4.02e{-4} \pm 2.6e{-4}$ 
& $7.27e{0} \pm 8.8e{-1}$ & $1.53e{-1} \pm 1.2e{-2}$ & $8.65e{-4} \pm 1.3e{-4}$ \\

$2.0$ & OT 
& \cellcolor{lightgray}$1.41e{0} \pm 1.1e{-1}$ & \cellcolor{lightgray}$5.62e{-2} \pm 2.3e{-2}$ & \cellcolor{lightgray}$1.71e{-4} \pm 1.2e{-4}$ 
& \cellcolor{lightgray}$6.84e{0} \pm 3.2e{-1}$ & $1.31e{-1} \pm 1.6e{-2}$ & $6.88e{-4} \pm 1.8e{-4}$ \\

$1.0$ & sino 
& $1.56e{0} \pm 1.0e{-1}$ & $8.86e{-2} \pm 3.6e{-2}$ & $5.00e{-4} \pm 3.0e{-4}$ 
& $7.05e{0} \pm 2.5e{-1}$ & \cellcolor{lightgray}$1.27e{-1} \pm 1.6e{-2}$ & $6.29e{-4} \pm 1.7e{-4}$ \\

$2.0$ & sino 
& $1.45e{0} \pm 1.3e{-1}$ & $6.59e{-2} \pm 3.3e{-2}$ & $2.88e{-4} \pm 2.1e{-4}$ 
& $7.04e{0} \pm 3.4e{-1}$ & \cellcolor{lightgray}$1.27e{-1} \pm 3.0e{-2}$ & $6.33e{-4} \pm 3.4e{-4}$ \\

$1.0$ & PG 
& $1.62e{0} \pm 1.8e{-1}$ & $1.05e{-1} \pm 3.7e{-2}$ & $7.40e{-4} \pm 4.9e{-4}$ 
& $6.94e{0} \pm 3.3e{-1}$ & $1.63e{-1} \pm 1.9e{-2}$ & $9.05e{-4} \pm 2.0e{-4}$ \\

$2.0$ & PG 
& $1.51e{0} \pm 9.4e{-2}$ & $8.48e{-2} \pm 2.5e{-2}$ & $4.40e{-4} \pm 2.1e{-4}$ 
& $7.00e{0} \pm 3.3e{-1}$ & $1.56e{-1} \pm 1.8e{-2}$ & $8.54e{-4} \pm 2.2e{-4}$ \\

$1.0$ & HB 
& $1.73e{0} \pm 1.9e{-1}$ & $2.00e{-1} \pm 3.4e{-2}$ & $3.22e{-3} \pm 9.7e{-4}$ 
& $7.27e{0} \pm 2.9e{-1}$ & $2.02e{-1} \pm 2.0e{-2}$ & $1.58e{-3} \pm 3.4e{-4}$ \\

$2.0$ & HB 
& $1.66e{0} \pm 7.2e{-2}$ & $1.92e{-1} \pm 1.6e{-2}$ & $3.06e{-3} \pm 5.8e{-4}$ 
& $7.48e{0} \pm 2.9e{-1}$ & $2.17e{-1} \pm 1.5e{-2}$ & $1.85e{-3} \pm 3.0e{-4}$ \\

\hline
\end{tabular}}
\end{table}

\begin{table}[htbp]
\centering
\caption{Comparison across scRNA (50d) and Lattice $\phi^4$ (64d) in $W_2$, ED, and MMD (mean $\pm$ std over 10 seeds).}
\label{tab:scientific2}
\resizebox{\columnwidth}{!}{
\begin{tabular}{l l |c c c |c c c}
\hline
 &  & \multicolumn{3}{c}{scRNA (50d)} & \multicolumn{3}{c}{Lattice $\phi^4$ (64d)} \\
\hline
DIST & PATH & $W_2$ $\downarrow$ & ED $\downarrow$ & MMD $\downarrow$ & $W_2$ $\downarrow$ & ED $\downarrow$ & MMD $\downarrow$ \\
\hline



$1.0$ & OT 
& \cellcolor{lightgray}$5.32e{0} \pm 4.0e{-2}$ & $5.17e{-1} \pm 2.0e{-2}$ & $1.88e{-2} \pm 1.5e{-3}$ 
& $7.54e{0} \pm 9.0e{-3}$ & $2.08e{-1} \pm 4.0e{-3}$ & $2.18e{-3} \pm 8.0e{-5}$ \\

$2.0$ & OT 
& $5.33e{0} \pm 5.0e{-2}$ & $5.10e{-1} \pm 2.0e{-2}$ & $1.82e{-2} \pm 1.3e{-3}$ 
& $7.53e{0} \pm 1.1e{-2}$ & $2.06e{-1} \pm 6.0e{-3}$ & $2.14e{-3} \pm 1.2e{-4}$ \\

$1.0$ & sino 
& $5.41e{0} \pm 5.0e{-2}$ & $5.36e{-1} \pm 2.0e{-2}$ & $2.04e{-2} \pm 1.5e{-3}$ 
& $7.50e{0} \pm 1.0e{-2}$ & $2.25e{-1} \pm 4.0e{-3}$ & $2.64e{-3} \pm 1.1e{-4}$ \\

$2.0$ & sino 
& $5.44e{0} \pm 5.0e{-2}$ & $5.46e{-1} \pm 2.0e{-2}$ & $2.09e{-2} \pm 1.4e{-3}$ 
& \cellcolor{lightgray}$7.47e{0} \pm 1.0e{-2}$ & $2.25e{-1} \pm 4.0e{-3}$ & $2.74e{-3} \pm 9.0e{-5}$ \\

$1.0$ & PG 
& $5.71e{0} \pm 4.0e{-2}$ & $2.93e{-1} \pm 2.2e{-2}$ & $5.07e{-3} \pm 6.0e{-4}$ 
& $7.73e{0} \pm 1.4e{-2}$ & $1.55e{-1} \pm 7.0e{-3}$ & $1.06e{-3} \pm 8.0e{-5}$ \\

$2.0$ & PG 
& $5.75e{0} \pm 5.0e{-2}$ & $2.73e{-1} \pm 1.3e{-2}$ & $4.14e{-3} \pm 4.0e{-4}$ 
& $7.69e{0} \pm 2.0e{-2}$ & \cellcolor{lightgray}$1.48e{-1} \pm 7.0e{-3}$ & \cellcolor{lightgray}$9.99e{-4} \pm 1.1e{-4}$ \\

$1.0$ & HB 
& $5.72e{0} \pm 5.1e{-2}$ & $1.83e{-1} \pm 2.1e{-2}$ & $1.42e{-3} \pm 3.7e{-4}$
& $8.31e{0} \pm 1.5e{-2}$ & $1.78e{-1} \pm 6.5e{-3}$ & $1.21e{-3} \pm 9.6e{-5}$ \\

2.0 & HB 
& $6.02e{0} \pm 4.1e{-2}$ & \cellcolor{lightgray}$1.58e{-1} \pm 1.6e{-2}$ & \cellcolor{lightgray}$8.18e{-4} \pm 1.6e{-4}$ 
& $8.28e{0} \pm 2.4e{-2}$ & $1.60e{-1} \pm 1.2e{-2}$ & $1.07e{-3} \pm 1.4e{-4}$ \\

\hline
\end{tabular}}
\end{table}

\subsection{De novo Molecule Generation}\label{apx:molecule}
The task is to draw a complete molecule from scratch: no reference structure is
supplied, and a single sample must specify the 3D coordinates, the atom type and
the formal charge of every atom, together with the bond order of every atom
\emph{pair}, jointly and consistently. Writing $n$ for the number of atoms and
$d_a, d_c, d_e$ for the number of atom types, charge states and bond types, one
molecule is a joint event of dimension
\begin{equation}
  D(n)=\underbrace{3(n-1)}_{\text{coordinates}}
      +\underbrace{d_a\,n}_{\text{atom types}}
      +\underbrace{d_c\,n}_{\text{charges}}
      +\underbrace{d_e\binom{n}{2}}_{\text{bonds}}
      =\frac{d_e}{2}\,n^{2}+\Bigl(3+d_a+d_c-\frac{d_e}{2}\Bigr)n-3,
  \label{eq:mol-dim}
\end{equation}
where the coordinate block is $3(n-1)$ rather than $3n$ because translations are
quotiented out, i.e.\ the flow acts on the zero-centre-of-mass subspace. For QM9
($d_a=5$, $d_c=6$, $d_e=5$) this is $D(n)=\tfrac{5}{2}n^{2}+\tfrac{23}{2}n-3$,
dominated by the bond block, which grows quadratically in $n$:
\begin{center}
\begin{tabular}{lccccc}
\toprule
Modality & Dimension & $n=3$ & $n=9$ & $n=18$ & $n=29$ \\
\midrule
Coordinates (zero-COM) & $3(n-1)$          &  6 &  24 &   51 &    84 \\
Atom types             & $d_a n$           & 15 &  45 &   90 &   145 \\
Charges                & $d_c n$           & 18 &  54 &  108 &   174 \\
Bonds (upper triangle) & $d_e\binom{n}{2}$ & 15 & 180 &  765 & 2030 \\
\midrule
Total            &  $D(n)$                 & 54 & 303 & 1014 & 2433 \\
\bottomrule
\end{tabular}
\end{center}
QM9 molecules contain at most nine heavy atoms and hence up to $n=29$ atoms with
hydrogens, averaging roughly $n\approx 18$: each molecule is an event of
$54$--$2{,}433$ dimensions, averaging roughly $10^{3}$, of which the bond block
alone accounts for $83\%$ at $n=29$. De novo generation on this benchmark is
therefore a problem of order $10^{3}$ dimensions per sample, despite QM9 being a
small-molecule dataset.

\clearpage

\section{Anisotropic Probabilistic Geodesic Flow}\label{sec:aniso-geod}

We now extend the geodesic flow of Section~4 to the \emph{anisotropic}
location--scale family in which the scalar scale $\sigma_t$ is replaced
by a per-coordinate scale vector $\boldsymbol{\sigma}_t\in\mathbb R_+^{\,d}$,
so that the covariance becomes $\diag(\boldsymbol{\sigma}_t^{2})$.
Following the component-wise independence assumed in Theorem~3.1, we
take the anisotropic exponential power conditional density to be the
product of $d$ univariate marginals,
\begin{equation}\label{eq:aniso-density}
  p_t(\mathbf x;\boldsymbol\mu_t,\boldsymbol\sigma_t)
  \;=\;\prod_{i=1}^{d} p_t\!\bigl(x_i;\mu_{t,i},\sigma_{t,i}^{2}\bigr),
\quad
  p_t(x_i;\mu_i,\sigma_i^{2})
  \;=\;\frac{q}{2^{1+\frac{1}{q}}\Gamma(1/q)\,\sigma_i}
       \exp\!\Bigl\{-\tfrac12\bigl|(x_i-\mu_i)/\sigma_i\bigr|^{q}\Bigr\}.
\end{equation}
The endpoints of the conditional flow are unchanged in form,
$p_0(\cdot\mid\mathbf x_1)=\prod_i \mP_q(\cdot;0,1)$ and
$p_1(\cdot\mid\mathbf x_1)=\prod_i \mP_q(\cdot;x_{1,i},\sigma_{\min}^{2})$,
i.e.\ $\boldsymbol\mu_0=\mathbf 0,\;\boldsymbol\sigma_0=\mathbf 1,\;
\boldsymbol\mu_1=\mathbf x_1,\;\boldsymbol\sigma_1=\sigma_{\min}\mathbf 1$.

\paragraph{Anisotropic Fisher metric.}
Because \eqref{eq:aniso-density} factorizes across coordinates and the
parameters $(\mu_i,\sigma_i)$ of distinct factors are decoupled, the
joint Fisher information is block-diagonal:
\begin{equation*}
  \mathbf I(\boldsymbol\mu,\boldsymbol\sigma)
  =\bigoplus_{i=1}^{d}\mathbf I^{(1)}(\mu_i,\sigma_i).
\end{equation*}
where each $2\times 2$ block is the univariate exponential power Fisher information
obtained by specializing the derivation of Section~4 to $d=1$:
\begin{equation}\label{eq:fisher-1d}
  \mathbf I^{(1)}(\mu_i,\sigma_i)
  \;=\;\frac{1}{\sigma_i^{2}}
       \begin{pmatrix} c_\mu & 0\\ 0 & c_\sigma \end{pmatrix},
\quad
  c_\mu \;=\; 2^{-2/q}q^{2}\,\frac{\Gamma\!\bigl(2-\tfrac{1}{q}\bigr)}{\Gamma(1/q)},
\quad
  c_\sigma \;=\; q .
\end{equation}
(The vanishing of the off-diagonal $\mu\!-\!\sigma$ block follows from
$\mathbb E[u_i]=\mathbb E[u_iu_j^{2}]=0$ on $S^{d-1}$, exactly as in the
isotropic case.) Consequently the anisotropic Fisher-Rao metric is the
product of $d$ rescaled Poincar\'e half-plane metrics,
\begin{equation}\label{eq:product-metric}
  ds^{2}
  \;=\;\sum_{i=1}^{d}
        \frac{c_\mu\,d\mu_i^{2}+c_\sigma\,d\sigma_i^{2}}{\sigma_i^{2}}
  \;=\;\sum_{i=1}^{d} ds_i^{2}.
\end{equation}

\paragraph{Per-coordinate geodesic equations.}
Because \eqref{eq:product-metric} is a Riemannian product, geodesics in
the joint manifold $(\mathbb R\times\mathbb R_+)^{\,d}$ are
\emph{Cartesian products of geodesics in each $H^{2}$ factor}; the
Christoffel symbols are nonzero only \emph{within} a coordinate. Repeating
the computation of Section~4 in every factor yields, for each
$i\in\{1,\ldots,d\}$,
\begin{subequations}\label{eq:aniso-geod-eq}
\begin{align}
  \ddot\mu_i \;-\;\tfrac{2}{\sigma_i}\,\dot\mu_i\,\dot\sigma_i &= 0,
\\[2pt]
  \ddot\sigma_i\;+\;\tfrac{c_\mu}{c_\sigma\sigma_i}\,\dot\mu_i^{2}
                \;-\;\tfrac{1}{\sigma_i}\,\dot\sigma_i^{2} &= 0 .
\end{align}
\end{subequations}
Each pair \eqref{eq:aniso-geod-eq} is identical to the isotropic system
(21a)--(21b) of Section~4, decoupled across $i$.

\paragraph{Closed-form solution: Poincar\'e semi-circle (per coordinate).}
Set $\lambda:=\sqrt{c_\sigma/c_\mu}$ and, for each coordinate $i$, write
$\Delta_i:=\mu_{1,i}-\mu_{0,i}$. \emph{Case $\Delta_i\neq 0$.} The geodesic
in the $i$th $H^{2}$ factor is the unique semi-circle in the
$(\mu_i,\lambda\sigma_i)$ upper half-plane connecting
$(\mu_{0,i},\lambda\sigma_{0,i})$ to $(\mu_{1,i},\lambda\sigma_{1,i})$,
\begin{equation}\label{eq:aniso-circle}
  (\mu_i-\mu_{0,i}-\mu_{c,i})^{2}+(\lambda\sigma_i)^{2}=R_i^{2},
\end{equation}
with center and radius
\begin{equation}\label{eq:aniso-cR}
  \mu_{c,i}
  \;=\;\frac{\Delta_i^{2}+\lambda^{2}(\sigma_{1,i}^{2}-\sigma_{0,i}^{2})}
            {2\Delta_i},
  \qquad
  R_i \;=\;\sqrt{\mu_{c,i}^{2}+\lambda^{2}\sigma_{0,i}^{2}}\, .
\end{equation}
Parametrizing $(\mu_i(t)-\mu_{0,i},\lambda\sigma_i(t))=
(\mu_{c,i}+R_i\cos\theta_i(t),\,R_i\sin\theta_i(t))$ with the
constant-arc-length parametrization
\begin{equation}\label{eq:aniso-theta}
  \theta_i(t)=2\arctan\!\bigl(e^{L_i t}\tan(\theta_{0,i}/2)\bigr),
  \qquad
  L_i=\log\!\frac{\tan(\theta_{1,i}/2)}{\tan(\theta_{0,i}/2)},
\end{equation}
where the boundary angles
$\theta_{0,i}=\operatorname{atan2}(\lambda\sigma_{0,i},-\mu_{c,i})$ and
$\theta_{1,i}=\operatorname{atan2}(\lambda\sigma_{1,i},\Delta_i-\mu_{c,i})$
both lie in $(0,\pi)$, gives the explicit geodesic
\begin{equation}\label{eq:aniso-mu-sigma}
\boxed{\;
\begin{aligned}
  \mu_i(t) &=\mu_{0,i}+\mu_{c,i}+R_i\cos\theta_i(t),
  \\[2pt]
  \sigma_i(t) &=\frac{R_i}{\lambda}\sin\theta_i(t).
\end{aligned}\;}
\end{equation}
\emph{Case $\Delta_i=0$.} The geodesic is the vertical
$\mu_i(t)\equiv\mu_{0,i}$ together with the geometric scale interpolation
$\sigma_i(t)=\sigma_{0,i}^{1-t}\sigma_{1,i}^{\,t}$, equivalently
$\log\sigma_i$ linear in $t$.

\paragraph{Equivalent hyperbolic ($\sech$/$\tanh$) parametrization.}
The $\sech$/$\tanh$ form of Section~4 carries over coordinate-wise.
With $\omega_i,\kappa_i,\alpha_i$ determined from the boundary conditions,
\begin{equation}\label{eq:aniso-hyper}
\boxed{\;
  \mu_i(t)
  =\mu_{0,i}+(\mu_{1,i}-\mu_{0,i})
   \frac{\tanh(\omega_i t+\alpha_i)-\tanh(\alpha_i)}
        {\tanh(\omega_i+\alpha_i)-\tanh(\alpha_i)},
\quad
  \sigma_i(t)=\frac{\omega_i}{\kappa_i}\sech(\omega_i t+\alpha_i),\;}
\end{equation}
where $\omega_i,\kappa_i,\alpha_i$ solve
\begin{align*}
  \frac{\omega_i}{\kappa_i}\sech(\alpha_i)&=\sigma_{0,i}, &
  \frac{\omega_i}{\kappa_i}\sech(\omega_i+\alpha_i)&=\sigma_{1,i}, \\
  \kappa_i&=|\Delta_i|\sqrt{\frac{c_\mu}{c_\sigma}}\,
  \frac{\kappa_i^{2}}
  {\omega_i\bigl(\tanh(\omega_i+\alpha_i)-\tanh(\alpha_i)\bigr)}.
\end{align*}
The two parametrisations are related by
\begin{equation*}
\sin(2\arctan e^{-s_i})=\sech s_i, \qquad
\cos(2\arctan e^{-s_i})=\tanh s_i,
\end{equation*}
where $s_i=\omega_i t+\alpha_i$ and $R_i=\lambda\,\omega_i/\kappa_i$.

\paragraph{Anisotropic Fisher--Rao distance and energy.}
The arc-length on the $i$th factor along the geodesic is
\begin{equation}\label{eq:aniso-FR-i}
  |L_i| \;=\;\operatorname{arcosh}\!\Biggl(
     1+\frac{c_\mu\,\Delta_i^{2}+c_\sigma(\sigma_{1,i}-\sigma_{0,i})^{2}}
            {2\,c_\sigma\,\sigma_{0,i}\sigma_{1,i}}
  \Biggr).
\end{equation}
Because the metric is a Riemannian product, the joint squared
Fisher--Rao distance and the geodesic energy are simply additive,
\begin{equation}\label{eq:aniso-FR-total}
  d_{\mathrm{FR}}^{\,2}(p_0,p_1)\;=\;
  c_\sigma\sum_{i=1}^{d} L_i^{2},
  \qquad
  e(v) \;=\; c_\sigma\sum_{i=1}^{d} L_i^{2}.
\end{equation}
For the symmetric anisotropic boundary
$\boldsymbol\sigma_0=\mathbf 1,\;\boldsymbol\sigma_1=\sigma_{\min}\mathbf 1$,
\eqref{eq:aniso-FR-i} reduces to
$L_i=\operatorname{arcosh}\!\bigl(1+[c_\mu x_{1,i}^{2}+c_\sigma(\sigma_{\min}-1)^{2}]/(2c_\sigma\sigma_{\min})\bigr)$.

\paragraph{Anisotropic vector field.}
By Theorem~3.1, applied component-wise, the vector field associated
with the anisotropic geodesic path is
\begin{equation}\label{eq:aniso-vfield}
\boxed{\;
 \mathbf v_t(\mathbf x)
  \;=\;\dot{\boldsymbol\mu}_t
       \;+\;\frac{\dot{\boldsymbol\sigma}_t}{\boldsymbol\sigma_t}
            \odot(\mathbf x-\boldsymbol\mu_t),\;}
\end{equation}
where $\odot$ and the division act component-wise. From
\eqref{eq:aniso-mu-sigma} and $\dot\theta_i(t)=L_i\sin\theta_i(t)$ one
gets, for every $i$ with $\Delta_i\neq 0$,
\begin{equation}\label{eq:aniso-mudot}
  \dot\mu_{t,i}=-R_iL_i\sin^{2}\!\theta_i(t),
\qquad
  \frac{\dot\sigma_{t,i}}{\sigma_{t,i}}=L_i\cos\theta_i(t),
\end{equation}
so that
\begin{equation}\label{eq:aniso-vi}
  v_{t,i}(x_i)
  \;=\;L_i\bigl[-R_i\sin^{2}\!\theta_i(t)+\cos\theta_i(t)\,(x_i-\mu_{t,i})\bigr].
\end{equation}
For coordinates with $\Delta_i=0$, $\dot\mu_{t,i}=0$ and
$\dot\sigma_{t,i}/\sigma_{t,i}=\log\sigma_{1,i}-\log\sigma_{0,i}$
is constant in $t$, so $v_{t,i}(x_i)=(\log\sigma_{1,i}-\log\sigma_{0,i})(x_i-\mu_{0,i})$.

\paragraph{Anisotropic conditional vector field and flow.}
Letting $\boldsymbol\mu_0=\mathbf 0,\;\boldsymbol\mu_1=\mathbf x_1,\;
\boldsymbol\sigma_0=\mathbf 1,\;\boldsymbol\sigma_1=\sigma_{\min}\mathbf 1$,
formula~\eqref{eq:aniso-vi} specialises (per coordinate, for $x_{1,i}\neq 0$) to
\begin{equation}\label{eq:aniso-uvi}
  u_{t,i}(x_i\mid x_{1,i})
  \;=\;L_i\bigl[-\bigl(R_i\sin^{2}\!\theta_i(t)
                       +\mu_{c,i}\cos\theta_i(t)
                       +R_i\cos^{2}\!\theta_i(t)\bigr)\,\mathrm{sgn}(x_{1,i})\, 1_{x_{1,i}\neq 0}
              \;+\;\cos\theta_i(t)\,x_i\bigr],
\end{equation}
which, after collecting $R_i\sin^{2}\!\theta_i+R_i\cos^{2}\!\theta_i=R_i$, gives
\begin{equation}\label{eq:aniso-uvi-clean}
  u_{t,i}(x_i\mid x_{1,i})
  \;=\;L_i\bigl[\cos\theta_i(t)\,x_i
           \;-\;(R_i+\mu_{c,i}\cos\theta_i(t))\,\mathrm{sgn}(x_{1,i})\bigr].
\end{equation}
Equivalently, in stacked vector form using component-wise scalars
$L_i,R_i,\mu_{c,i},\theta_i(t)$,
\begin{equation}\label{eq:aniso-u}
  \mathbf u_t(\mathbf x\mid\mathbf x_1)
  \;=\;\mathbf L\odot\bigl[\cos\boldsymbol\theta(t)\odot\mathbf x
       \;-\;\bigl(\mathbf R+\boldsymbol\mu_c\odot\cos\boldsymbol\theta(t)\bigr)
            \odot\operatorname{sgn}(\mathbf x_1)\bigr].
\end{equation}
The associated conditional flow has the explicit form
\begin{equation}\label{eq:aniso-flow}
  \psi_t(\mathbf x\mid\mathbf x_1)
  \;=\;\bigl(\boldsymbol\mu_c+\mathbf R\odot\cos\boldsymbol\theta(t)\bigr)
        \odot\operatorname{sgn}(\mathbf x_1)
       \;+\;\frac{1}{\lambda}\,\mathbf R\odot\sin\boldsymbol\theta(t)\odot\mathbf x ,
\end{equation}
which reduces to the isotropic flow of Section~4 when
$\boldsymbol\sigma_t=\sigma_t\mathbf 1$ for some scalar $\sigma_t$ and
the angles $\theta_i(t)$ collapse to a common $\theta(t)$.

\paragraph{Remark (training).}
With \eqref{eq:aniso-u} the FM loss \eqref{eq:loss_epd} is unchanged in form,
\begin{equation*}
\begin{aligned}
  \mathcal L_{\mathrm{CFM}}(\theta)
  ={}&\mathbb E_{t,q(\mathbf x_1),p_0(\mathbf x_0)}
  \Bigl\|v_t(\psi_t(\mathbf x_0);\theta) \\
  &\hspace{7em}-\frac{d}{dt}\psi_t(\mathbf x_0)\Bigr\|_2^2,
\end{aligned}
\end{equation*}
but the per-coordinate scalars $\{L_i,R_i,\mu_{c,i},\theta_i(t)\}_{i=1}^{d}$
must be precomputed once per sampled $\mathbf x_1$ (closed form, no ODE
solve required). Hence the anisotropic PG path is no more
expensive at training time than its isotropic counterpart.



\label{apx:aniso}

\end{document}